\documentclass{article}

    \usepackage[preprint]{neurips_2025}

\usepackage[utf8]{inputenc} 
\usepackage[T1]{fontenc}
\usepackage{hyperref}       
\usepackage{url}            
\usepackage{booktabs}       
\usepackage{nicefrac}       
\usepackage{microtype}      
\usepackage{xcolor}         
\usepackage{enumitem}
\usepackage{wrapfig}

\usepackage{mymacros}

\usepackage{tikz}
\usetikzlibrary{arrows,automata}
\usetikzlibrary{positioning}
\usetikzlibrary{arrows.meta,calc,decorations.markings,math}

\title{Generalizing Behavior via Inverse Reinforcement\\Learning with Closed-Form Reward Centroids}

\author{%
Filippo Lazzati\\
Politecnico di Milano\\
Milan, Italy\\
\texttt{filippo.lazzati@polimi.it} \\
\And
Alberto Maria Metelli \\
Politecnico di Milano\\
Milan, Italy\\
}

\usepackage[ruled]{algorithm2e}
\usepackage{subcaption}
\SetKwComment{Comment}{/* }{ */} 
\allowdisplaybreaks[4]

\let\oldnl\nl
\newcommand{\nonl}{\renewcommand{\nl}{\let\nl\oldnl}}

\begin{document}
\setlength{\abovedisplayskip}{2.5pt}
\setlength{\belowdisplayskip}{2.5pt}
\setlength{\abovedisplayshortskip}{1.5pt}
\setlength{\belowdisplayshortskip}{1.5pt}
\setlength{\belowcaptionskip}{-1pt}  
\setlength{\textfloatsep}{8pt}
\thinmuskip=3.5mu
\medmuskip=3.5mu
\thickmuskip=3.5mu

\renewcommand\thmcontinues[1]{Continued}

\maketitle

\begin{abstract}
We study the problem of generalizing an expert agent's behavior, provided
through demonstrations, to new environments and/or additional constraints.
Inverse Reinforcement Learning (IRL) offers a promising solution by seeking to
recover the expert's underlying \emph{reward function}, which, if used for
planning in the new settings, would reproduce the desired behavior. However, IRL
is inherently \emph{ill-posed}: multiple reward functions, forming the so-called
\emph{feasible set}, can explain the same observed behavior. Since these rewards
may induce different policies in the new setting, in the absence of additional
information, a decision criterion is needed to select which policy to deploy.
In this paper, we propose a novel, principled criterion that selects the
``average'' policy among those induced by the rewards in a certain
\emph{bounded} subset of the feasible set. Remarkably, we show that this policy
can be obtained by planning with the reward \emph{centroid} of that subset, for
which we derive a \emph{closed-form expression}. We then present a provably
efficient algorithm for estimating this centroid using an \emph{offline} dataset
of expert demonstrations only. Finally, we conduct \emph{numerical simulations}
that illustrate the relationship between the expert's behavior and the behavior
produced by our method.
\end{abstract}

\section{Introduction}\label{sec: introduction}

Imitation Learning (IL) \cite{pomerlau1988alvinn,abbeel2004apprenticeship}
addresses the problem of \emph{efficiently learning a desired behavior by
imitating an expert's behavior} \cite{osa2018IL}. By avoiding the difficult task
of specifying a reward function, IL offers a direct and effective way to
construct artificial agents with human-like capabilities using demonstrations of
behavior alone
\cite{ho2016generativeadversarialimitationlearning,Fu2017LearningRR,reddy2019sqil,garg2021IQlearn}.
Consequently, it has been successfully applied to a variety of real-world
applications and control problems, including autonomous driving
\cite{kuderer2015driving}, healthcare \cite{yu2019healthcare}, motion planning
\cite{finn2016guided}, and path planning
\cite{ratliff2006maximum,silver2010terrain}.

However, in many situations, the environment where the learner must operate has
different dynamics or requires additional constraints compared to the
environment where the expert's demonstrations were collected
\cite{Fu2017LearningRR,arora2020survey,schlaginhaufen2023identifiability}. A
common example are ``sim-to-real'' scenarios, where demonstrations are easily
obtained in a simulator while the target environment is the (possibly more
constrained) real world
\cite{openai2019solvingrubikscuberobot,matas2018simtoreal,desai2020IfoTX}. In
such cases, the expert's demonstrations provide valuable information about the
desired behavior, but this information must be extrapolated, as merely
transferring the expert's observed policy can lead to failure
\cite{Fu2017LearningRR,ni2021firl,arora2020survey}.

Inverse Reinforcement Learning (IRL) \cite{russell1998learning,ng2000algorithms}
addresses this \emph{generalization} problem by assuming the existence of an
\emph{expert's reward function} that, if used for planning in the new
environment, would induce the desired behavior
\cite{amin2016resolving,Fu2017LearningRR,arora2020survey,rolland2022identifiability,schlaginhaufen2024transferability}.
The task is thus reduced to inferring the expert's reward function from
demonstrations of behavior. This inference requires knowledge of the true
relationship between the expert's policy and reward, referred to as the
\emph{model of behavior}
\cite{skalse2023misspecificationinversereinforcementlearning}.
Common models of behavior studied in the literature include assuming that the
expert's policy is optimal for the expert's reward \cite{ng2000algorithms},
soft-optimal \cite{Ziebart2010ModelingPA,Fu2017LearningRR}, $\epsilon$-optimal
\cite{poiani2024inversereinforcementlearningsuboptimal}, the best policy
satisfying certain constraints \cite{schlaginhaufen2023identifiability}, or
derived from a softmax over the optimal $Q$-function
\cite{Ramachandran2007birl}.

Unfortunately, regardless of the chosen model of behavior, even if the expert's
policy is known exactly, the expert's reward cannot be \emph{uniquely}
identified \cite{ng2000algorithms,cao2021identifiability}. Instead, there exists
a set of equally plausible reward functions, known as the \emph{feasible set},
that are consistent with the expert's behavior
\cite{metelli2021provably,metelli2023towards}. Crucially, these rewards
\emph{generalize differently}, meaning they induce different optimal policies in
the target environment \cite{skalse2023invariancepolicyoptimisationpartial},
leading to ambiguity about which behavior should be deployed. To address this,
many works in the literature aim to reduce the size of the feasible set using
additional data or structural assumptions
\cite{amin2016resolving,Fu2017LearningRR,cao2021identifiability,kim2021rewardidentification,rolland2022identifiability,schlaginhaufen2023identifiability,schlaginhaufen2024transferability}.

In this work, we consider a more challenging and realistic setting where no
additional data or structure is available, and we are requested to make the best
possible decision given only the demonstrations. In this context, existing IRL
approaches can be broadly grouped into methods that work well in the best-case,
worst-case, or average-case. ``Best-case'' methods
\cite{ziebart2008maxent,choi2011mapbirl,wulfmeier2016maximumentropydeepirl,finn2016guided}
arbitrarily select a reward from the feasible set and deploy its corresponding
policy. ``Worst-case'' methods \cite{lazzati2025rel} deploy the policy that
minimizes the worst-case suboptimality over all rewards in the feasible set.
``On-average'' methods
\cite{Ramachandran2007birl,michini2012improving,bajgar2024walking}
minimize expected suboptimality under a prior reward distribution.

Unsurprisingly, ``best-case'' methods tend to be \emph{overly optimistic} and
often fail to generalize, as observed in empirical studies such as
\cite{finn2016guided,Fu2017LearningRR,viano2021robust}. Conversely,
``worst-case'' methods can be \emph{overly conservative}, and the resulting
policy can be inefficient to compute \cite{lazzati2025rel}. Positioned between
these extremes, ``on-average'' methods offer a more reasonable compromise.
However, the common choices of priors considered in the literature either rely
on \emph{domain knowledge} or turn out to be \emph{biased} in the set of
policies that the expert can demonstrate (see Section \ref{sec: existing
methods}), making them intuitively unsuitable for our specific setting, where no
additional information is available.
%

In this paper, we introduce novel priors over rewards that are \emph{unbiased}
with respect to the policy that the expert will demonstrate.
Building on them, we develop \emph{provably efficient} ``on-average'' algorithms
for generalizing the observed expert's behavior to new environments and
constraints.
%

\textbf{Contributions.}~~%
We summarize the main contributions of this work below:
\begin{itemize}
  [noitemsep, leftmargin=*, topsep=-2pt]
    \item We categorize existing IRL-based approaches for generalizing behavior
    into three groups, and we demonstrate that the most popular reward prior
    considered in the literature is \emph{biased} (Section \ref{sec:
    identifiability}).
    \item We propose novel \emph{unbiased} reward priors for three common models
    of behavior (Section \ref{sec: new prior}), and we    
    derive \emph{closed-form} expressions for their centroids restricted to the
    feasible sets (Section \ref{sec: centroids}).
    \item We develop provably efficient algorithms to estimate these centroids
    from \emph{offline} data (Section \ref{sec: sample complexity analysis}).
    \item Finally, we conduct \emph{numerical simulations} to illustrate the
    relationship between the expert's behavior and the behavior produced by our
    method (Section \ref{sec: experiments}).
\end{itemize}
%
Proofs are provided in Appendices \ref{apx: proofs existing methods}-\ref{apx:
sample complexity}, and additional related work is discussed in
Appendix \ref{apx: additional related work}.

\section{Preliminaries}\label{sec: preliminaries}

\textbf{Notation.}~~%
Given an integer $N\in\Nat$, we let $\dsb{N}\coloneqq\{1,2,\dotsc,N\}$. Given
two finite sets $\cX,\cY$, we denote by $\Delta^\cX$ and $\Delta_\cY^\cX$,
respectively, the set of probability measures on $\cX$ and the set of functions
from $\cY$ to $\Delta^\cX$.
For vectors $x\in\RR^n$, we define $\|x\|_\infty\coloneqq\max_i |x_i|$.
Given a matrix $W$, we let $\text{det}(W)$ be its determinant.
We use notation $a\,\propto\, b$ to denote proportionality between two
quantities $a,b$.
Given a Riemannian manifold $\cX\subseteq\RR^n$ with dimensionality $k\le n$
\cite{lee2019introduction}, we denote the integral of a function $f:\cX\to\RR$
over $\cX$ as $\int_\cX f(x)dx$, where $dx$ represents the $k$-dimensional
volume element of the manifold. We denote by $\text{vol}(\cX)\coloneqq \int_\cX
dx$ the $k$-dimensional volume of $\cX$.
The \emph{centroid} $C(\cX)$ is a generalization of the arithmetic mean:
$C(\cX)\coloneqq \int_\cX x dx/\text{vol}(\cX) \propto \int_\cX x dx$.
Appendix \ref{apx: additional notation} provides additional notation for the
supplementary material.

\textbf{Markov Decision Processes (MDPs).}~~%
An infinite-horizon discounted Markov Decision Process (MDP)
\cite{puterman1994markov} is defined as a tuple
$\cM_r\coloneqq\tuple{\cS,\cA,s_0,p,\gamma,r}$, where $\cS$ is the finite state
space ($S\coloneqq |\cS|$), $\cA$ is the finite action space ($A\coloneqq
|\cA|$), $s_0\in\cS$ is the initial state, $p\in\Delta_{\SA}^\cS$ is the
transition model, $\gamma\in[0,1)$ is the discount factor, and $r\in\RR^{SA}$ is
the reward function.
%
%
We will use symbol $\cM\coloneqq\tuple{\cS,\cA,s_0,p,\gamma}$ to denote an MDP
without reward (MDP$\setminus$R), i.e., an MDP-like environment where there is
no notion of reward \cite{abbeel2004apprenticeship}.
A policy $\pi\in\Pi\coloneqq\Delta_\cS^\cA$ is a strategy that prescribes
actions in states. We let $\P_{\cM,\pi}$ represent the probability distribution
induced by playing $\pi$ in $\cM_r$ or $\cM$.
We let $d_{\cM,\pi}(s,a)\coloneqq (1-\gamma)\sum_{t=0}^{+\infty} \gamma^t
\P_{\cM,\pi}(s_t=s,a_t=a)$ be the state-action occupancy measure of $\pi$ in
$\cM$, and $d_{\cM,\pi}(s)\coloneqq\sum_{a\in\cA}d_{\cM,\pi}(s,a)$ be the
state-only occupancy measure.
We denote by $\cS_{\cM,\pi}$ the set of states reachable with non-zero
probability from $s_0$ in $\cM$ when playing policy $\pi$. Formally:
$\cS_{\cM,\pi}\coloneqq\{s\in\cS\,|\,d_{\cM,\pi}(s)>0\}$.
%
We define $W_{p,\pi}\in\RR^{\cS\times \cS}$ as the linear map that, to every
pair of states $s,s'\in\cS$, associates $W_{p,\pi}(s,s')=\indic{s'=s}-\gamma
\sum_{a\in\cA}\pi(a|s) p(s'|s,a)$.

\textbf{Value and advantage functions.}~~%
We denote the value and $Q$-functions induced by a policy $\pi$ in $\cM_r$ as
$V^\pi(s;p,r)\coloneqq \E_{\cM,\pi}\bigs{\sum_{t=0}^{+\infty}\gamma^t
r(s_t,a_t)|s_0=s}$ and $Q^\pi(s,a;p,r)\coloneqq
\E_{\cM,\pi}\bigs{\sum_{t=0}^{+\infty}\gamma^t r(s_t,a_t)|s_0=s,a_0=a}$, where
$\E_{\cM,\pi}$ denotes the expectation w.r.t. $\P_{\cM,\pi}$.
%
Moreover, we let $V^\pi_{\lambda}(s;p,r)\coloneqq
\E_{\cM,\pi}\bigs{\sum_{t=0}^{+\infty}\gamma^t \bigr{r(s_t,a_t)
-\lambda\log\pi(a_t|s_t)}|s_0=s}$ and $Q^\pi_{\lambda}(s,a;p,r)\coloneqq
\E_{\cM,\pi}\bigs{\sum_{t=0}^{+\infty}\gamma^t \bigr{r(s_t,a_t)
-\lambda\log\pi(a_t|s_t)}|s_0=s,a_0=a}$ denote the \emph{soft} value and
$Q$-functions associated to a policy $\pi$ where $\lambda\ge 0$ is the
regularization coefficient
\cite{Ziebart2010ModelingPA,haarnoja2017rldeepenergypolicies}.
It is useful to introduce notation for the advantage and soft-advantage
functions as $A^\pi(s,a;p,r)\coloneqq Q^\pi(s,a;p,r)-V^\pi(s;p,r)$ and
$A^\pi_\lambda(s,a;p,r)\coloneqq Q^\pi_\lambda(s,a;p,r)-V^\pi_\lambda(s;p,r)$.
The optimal counterparts of all these quantities will be denoted using a star
$*$; namely, we let $V^*(s;p,r)\coloneqq \max_\pi V^\pi(s;p,r)$,
$Q^*(s,a;p,r)\coloneqq \max_\pi Q^\pi(s,a;p,r)$, and $A^*(s,a;p,r)\coloneqq
Q^*(s,a;p,r)-V^*(s;p,r)$ be the optimal value, $Q$ and advantage functions, and
$V^*_\lambda(s;p,r)\coloneqq \max_\pi V^\pi_\lambda(s;p,r)$,
$Q^*_\lambda(s,a;p,r)\coloneqq \max_\pi Q^\pi_\lambda(s,a;p,r)$ and
$A^*_\lambda(s,a;p,r)\coloneqq Q^*_\lambda(s,a;p,r)-V^*_\lambda(s;p,r)$ the
soft-optimal value, $Q$ and advantage functions.
Finally, we define
$\Pi^*(p,r)\coloneqq \{\pi\in\Pi\,|\, \forall
s\in\cS:V^\pi(s;p,r)=V^*(s;p,r)\}$.

\textbf{Sets of rewards.}~~%
Given an \MDPr $\cM$, a set of states $\overline{\cS}\subseteq
\cS$ and a policy $\pi$ deterministic on $\overline{\cS}$, we define the set of
rewards that make $\pi$ optimal in $\overline{\cS}$ as:
\begin{align}
  \cR^{\text{OPT},\overline{\cS}}_{\cM,\pi}\coloneqq\bigc{r\in\RR^{SA}\,\big|\,
  \forall s\in\overline{\cS}:\;\pi(s)\in\argmax\nolimits_{a\in\cA} Q^*(s,a;p,r)}
  \label{eq: fs OPT}.
\end{align}
Similarly, given a policy $\pi$ stochastic on $\overline{\cS}$, we define the
set of rewards that make $\pi$ respectively soft-optimal (for some $\lambda>0$) and Boltzmann w.r.t.
$Q^*$ (for some $\beta>0$) in $\overline{\cS}$ as:
\begin{align}  
  &\cR^{\text{MCE},\overline{\cS}}_{\cM,\pi}\coloneqq\bigc{r\in\RR^{SA}\,
  \big|\, 
  \forall s\in\overline{\cS},\forall a\in\cA:\;\pi(a|s) \propto e^{\frac{1}{\lambda}
  Q^*_\lambda(s,a;p,r)}
  },
  \label{eq: fs MCE}\\
  &\cR^{\text{BIRL},\overline{\cS}}_{\cM,\pi}\coloneqq\bigc{r\in\RR^{SA}\,
  \big|\, 
  \forall s\in\overline{\cS},\forall a\in\cA:\;
  \pi(a|s) \propto e^{\frac{1}{\beta}
  Q^*(s,a;p,r)}}.\label{eq: fs BIRL}
\end{align}
The meaning of the acronyms OPT, MCE, BIRL will be clear in the next paragraph.
%

\textbf{Inverse Reinforcement Learning (IRL).}~~%
In IRL \cite{russell1998learning,ng2000algorithms}, we are given demonstrations
from an expert's policy $\pi_E$ in an \MDPr $\cM\coloneqq\tuple{\cS,\cA,
s_0,p,\gamma}$, and we aim to recover the expert's reward $r_E\in\RR^{SA}$ from
$\pi_E$.
To complete the formulation of an IRL problem, we need to specify a \emph{model
of behavior} \cite{skalse2023misspecificationinversereinforcementlearning},
namely, the relationship between $\pi_E$ and $r_E$ that an IRL algorithm must
invert.
In this paper, we focus on the popular models of behavior introduced by
\cite{ng2000algorithms}, \cite{Ziebart2010ModelingPA} and
\cite{Ramachandran2007birl}, that we abbreviate respectively with OPT, MCE and
BIRL.

In OPT (OPTimal) \cite{ng2000algorithms}, the expert's policy $\pi_E$
is assumed to be optimal in $\cM_{r_E}$:
\begin{align}\label{eq: OPT}
    \pi_E\in\argmax_{\pi\in\Pi} V^\pi(s_0;p,r_E).
\end{align}
In MCE (Maximum Causal Entropy) \cite{Ziebart2010ModelingPA}, $\pi_E$
is assumed to be soft-optimal in $\cM_{r_E}$ for some $\lambda> 0$:
\begin{align}\label{eq: MCE}
    \pi_E=\argmax_{\pi\in\Pi} V^\pi_\lambda(s_0;p,r_E).
\end{align}
In BIRL (Bayesian IRL) \cite{Ramachandran2007birl}, $\pi_E$ is assumed to be
softmax w.r.t. $Q^*$ in $\cM_{r_E}$ for some $\beta> 0$:
\begin{align}\label{eq: BIRL}
    \pi_E(a|s) \,\propto\, e^{\frac{1}{\beta}
    Q^*(s,a;p,r_E)}
    \qquad \forall s\in\cS_{\cM,\pi_E},\forall a\in\cA.
\end{align}
It is well-known that the IRL problem is \emph{ill-posed}
\cite{ng2000algorithms,cao2021identifiability,skalse2023invariancepolicyoptimisationpartial},
%
i.e., given any expert's policy $\pi_E$ and a model of behavior $m\in\{$OPT,MCE,BIRL$\}$, we are not able to uniquely retrieve the expert's reward $r_E$,
because there is a set of rewards $\cR^{m}_{\cM,\pi_E}$, known as the
\emph{feasible set} \cite{metelli2021provably,metelli2023towards}, that satisfy
the constraint defined by $\pi_E$ and $m$ (i.e., Eqs. \ref{eq: OPT}, \ref{eq:
MCE} or \ref{eq: BIRL}).
It can be shown that the feasible sets $\cR^{\text{OPT}}_{\cM,\pi_E}$,
$\cR^{\text{MCE}}_{\cM,\pi_E}$, $\cR^{\text{BIRL}}_{\cM,\pi_E}$ coincide,
respectively, with $\cR^{\text{OPT},\cS_{\cM,\pi_E}}_{\cM,\pi_E}$,
$\cR^{\text{MCE},\cS_{\cM,\pi_E}}_{\cM,\pi_E}$,
$\cR^{\text{BIRL},\cS_{\cM,\pi_E}}_{\cM,\pi_E}$ (see Appendix \ref{apx:
rewriting fs}).
%
%
%
Finally, we remark that these sets are \emph{unbounded} subsets of $\RR^{SA}$
\cite{ng2000algorithms,cao2021identifiability,skalse2023invariancepolicyoptimisationpartial}.

\section{Problem Setting and Existing Methods}\label{sec:
identifiability}

In this section, we formalize the problem setting (Section~\ref{sec: problem
setting}) and review existing IRL-based approaches that address it
(Section~\ref{sec: existing methods}). Later, in Section~\ref{sec: our
approach}, we introduce our proposed approach.

\subsection{Problem Setting}\label{sec: problem setting}

We consider the setting where we are given a dataset $\cD_E \sim \pi_E$ of
trajectories collected by an expert's policy $\pi_E$ in an \MDPr
$\cM=\tuple{\cS,\cA,s_0,p,\gamma}$, and we \emph{know} the model of behavior
$m\in\{$OPT,MCE,BIRL$\}$ that relates $\pi_E$ to the \emph{unknown} expert's
reward $r_E$.
%
%
Informally, our objective is to ``generalize the observed expert's behavior''
$\cD_E$ to a new \MDPr $\cM'=\tuple{\cS,\cA,s_0',p',\gamma'}$ under a cost
constraint given by $c\in\RR^{SA}$ and a budget $k\ge0$.
Formally, we aim to find a policy $\widehat{\pi}^*$ that satisfies the cost
constraint $V^{\pi}(s_0';p',c)\le k$ and whose suboptimality in $\cM_{r_E}'$ is
as small as possible:\footnote{Note that Eq. \eqref{eq: objective} can be
equivalently rewritten as $\widehat{\pi}^*\in\argmax_{\pi'\in\Pi_{c,k}}
V^{\pi'}(s_0';p',r_E)$, but the form in Eq. \eqref{eq: objective}
is more convenient for the presentation of Section \ref{sec: existing methods}. }
\begin{align}\label{eq: objective}
 \widehat{\pi}^*\in \argmin\limits_{\pi'\in\Pi_{c,k}} \Bigr{\max\limits_{\pi\in\Pi_{c,k}} V^{\pi}(s_0';p',\popred{r_E})
  - V^{\pi'}(s_0';p',\popred{r_E})},
\end{align}
where $\Pi_{c,k}\coloneqq \{\pi\in\Pi\,|\, V^{\pi}(s_0';p',c)\le k\}$ is the set
of feasible policies.
For simplicity, we assume here that $\pi_E$ is \emph{known} on its support
$\cS_{\cM,\pi_E}$, corresponding to an infinite dataset $\cD_E$ of expert's
demonstrations. We will remove this assumption in Section \ref{sec: sample
complexity analysis}.
Finally, for convenience, we define $\pi_{\min}\coloneqq
\min_{(s,a)\in\SA}\pi_E(a|s)>0$ for MCE and BIRL.

\subsection{Existing Methods}\label{sec: existing methods}

As explained in Section \ref{sec: preliminaries}, no behavioral model
$m\in\{$OPT,MCE,BIRL$\}$ allows to uniquely identify $r_E$ from $\pi_E$ alone.
Therefore, we cannot directly optimize Eq. \eqref{eq: objective} as
$\popred{r_E}$ is \popred{unknown}.
\emph{In the absence of additional information}, the best we can do is replace
the objective in Eq. \eqref{eq: objective} with a \emph{surrogate} that does not
depend on the unknown $r_E$, but instead leverages the fact that $r_E$ belongs
to the \popb{known} feasible set $\popb{\cR_{\cM,\pi_E}^m}$.
Below, we review common surrogate objectives considered in the literature.

\textbf{``Best-case'' methods
\cite{ziebart2008maxent,choi2011mapbirl,wulfmeier2016maximumentropydeepirl,finn2016guided}.}~~%
These approaches adopt an \emph{optimistic} surrogate objective:
\begin{align}\label{eq: best case}
  \widehat{\pi}_{\text{B}}\in\argmin\limits_{\pi'\in\Pi_{c,k}}
  \popb{\min\limits_{r\in\cR_{\cM,\pi_E}^m}}\Bigr{\max\limits_{\pi\in\Pi_{c,k}} V^{\pi}(s_0';p',\popb{r})
  - V^{\pi'}(s_0';p',\popb{r})},
\end{align}
which can be optimized by first computing \emph{any} reward $r$ in the feasible
set, and then deriving the policy $\widehat{\pi}_{\text{B}}$ it induces.
In the \emph{best-case} scenario where $r=r_E$, the resulting policy
$\widehat{\pi}_{\text{B}}$ successfully generalizes the expert's behavior.
However, as noted in empirical studies such as
\cite{finn2016guided,Fu2017LearningRR,viano2021robust}, this approach is often
\emph{overly optimistic}, and the resulting policy $\widehat{\pi}_{\text{B}}$
may fail to accurately reflect the expert's true preferences.


\textbf{``Worst-case'' methods \cite{lazzati2025rel}.}~~%
Recently, \cite{lazzati2025rel} proposed a \emph{conservative} surrogate
objective:
\begin{align}\label{eq: worst case}
  \widehat{\pi}_{\text{W}}\in\argmin\limits_{\pi'\in\Pi_{c,k}} 
  \popb{\max\limits_{r\in\cR_{\cM,\pi_E}^m\cap\fR}}
  \Bigr{\max\limits_{\pi\in\Pi_{c,k}} V^{\pi}(s_0';p',\popb{r})
  - V^{\pi'}(s_0';p',\popb{r})},
\end{align}
where $\fR\coloneqq[-1,+1]^{SA}$ is the unit hypercube.
Intuitively, $\widehat{\pi}_{\text{W}}$ minimizes the suboptimality under the
most adversarial value that $r_E$ can take on inside the feasible set
$\cR_{\cM,\pi_E}^m$, thereby representing the best policy in the
\emph{worst-case}.
Note that restricting the otherwise \emph{unbounded} set $\cR_{\cM,\pi_E}^m$ to
a bounded subset like $\fR$ is necessary to prevent the objective from becoming
infinite for any feasible policy (see Appendix \ref{apx: unbounded worst case}).
However, beyond potentially being \emph{too conservative}, this approach has the
problem that it is not clear what \emph{bias} is introduced by using $\fR$
instead of another bounded subset of $\RR^{SA}$, and also that Eq. \eqref{eq:
worst case} is generally \emph{computationally hard}, as it involves maximizing
a convex function, i.e., the pointwise maximum of linear
functions \cite{boyd2004convex}.


\textbf{``On-average'' methods
\cite{Ramachandran2007birl,michini2012improving,bajgar2024walking}.}~~%
These approaches select a function $w:\RR^{SA}\to[0,1]$ that assigns a weight to
every reward, and then compute a policy $\widehat{\pi}_w$ that minimizes the
\emph{average} suboptimality with respect to the rewards in the feasible set,
\emph{weighted by} $w$:\footnote{Note that $dr$ represents the volume element of
the manifold $\cR_{\cM,\pi_E}^m$ (see Section \ref{sec: preliminaries}).}
\begin{align}\label{eq: avg case}
  \widehat{\pi}_w&\in \argmin\limits_{\pi'\in\Pi_{c,k}} 
  \popb{\int_{\cR_{\cM,\pi_E}^m} w(r)}\Bigr{\max\limits_{\pi\in\Pi_{c,k}} V^{\pi}(s_0';p',\popb{r})
  - V^{\pi'}(s_0';p',\popb{r})}\popb{dr},
\end{align}
%
%
This ``on-average'' formulation offers a reasonable compromise between the
optimism of ``best-case'' methods and the conservatism of ``worst-case''
ones.
Note that this approach is adopted by
\cite{Ramachandran2007birl,michini2012improving,bajgar2024walking} in a
\emph{Bayesian} fashion, by interpreting $w$ as a \emph{prior} distribution for
$r_E$.
Since our problem setting assumes no domain knowledge, $w$ shall
be as \emph{uninformative} as possible. To this end, \cite{Ramachandran2007birl}
proposed the function $\overline{w}$ that assigns \emph{equal weight} to all the
rewards in the unit hypercube $\fR\coloneqq[-1,+1]^{SA}$:%
\footnote{Although \cite{Ramachandran2007birl} focus on state-only
rewards, our considerations still apply (see Appendix \ref{apx: uniform prior
state-only rewards}).}
\begin{align}\label{eq: uniform prior}
  \overline{w}(r)\coloneqq \begin{cases}
    1&\text{if }r\in \fR\\
    0&\text{otherwise}
  \end{cases},
\end{align}
as the constant prior $w(r)=1$ $\forall r\in\RR^{SA}$ is improper and cannot be
used, since the feasible set $\cR_{\cM,\pi_E}^m$ is unbounded and the integral in
Eq. \eqref{eq: avg case} would diverge.
At first glance, $\overline{w}$ appears natural and uninformative. Interpreted
in a Bayesian fashion, it models an expert whose reward $r_E$ is uniformly
distributed over $\fR$. Surprisingly, however, $\overline{w}$ turns out to be
\emph{biased} in the space of policies.
To see this, let $m=$ OPT and define, for every deterministic policy
$\pi$:
%
%
\begin{align}
  \overline{w}_\pi\coloneqq \int_{\cR_{\cM,\pi}^{\text{OPT},\cS}} \overline{w}(r)dr
  = \text{vol}\bigr{\fR\cap\cR^{\text{OPT},\cS}_{\cM,\pi}}.
\end{align}
$\overline{w}_\pi$ can be Bayesianly interpreted as the (unnormalized)
probability that a reward drawn from $\overline{w}$ makes $\pi$ optimal, i.e.,
it corresponds to \emph{the probability that an OPT expert demonstrates policy
$\pi$ in $\cM$} regardless of the initial state.
As such, to be truly uninformative, we desire that all deterministic policies
should be equally ``likely'' under $\overline{w}$.
However, it turns out that this is not the case:
%
%
\begin{restatable}{prop}{counterexampleOPT}\label{prop: counterexample OPT}
  %
  There exists an \MDPr $\cM$ and two deterministic policies $\pi_1,\pi_2\in\Pi$ such
  that:
  %
  \begin{align*}
      \text{vol}\bigr{\fR\cap\cR^{\text{OPT},\cS}_{\cM,\pi_1}} \neq
      \text{vol}\bigr{\fR\cap\cR^{\text{OPT},\cS}_{\cM,\pi_2}}.
  \end{align*}
\end{restatable}
\begin{figure}[t!]
  \centering
  \scalebox{0.83}{\begin{minipage}[t!]{0.44\textwidth}
    \centering
    \begin{tikzpicture}
      \coordinate (A) at (-1,-1);
      \coordinate (B) at (1,-1);
      \coordinate (C) at (1,1);
      \coordinate (D) at (-1,1);
  
      \fill[blue!20] (A) -- (B) -- (C) -- (0.5,1) -- (-1,-0.5) -- cycle;
      \fill[red!20] (-1,-0.5) -- (0.5,1) -- (D) -- cycle;      
  
      \draw[dashed, thick] (-1.5, -1) -- (-1.83, -1.33);
      \draw[thick] (-1.5, -1) -- (1, 1.5);
      \draw[dashed, thick] (1, 1.5) -- (1.33, 1.83);
  
      \draw[thick] (A) rectangle (C);
  
      \node at (-1.4,0.8) {$\fR$};
      \node at (0.2,1.5) {$\cR_{\cM,\pi_1}^{m,\cS}$};
      \node at (1.7,1.3) {$\cR_{\cM,\pi_2}^{m,\cS}$};
  \end{tikzpicture}
  \end{minipage}}
  \hspace{3pt}
  \scalebox{0.83}{\begin{minipage}[t!]{0.44\textwidth}
    \centering
    \begin{tikzpicture}
  
    \def\coords{(1.4,0) (1.1,1.1) (0.8,1.3) (0,1.9) (-1.7,1.7)
                (-1.7,0) (-1.3,-0.8) (-1.1,-1.1) (1.3,-1.2)}
  
    \begin{scope}
      \clip plot[smooth cycle, tension=0.5] coordinates \coords;
      \clip (-2,-2) -- (-1.5,-1) -- (1,1.5) -- (2,2) -- (2,-2) -- cycle;
      \fill[blue!20] (-2,-2) rectangle (2,2); 
    \end{scope}
  
    \begin{scope}
      \clip plot[smooth cycle, tension=0.5] coordinates \coords;
      \fill[red!20] (-2, -1.5) -- (-4, +2) -- (0,+7) -- (2,2.5) -- cycle; 
    \end{scope}
  
  \draw[thick] plot[smooth cycle, tension=0.5] coordinates \coords;
  
      \coordinate (A) at (-1,-1);
      \coordinate (B) at (1,-1);
      \coordinate (C) at (1,1);
      \coordinate (D) at (-1,1);

      \draw[thick] (A) rectangle (C);
  
    \draw[dashed, thick] (-1.83, -1.33) -- (-1.5, -1);
    \draw[thick] (-1.5, -1) -- (1, 1.5);
    \draw[dashed, thick] (1, 1.5) -- (1.33, 1.83);
  
      \node at (0.8,2.1) {$\cR_{\cM,\pi_1}^{m,\cS}$};
      \node at (1.9,1.5) {$\cR_{\cM,\pi_2}^{m,\cS}$};
      \node at (-2.3,1.4) {$\sR_\cM^m$};

  \end{tikzpicture}
    \end{minipage}}
  \caption{(Left)
  In general, the volumes of the red ($\fR\cap\cR^{m,\cS}_{\cM,\pi_1}$) and blue
  ($\fR\cap\cR^{m,\cS}_{\cM,\pi_2}$) regions are different. (Right)
  The set $\sR_\cM^m$ contains $\fR$ and expands
  $\fR\cap\cR^{m,\cS}_{\cM,\pi_1}$ and $\fR\cap\cR^{m,\cS}_{\cM,\pi_2}$
  differently, so that the new red ($\sR_\cM^m\cap\cR^{m,\cS}_{\cM,\pi_1}$) and
  blue ($\sR_\cM^m\cap\cR^{m,\cS}_{\cM,\pi_2}$) regions have the same volumes.}
  \label{fig: rel fR sR}
\end{figure}
\begin{restatable}{prop}{counterexampleMCE}\label{prop: counterexample MCE}
  %
  There exists an \MDPr $\cM$ and two stochastic policies $\pi_1,\pi_2\in\Pi$ such
  that:
  \begin{align*}
      \text{vol}\bigr{\fR\cap\cR^{\text{MCE},\cS}_{\cM,\pi_1}} \neq
      \text{vol}\bigr{\fR\cap\cR^{\text{MCE},\cS}_{\cM,\pi_2}}
      \quad\text{and}\quad 
      \text{vol}\bigr{\fR\cap\cR^{\text{BIRL},\cS}_{\cM,\pi_1}} \neq
      \text{vol}\bigr{\fR\cap\cR^{\text{BIRL},\cS}_{\cM,\pi_2}}.
  \end{align*}
\end{restatable}
This situation is exemplified in Fig. \ref{fig: rel fR sR} (left).\footnote{In
Appendix \ref{apx: uniform transition models}, we provide related results
considering also a uniform prior over transition models.}
Roughly speaking, the issue is that, being biased, $\overline{\pi}$ effectively
encodes some sort of prior knowledge, which we do \emph{not} possess.

\section{Our Approach}\label{sec: our approach}

In the previous section, we saw that ``on-average'' methods, by balancing
optimism and conservatism through a prior over rewards $w$, should be the preferred
approach for generalizing expert behavior.
However, we also observed that the commonly used prior $\overline{w}$ introduces
a bias in the space of policies, by assigning different weights to different
policies.
In this section, we propose novel \emph{unbiased} priors over rewards $w_\cM^m$
(Section \ref{sec: new prior}) and leverage them to design efficient algorithms
for generalizing expert behavior (Section \ref{sec: centroids}).

\subsection{Novel Priors over Rewards}\label{sec: new prior}

Although the prior $\overline{w}$ is biased in the policy space, it has
appealing properties: it is simple to represent, centered around the zero
reward, and assigns equal weight to all rewards in the considered set ($\fR$).
To retain these advantages while removing the bias, we propose replacing
$\overline{w}$, the uniform prior over the hypercube $\fR$, with $w_\cM^m$
(defined in Eq. \ref{eq: our prior} below), the uniform prior over the set
$\sR_\cM^m$ (see Eq. \ref{eq: our sets of rewards}), which resembles $\fR$
(Proposition \ref{prop: rel hypercube}) but avoids its bias (Propositions
\ref{prop: same vol OPT}-\ref{prop: same vol BIRL}).

Specifically, for any $m\in\{$OPT,MCE,BIRL$\}$ and \MDPr $\cM$, we define
$\sR_\cM^m$ as the set of rewards whose (soft-)optimal value and advantage
functions are bounded by positive constants $C_1^m,C_2^m>0$:
\begin{align}
  \mathscr{R}^{\text{OPT}}_{\cM}&\coloneqq
  \Bigc{r\in\RR^{SA}\;\Big|\;\forall \pi\in\Pi^*(p,r),\forall s,a:\;
  |V^\pi(s;p,r)|\le C_1^{\text{OPT}} k_\pi\wedge 
  |A^\pi(s,a;p,r)|\le C_2^{\text{OPT}}
  },\nonumber\\
  \mathscr{R}^{\text{MCE}}_{\cM}&\coloneqq
  \Bigc{r\in\RR^{SA}\;\Big|\;
  \forall s,a:\, |V^*_\lambda (s;p,r)|\le C_1^{\text{MCE}}\wedge 
  |A^*_\lambda (s,a;p,r)|\le C_2^{\text{MCE}}
  },\label{eq: our sets of rewards}\\
  \mathscr{R}^{\text{BIRL}}_{\cM}&\coloneqq
  \Bigc{r\in\RR^{SA}\;\Big|\;
  \forall s,a:\, |V^* (s;p,r)|\le C_1^{\text{BIRL}}\wedge 
  |A^* (s,a;p,r)|\le C_2^{\text{BIRL}}
  },\nonumber
\end{align}
where $k_\pi\coloneqq |\text{det}(W_{p,\pi})|^{-1/S}$ for all $\pi$.
Note that these sets replace the max-norm bound on the reward (as in $\fR$) with
bounds on the max norm of the induced value and advantage functions.
With appropriate choices of $C_1^m,C_2^m$, the sets $\sR^m_\cM$ remain close in
size to the hypercube $\fR$, as desired:
\begin{restatable}{prop}{relhypercube}\label{prop: rel hypercube}
  {\thinmuskip=2mu
\medmuskip=2mu
\thickmuskip=2mu Let $\cM$ be any \MDPr. If we choose $C_1^{\text{OPT}},C_2^{\text{OPT}}\ge
(1+\gamma)/(1-\gamma)$, $C_1^{\text{MCE}},C_2^{\text{MCE}}\ge (2+\lambda\log
A)/(1-\gamma)$, and $C_1^{\text{BIRL}},C_2^{\text{BIRL}}\ge 1/(1-\gamma)$, then,
for any $m\in\{$OPT,MCE,BIRL$\}$:}
\begin{align*}
  \fR\subseteq \mathscr{R}^m_{\cM}\subseteq
  [-(1+\gamma)/(1-\gamma)C_1^m-C_2^m,+(1+\gamma)/(1-\gamma)C_1^m]^{SA}.
\end{align*}
\end{restatable}
In addition,
the sets $\sR^m_\cM$ are not biased in the space of policies:\footnote{In
Appendix \ref{apx: no same vol new envs}, we show that these properties migh not
hold in $\cM'\neq\cM$.}
%
%
\begin{restatable}{prop}{allsamevolOPT}\label{prop: same vol OPT}
  For any \MDPr $\cM$ and any pair of deterministic policies $\pi_1,\pi_2\in\Pi$, it
  holds that:
  \begin{align*}
      \text{vol}(\sR^{\text{OPT}}_\cM\cap \cR^{\text{OPT},\cS}_{\cM,\pi_1})=
      \text{vol}(\sR^{\text{OPT}}_\cM\cap \cR^{\text{OPT},\cS}_{\cM,\pi_2}).
  \end{align*}
\end{restatable}
\begin{restatable}{prop}{allsamevolMCE}\label{prop: same vol MCE}
  For any \MDPr $\cM$ and any pair of stochastic policies $\pi_1,\pi_2\in\Pi$
  such that $\pi_1(a|s)\ge\pi_{\min}$ and $\pi_2(a|s)\ge\pi_{\min}$ in every
  $s,a$, if $C_2^{\text{MCE}} \ge \lambda \log\frac{1}{\pi_{\min}}$, then:
  %
  \begin{align*}
      \text{vol}(\sR^{\text{MCE}}_\cM\cap \cR^{\text{MCE},\cS}_{\cM,\pi_1})=
      \text{vol}(\sR^{\text{MCE}}_\cM\cap \cR^{\text{MCE},\cS}_{\cM,\pi_2}).
  \end{align*}
\end{restatable}
\begin{restatable}{prop}{allsamevolBIRL}\label{prop: same vol BIRL}
   For any \MDPr $\cM$ and any
pair of stochastic policies $\pi_1,\pi_2\in\Pi$ such that
$\pi_1(a|s)\ge \pi_{\min}$ and
$\pi_2(a|s)\ge \pi_{\min}$ in every $s,a$, if
$C_2^{\text{BIRL}} \ge \beta \log\frac{1}{A\pi_{\min}}$, then:
  %
  \begin{align*}
      \text{vol}(\sR^{\text{BIRL}}_\cM\cap \cR^{\text{BIRL},\cS}_{\cM,\pi_1})=
      \text{vol}(\sR^{\text{BIRL}}_\cM\cap \cR^{\text{BIRL},\cS}_{\cM,\pi_2}).
  \end{align*}
\end{restatable}
Simply put, the sets $\sR^m_\cM$ expand the hypercube $\fR$ in directions that
ensure all policies of interest for $m$ are equally-represented in $\cM$, as
illustrated on the right of Fig. \ref{fig: rel fR sR}.
Note that the additional in MCE and BIRL is mild, as $C_2^m$ can be chosen
arbitrarily.
%
%
We are now ready to define our new priors $w_\cM^m$ as uniform distributions
over $\sR_\cM^m$. Formally, for any $m\in\{$OPT,MCE,BIRL$\}$ and \MDPr $\cM$, we
define $w_\cM^m$ as:
\begin{align}\label{eq: our prior}
  w_\cM^m(r)\coloneqq \begin{cases}
    1&\text{if }r\in \sR_\cM^m\\
    0&\text{otherwise}
  \end{cases}.
\end{align}
At first glance, it might seem odd to use different priors for different $\cM$
and $m$. However, as discussed in Section \ref{sec: existing methods}, we do
\emph{not}
assume that $r_E\sim w_\cM^m$. Instead, we use $w_\cM^m$ as a mathematical tool
that allows us to replace Eq. \eqref{eq: objective} with a surrogate objective
that is less optimistic than Eq. \eqref{eq: best case}, less conservative than
Eq. \eqref{eq: worst case}, and still assigns equal weight to all relevant
policies.

\subsection{Algorithms for Generalizing Behavior via Closed-Form Reward Centroids}
\label{sec: centroids}

To generalize expert behavior using $w_\cM^m$, we need to optimize Eq.
\eqref{eq: avg case} with $w=w_\cM^m$. To this end:
\begin{restatable}{prop}{avgcaserewritten}\label{prop: avg case rewritten}
  For every $w:\RR^{SA}\to[0,1]$, Eq. \eqref{eq: avg case} can be rewritten as:
\begin{align}\label{eq: avg case rewritten}
  \widehat{\pi}_w\in \argmax\limits_{\pi'\in\Pi_{c,k}} 
 V^{\pi'}(s_0';p',r_{w,\pi_E}),
 \qquad \text{where }\quad
 r_{w,\pi_E}=\int_{\cR_{\cM,\pi_E}^m}w(r)dr.
\end{align}
\end{restatable}
Let $\widehat{\pi}_\cM^m\coloneqq \widehat{\pi}_{w_\cM^m}$ and
$r^m_{\cM,\pi_E}\coloneqq r_{w_\cM^m,\pi_E}$.
Proposition \ref{prop: avg case rewritten} tells us that to compute the policy
$\widehat{\pi}_\cM^m$ that generalizes expert behavior using our
new prior $w_\cM^m$, we can first compute the (unnormalized) \emph{centroid}
$r^m_{\cM,\pi_E}$ of the feasible set:
\begin{align*}
  r^m_{\cM,\pi_E}=\int_{\sR^m_\cM\cap \cR_{\cM,\pi_E}^m} rdr,
\end{align*}
and then perform (constrained) planning on it.
Since there exists a rich body of literature on (constrained) planning
algorithms (see \cite{puterman1994markov,sutton2018RL}), the problem reduces to
computing $r^m_{\cM,\pi_E}$.
For doing so, one possibility is to use the Markov Chain
Monte Carlo (MCMC)-based algorithms introduced in
\cite{Ramachandran2007birl,michini2012improving,bajgar2024walking} for
approximating the mean reward $r_{w,\pi_E}$ for almost any $w$.
However, we can do significantly better. Interestingly, it turns out that, for
any model of behavior $m$, \MDPr $\cM$ and expert's policy $\pi_E$, the reward
centroid $r^m_{\cM,\pi_E}$ admits a \emph{closed-form expression}:
%
%
\begin{restatable}[OPT]{thr}{centroidOPT}
  \label{thr: centroid OPT}
    For any \MDPr $\cM$ and deterministic policy $\pi_E$, the reward centroid
    $r^{\text{OPT}}_{\cM,\pi_E}$
    after rescaling\footnote{By ``rescaling'' a vector $x\in\RR^n$ we mean
    applying a transformation $\alpha x + \beta$ with $\alpha>0,\beta\in\RR$. It
    is well-known that rescaling a reward function does not affect the rank of
    policies \cite{ng1999shaping,sutton2018RL}.} is:
    \begin{align}\label{eq: centroid OPT}
      r^{\text{OPT}}_{\cM,\pi_E}(s,a)=\begin{cases}
        1&\text{if }s\in\cS_{\cM,\pi_E}\,\wedge\,\pi_E(s)=a\\
        0&\text{if }s\in\cS_{\cM,\pi_E}\,\wedge\,\pi_E(s)\neq a\\
        \frac{1}{A}&\text{if }s\notin\cS_{\cM,\pi_E}
      \end{cases}
      \qquad \forall (s,a)\in\SA.
    \end{align}
\end{restatable}
\begin{restatable}[MCE]{thr}{centroidMCE}
  \label{thr: centroid MCE}
  Let $C_2^{\text{MCE}} \ge \lambda \log\frac{1}{\pi_{\min}}$. Then, for any
  \MDPr $\cM$ in which $\cS_{\cM,\pi_E}=\cS$, the reward centroid
  $r^{\text{MCE}}_{\cM,\pi_E}$ after rescaling is:
    \begin{align}\label{eq: centroid MCE}
      r^{\text{MCE}}_{\cM,\pi_E}(s,a)=\log\pi_E(a|s)
      \qquad \forall (s,a)\in\SA.
\end{align}
\end{restatable}
\begin{restatable}[BIRL]{thr}{centroidBIRL}
  \label{thr: centroid BIRL}
  Let $C_2^{\text{BIRL}} \ge \beta \log\frac{1}{A\pi_{\min}}$. Then, for any
  \MDPr $\cM$ in which $\cS_{\cM,\pi_E}=\cS$, the reward centroid
  $r^{\text{BIRL}}_{\cM,\pi_E}$ after rescaling is:
    \begin{align}\label{eq: centroid BIRL}
      r^{\text{BIRL}}_{\cM,\pi_E}(s,a)=
        \log \frac{\pi_E(a|s)}{\max_{a'\in\cA}\pi_E(a'|s)}
        \qquad \forall (s,a)\in\SA.
    \end{align}
\end{restatable}
Some observations are in order.
First, observe that the centroids depend only on $\pi_E$ and the set of
visited states $\cS_{\cM,\pi_E}$. As a result, we will be able to efficiently
estimate them from offline data $\cD_E$ alone (see Section \ref{sec: sample
complexity analysis}).
Next, note that $r^{\text{OPT}}_{\cM,\pi_E}$, $r^{\text{MCE}}_{\cM,\pi_E}$ and
$r^{\text{BIRL}}_{\cM,\pi_E}$ all assign higher rewards to actions that the
expert executes with larger probability in the visited states $\cS_{\cM,\pi_E}$.
Intuitively, this means that they favor behaviors that aim to ``replay'' the
actions demonstrated by the expert as much possible, as we will see in Section
\ref{sec: experiments}.
Observe also that we rescale the centroids just to obtain simpler equivalent forms.
Finally, we require $\cS_{\cM,\pi_E} = \cS$ for MCE and BIRL to exclude
unnecessarily complex cases where some states are unreachable by any policy.
We conclude this section by showing that the centroids $r^m_{\cM,\pi_E}$ can be
effectively used also for IL:\footnote{See Appendix \ref{apx: proofs IL} for a
result on $r^{\text{OPT}}_{\cM,\pi_E}$ and $r^{\text{BIRL}}_{\cM,\pi_E}$ in
$\cM'\neq\cM$ in absence of constraints.}
\begin{restatable}{prop}{intIL}\label{prop: int IL}
  Let $m\in\{$OPT,MCE,BIRL$\}$ be the model of behavior between $\pi_E$ and
  $r_E$ in an \MDPr $\cM=\tuple{\cS,\cA,s_0,p,\gamma}$.
  Let $\pi^*\in\argmax_\pi V^\pi(s_0;p,r^m_{\cM,\pi_E})$ be an optimal policy in
  $\cM$ under $r^m_{\cM,\pi_E}$. Then, $\pi^*$ is optimal also under $r_E$:
  $V^{\pi^*}(s_0;p,r_E)= V^*(s_0;p,r_E)$.
\end{restatable}
Simply put, Proposition \ref{prop: int IL} shows that if the relation between
$\pi_E$ and $r_E$ is accurately modeled by $m$, then planning with
$r^m_{\cM,\pi_E}$ yields the same optimal policy as planning with $r_E$, a
result that is not immediate especially for MCE and BIRL.

\section{Sample Complexity Analysis}\label{sec: sample complexity analysis}

In this section, we drop the assumption that $\pi_E$ is \emph{known} in its
support $\cS_{\cM,\pi_E}$, and focus on estimating the reward centroids
\scalebox{0.9}{$r^{\text{OPT}}_{\cM,\pi_E}$},
\scalebox{0.9}{$r^{\text{MCE}}_{\cM,\pi_E}$} and
\scalebox{0.9}{$r^{\text{BIRL}}_{\cM,\pi_E}$}, from a batch dataset
\scalebox{0.95}{$\cD_E=\{\tau^i\}_{i\in\dsb{N}}=
\{\tuple{s_1^i,a_1^i,s_2^i,a_2^i,\dotsc,s_H^i,a_H^i}\}_{i\in\dsb{N}}$} of $N$
trajectories, each of length $H$, collected by $\pi_E$ in
$\cM=\tuple{\cS,\cA,s_0,p,\gamma}$.
For the analysis, it is convenient to define \scalebox{0.9}{$p^{\min,H}_{\cM,\pi_E}$} as the
minimum probability with which $\pi_E$ visits any state in its support in the
first $H$ steps:
\begin{align}\label{eq: def p min}
  p^{\min,H}_{\cM,\pi_E}\coloneqq \min\limits_{s\in\cS_{\cM,\pi_E}} \P_{\cM,\pi_E}\bigr{
    \exists t\in\dsb{H}:\, s_t=s}.
\end{align}
\begin{figure}[!t]
  \centering
  \scalebox{0.86}{\begin{minipage}[t!]{0.33\textwidth}
    \centering
    \input{opt.tex}
\end{minipage}}
\hfill
\scalebox{0.86}{\begin{minipage}[t!]{0.36\textwidth}
    \centering
    \input{mce.tex}
\end{minipage}}
\hfill
\scalebox{0.86}{\begin{minipage}[t!]{0.45\textwidth}
    \centering
    \input{birl.tex}
\end{minipage}}
 \end{figure}
Our algorithm for estimating \scalebox{0.9}{$r^{\text{OPT}}_{\cM,\pi_E}$} is
reported in Algorithm \ref{alg: opt}. In short, we construct sets
\scalebox{0.9}{$\widehat{\cS}$} and \scalebox{0.9}{$\widehat{\cZ}$} containing,
respectively, all states and state-action pairs that have been visited by the
expert in $\cD_E$, and then we use them in place of $\cS_{\cM,\pi_E}$ and
$\pi_E$.
This algorithm enjoys the following guarantees:
\begin{restatable}{thr}{sampleOPT}\label{thr: sample OPT}
    Let $\delta\in(0,1)$ and $H\ge S$. If $\pi_E$ is deterministic, then, with
    probability at least $1-\delta$, it holds that $\widehat{r}^{\text{OPT}}=
    r^{\text{OPT}}_{\cM,\pi_E}$ provided that:
    \begin{align*}
      NH\ge 
      \frac{H\log\frac{|\cS_{\cM,\pi_E}|}{\delta}}{p^{\min,H}_{\cM,\pi_E}}.
    \end{align*}
  \end{restatable}
  %
  %
  We remark that Theorem \ref{thr: sample OPT} provides the number of
  samples necessary to estimate the reward $r^{\text{OPT}}_{\cM,\pi_E}$
  \emph{exactly}.
  Note that, since the available trajectories are $H$ stages long, we need to
  assume that $H\ge S$ to ensure that \scalebox{0.9}{$p^{\min,H}_{\cM,\pi_E}>0$}
  (see Lemma \ref{lemma: H at least S}).
  Moreover, there is a trade-off in choosing $H$. Increasing it improves the
  probability \scalebox{0.9}{$p^{\min,H}_{\cM,\pi_E}$}, but also increases the
  numerator $H$, and viceversa.
Regarding MCE and BIRL, we present our estimators in Algorithms \ref{alg: mce}
and \ref{alg: birl}. For every $s,a$, they both compute $N(s,a)$ as the number
of trajectories $\tau^i\in\cD_E$ in which state $s$ is visited \emph{at least
once} and, at the \emph{first visit}, the expert played action
$a$.\footnote{This estimator simplifies the theoretical analysis. In practice,
we count all occurrences of $s,a$.}
Formally, let $t(s,i)$ denote the first stage $h\in\dsb{H}$ at which $s_h^i=s$,
or $-1$ otherwise. Then:
\begin{align}\label{eq: def N sa}
  N(s,a)\coloneqq\sum\limits_{i\in\dsb{N}}\indic{t(s,i)\neq -1\wedge a_{t(s,i)}^i=a}.
\end{align}
Next, both Algorithms \ref{alg: mce} and \ref{alg: birl} estimate the expert's
policy $\pi_E$ by clipping the empirical frequency $N(s,a)/N(s)$ using a
threshold parameter $\pi_{\min}'>0$ to avoid evaluating $\log 0$, and finally they
estimate the rewards $r^{\text{MCE}}_{\cM,\pi_E}$ and
$r^{\text{BIRL}}_{\cM,\pi_E}$ by taking the logarithm.
We can show the following:
  \begin{restatable}{thr}{sampleMCE}\label{thr: sample MCE}
     Let
$\epsilon,\delta\in(0,1)$ and $H\ge S$. Under the conditions of Theorem
\ref{thr: centroid MCE}, if $\pi_{\min}\ge \pi_{\min}'$, then, with probability
$1-\delta$, we have
$\|\widehat{r}^{\text{MCE}}-r^{\text{MCE}}_{\cM,\pi_E}\|_\infty\le\epsilon$
provided that:
    \begin{align*}
      NH\ge
      \frac{H
      \log^2\frac{4SA}{\delta}}{p^{\min,H}_{\cM,\pi_E}}
      \cdot \frac{16}{\epsilon^2\pi_{\min}'} .
    \end{align*}
  \end{restatable}
  \begin{restatable}{thr}{sampleBIRL}\label{thr: sample BIRL}
    Let $\epsilon,\delta\in(0,1)$ and $H\ge S$. Under the conditions of Theorem
    \ref{thr: centroid BIRL}, if $\pi_{\min}\ge \pi_{\min}'$, then, with
    probability $1-\delta$, we have
    $\|\widehat{r}^{\text{BIRL}}-r^{\text{BIRL}}_{\cM,\pi_E}\|_\infty\le\epsilon$
    provided that:
    \begin{align*}
      NH\ge
      \frac{H \log^2\frac{8SA}{\delta}
      }{p^{\min,H}_{\cM,\pi_E}}
      \cdot \frac{33}{\epsilon^2\pi_{\min}'} .
    \end{align*}
  \end{restatable}
  Theorems \ref{thr: sample MCE} and \ref{thr: sample BIRL} establish sample
  complexity rates comparable to that of OPT in terms of $H$,
  $p^{\min,H}_{\cM,\pi_E}$ and $\log\frac{1}{\delta}$. In addition, they depend
  linearly on $\frac{1}{\pi_{\min}}$ and $\frac{1}{\epsilon^2}$ since they need to
  estimate $\pi_E$ and not only $\cS_{\cM,\pi_E}$.
  %
  %
  The next proposition bounds the error incurred when optimizing the estimate
  $\widehat{r}^m$ instead of $r_{\cM,\pi_E}^m$:
  \begin{restatable}{prop}{bounddiffr}\label{prop: bound diff r}
    For any \MDPr $\cM'=\tuple{\cS,\cA,s_0',p',\gamma'}$, cost $c\in\RR^{SA}$
    and budget $k\ge0$, it holds that:
      $|\max_{\pi\in\Pi_{c,k}}V^\pi(s_0';p',\widehat{r}^m)
      -\max_{\pi\in\Pi_{c,k}}V^\pi(s_0';p',r_{\cM,\pi_E}^m)|\le\|\widehat{r}^m-
      r_{\cM,\pi_E}^m\|_\infty/(1-\gamma)$.
    %
  \end{restatable}

\section{Numerical Simulations}\label{sec: experiments}

This section provides high-level insights into how the proposed IRL algorithms
generalize expert behavior in practice.%
\footnote{ A comparison of Algorithms \ref{alg: opt}, \ref{alg: mce} and
\ref{alg: birl} with those in
\cite{Ramachandran2007birl,michini2012improving,bajgar2024walking} using
$r_E\sim w_\cM^m$ is unfair, as the latter algorithms are necessarily
slower and less accurate because they require an additional approximation step.}
To this end, we present results from illustrative simulations. For
clarity, we use small-scale \MDPrs with deterministic transitions and assume
full knowledge of $\pi_E$ over its support, allowing us to analyze the
\emph{exact} reward centroids.
%
More details in Appendix \ref{apx: experimental details}.
%

\textbf{Experimental setting.}~~%
We consider two grid-world environments: $\cM = \tuple{\cS, \cA, s_0, p,
\gamma}$ and $\cM' = \tuple{\cS, \cA, s_0', p', \gamma'}$, with a 2D state space
$\cS = \dsb{10} \times \dsb{10}$ and five actions $\cA = \{\leftarrow,
\rightarrow, \uparrow, \downarrow, \cdot\}$. In $\cM$, each directional action
moves the agent deterministically to the adjacent cell in the corresponding
direction, while $\cdot$ keeps the agent in the current cell. In $\cM'$, the
transition model $p'$ reverses the direction of the arrow actions (e.g.,
$\leftarrow$ moves right), while $\cdot$ behaves as in $p$.
Initial states $s_0$, $s_0'$ and discount factors $\gamma$, $\gamma'$ vary
across simulations. In some simulations, $\cM'$ includes ``hard'' constraints
with cost $c\in\RR^{SA}$ and budget $k\ge0$ that restrict access to certain
states.
We visualize the state space $\cS$ as a $10\times10$ grid, marking $s_0$ or
$s_0'$ in red and constrained states (due to $c, k$) in brown (see Fig.
\ref{fig: experiments OPT}-\ref{fig: experiments BIRL}). For deterministic
policies, we display the action taken in each state (see Fig. \ref{fig:
experiments OPT} (c)); for stochastic ones, each action is shown with size
proportional to its selection probability (see Fig. \ref{fig: experiments OPT}
(b)). Policies defined only on subsets of $\cS$ (e.g., $\pi_E$ on $\cS_{\cM,
\pi_E}$) leave the remaining cells blank (see Fig. \ref{fig: experiments OPT}
(a)).
To indicate that a policy $\pi$ is executed in $\cM$ or $\cM'$, we color each
cell by the state-only occupancy measure $d_{\cM,\pi}$ or $d_{\cM',\pi}$ (darker
blue means higher occupancy).

\textbf{Simulations and results.}~~%
\begin{figure}[!t]
  \centering
  \begin{minipage}[t!]{0.195\textwidth}
    \centering
    \includegraphics[width=0.98\linewidth]{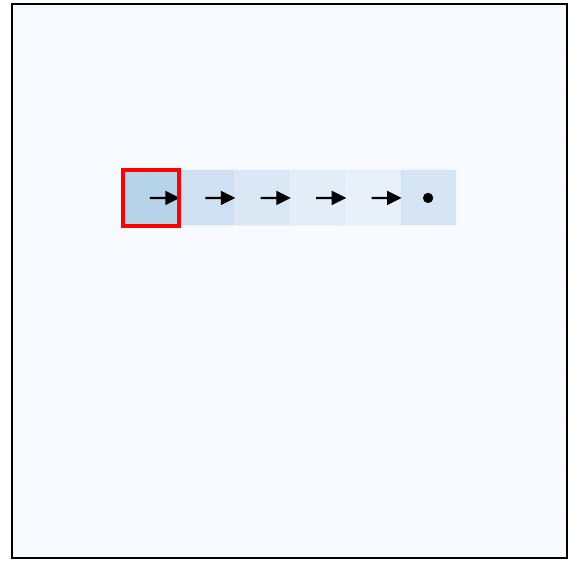}
\end{minipage}
\begin{minipage}[t!]{0.195\textwidth}
    \centering
    \includegraphics[width=0.98\linewidth]{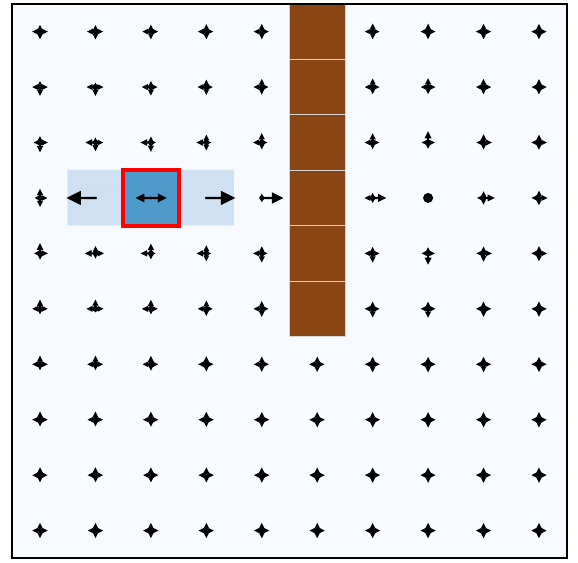}
\end{minipage}
\begin{minipage}[t!]{0.195\textwidth}
    \centering
    \includegraphics[width=0.98\linewidth]{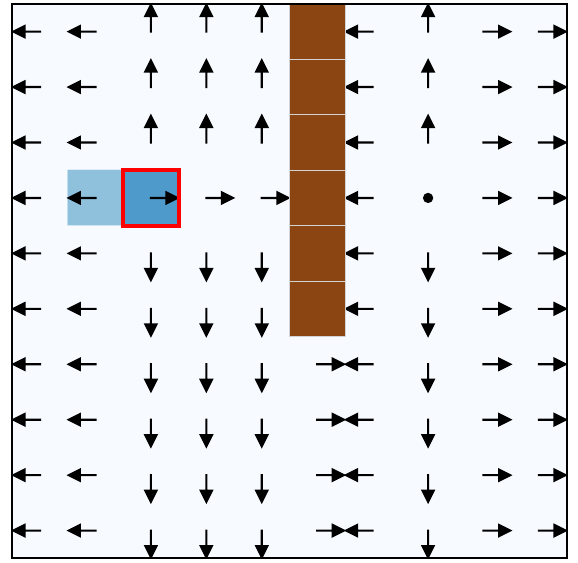}
\end{minipage}
\begin{minipage}[t!]{0.195\textwidth}
    \centering
    \includegraphics[width=0.98\linewidth]{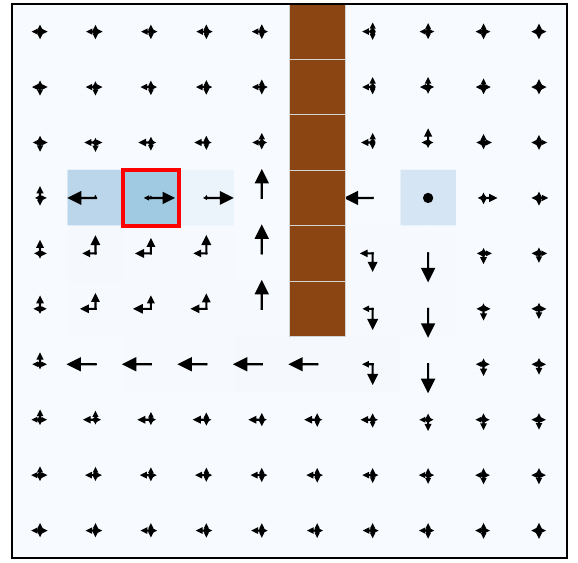}
\end{minipage}
\begin{minipage}[t!]{0.195\textwidth}
    \centering
    \includegraphics[width=0.98\linewidth]{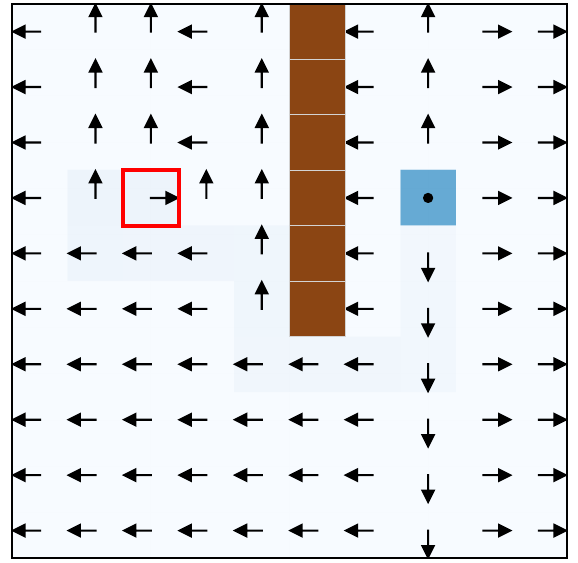}
\end{minipage}
\caption{From left to right: (a) deterministic $\pi_E$ in $\cM$ with $\gamma=0.7$, (b)-(c)
$\pi^{\text{MIMIC}}_{\cM,\pi_E,\cM'}$ and $\pi^{\text{OPT}}_\cM$ in
$\cM'$ with $\gamma'=0.7$, (d)-(e) $\pi^{\text{MIMIC}}_{\cM,\pi_E,\cM'}$ and
$\pi^{\text{OPT}}_\cM$ in $\cM'$ with $\gamma'=0.95$.
}
\label{fig: experiments OPT}
 \end{figure}
 \begin{figure}[!t]
  \centering
  \begin{minipage}[t!]{0.195\textwidth}
    \centering
    \includegraphics[width=0.98\linewidth]{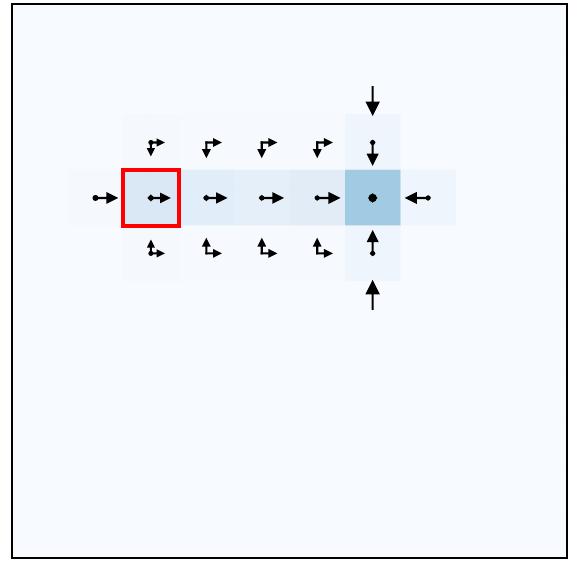}
\end{minipage}
\begin{minipage}[t!]{0.195\textwidth}
  \centering
  \includegraphics[width=0.98\linewidth]{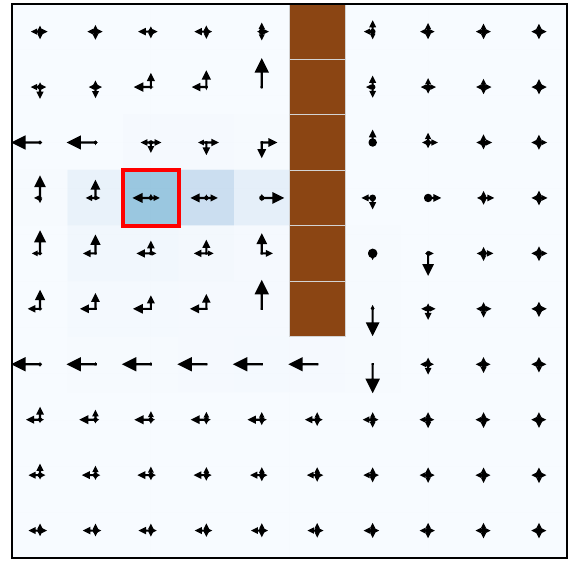}
\end{minipage}
\begin{minipage}[t!]{0.195\textwidth}
    \centering
    \includegraphics[width=0.98\linewidth]{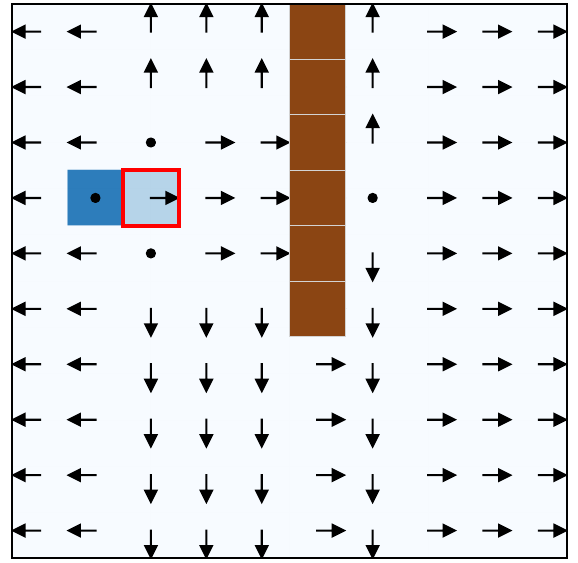}
\end{minipage}
\begin{minipage}[t!]{0.195\textwidth}
  \centering
  \includegraphics[width=0.98\linewidth]{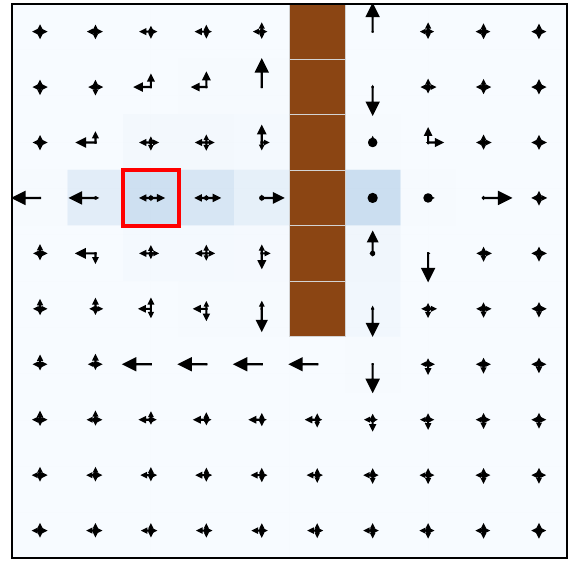}
\end{minipage}
\begin{minipage}[t!]{0.195\textwidth}
    \centering
    \includegraphics[width=0.98\linewidth]{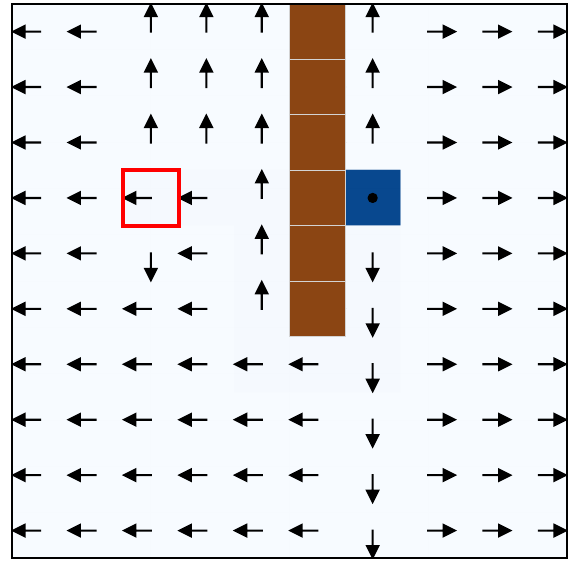}
\end{minipage}
\caption{From left to right: (a) stochastic $\pi_E$ in $\cM$ with $\gamma=0.9$,
(b)-(c) $\pi^{\text{MIMIC}}_{\cM,\pi_E,\cM'}$ and
$\pi^{\text{BIRL}}_{\cM}$ in $\cM'$ with $\gamma'=0.7$, (d)-(e)
$\pi^{\text{MIMIC}}_{\cM,\pi_E,\cM'}$ and $\pi^{\text{BIRL}}_{\cM}$
in $\cM'$ with $\gamma'=0.99$. }
\label{fig: experiments BIRL}
 \end{figure}
For OPT, we used the deterministic policy $\pi_E$ shown in Fig. \ref{fig:
experiments OPT} (a), which moves right for a few steps before stopping.
We found that the policy $\pi^{\text{OPT}}_\cM$ obtained through our method
(i.e., by planning in $\cM'$ with $r^{\text{OPT}}_{\cM,\pi_E}$, see Eq. \ref{eq:
avg case rewritten}), is quite similar to the policy
$\pi^{\text{MIMIC}}_{\cM,\pi_E,\cM'}$, defined as the policy whose state-action
occupancy measure in $\cM'$ is closest to that of the expert $d_{\cM,\pi_E}$ in
$\cM$ (details in Appendix \ref{apx: match occupancy measure}).
Intuitively, both $\pi^{\text{OPT}}_\cM$ and
$\pi^{\text{MIMIC}}_{\cM,\pi_E,\cM'}$ aim to ``re-play'' in $\cM'$ the actions
demonstrated by the expert in $\cM$ (as shown by Fig. \ref{fig: experiments OPT}
(b)-(c) and (d)-(e)), with the difference that
$\pi^{\text{MIMIC}}_{\cM,\pi_E,\cM'}$ weights states in $\cS_{\cM,\pi_E}$ by the
expert's discounted visitation frequency $d_{\cM,\pi_E}$, while
$\pi^{\text{OPT}}_\cM$ treats all visited states equally (see Appendix \ref{apx:
exps more difference}).
For MCE and BIRL, the results are quite similar. Here, we considered as $\pi_E$
the stochastic policy represented in Fig. \ref{fig: experiments BIRL} (a).
As shown by Fig. \ref{fig: experiments BIRL} (b)-(c) and (d)-(e) (and by Fig.
\ref{fig: experiments MCE} (b)-(c) and (d)-(e) in Appendix \ref{apx: compare
with MCE}), both $\pi^{\text{MCE}}_{\cM}$ and $\pi^{\text{BIRL}}_{\cM}$ resemble
$\pi^{\text{MIMIC}}_{\cM,\pi_E,\cM'}$ in trying to re-play the actions
\emph{mostly} played by the expert in its support, with the difference that, as
for OPT, they treat the states in $\cS_{\cM,\pi_E}$ uniformly.

\section{Conclusion}\label{sec: conclusion}

In this paper, we proposed a novel ``on-average'' approach for generalizing
expert behavior and derived three efficient IRL algorithms for estimating the
corresponding reward centroids from offline data. We also provided numerical
simulations to illustrate the resulting policies' behavior.

\textbf{Limitations and future directions.}~~Our approach is currently limited
to tabular settings. A key direction for future work is to identify structural
assumptions that enable scaling to large or continuous spaces. Broader empirical
studies could also clarify in which application domains our method generalizes
well and when additional structure is needed.
%

\bibliographystyle{plain}
\bibliography{refs.bib}

\begin{thebibliography}{10}

\bibitem{abbeel2004apprenticeship}
Pieter Abbeel and Andrew~Y. Ng.
\newblock Apprenticeship learning via inverse reinforcement learning.
\newblock In {\em International Conference on Machine Learning 21 (ICML)}, 2004.

\bibitem{amin2016resolving}
Kareem Amin and Satinder Singh.
\newblock Towards resolving unidentifiability in inverse reinforcement learning, 2016.

\bibitem{arora2020survey}
Saurabh Arora and Prashant Doshi.
\newblock A survey of inverse reinforcement learning: Challenges, methods and progress.
\newblock {\em Artificial Intelligence}, 297:103500, 2021.

\bibitem{bajgar2024walking}
Ondrej Bajgar, Alessandro Abate, Konstantinos Gatsis, and Michael~A. Osborne.
\newblock Walking the values in bayesian inverse reinforcement learning.
\newblock In {\em Conference on Uncertainty in Artificial Intelligence 40 (UAI)}, 2024.

\bibitem{bernstein2011matrixmathematics}
Dennis~S. Bernstein.
\newblock {\em Matrix Mathematics: Theory, Facts, and Formulas}.
\newblock Princeton University Press, 2011.

\bibitem{boothby1986manifolds}
William~Munger Boothby.
\newblock {\em {An introduction to differentiable manifolds and Riemannian geometry}}.
\newblock Academic Press, Orlando, FL, 1986.

\bibitem{boyd2004convex}
Stephen Boyd and Lieven Vandenberghe.
\newblock {\em Convex optimization}.
\newblock Cambridge university press, 2004.

\bibitem{Brantley2020DisagreementRegularizedIL}
Kiant{\'e} Brantley, Wen Sun, and Mikael Henaff.
\newblock Disagreement-regularized imitation learning.
\newblock In {\em International Conference on Learning Representations}, 2020.

\bibitem{cao2021identifiability}
Haoyang Cao, Samuel Cohen, and Lukasz Szpruch.
\newblock Identifiability in inverse reinforcement learning.
\newblock In {\em Advances in Neural Information Processing Systems 34 (NeurIPS)}, pages 12362--12373, 2021.

\bibitem{choi2011mapbirl}
Jaedeug Choi and Kee-eung Kim.
\newblock Map inference for bayesian inverse reinforcement learning.
\newblock In {\em Advances in Neural Information Processing Systems 24 (NeurIPS)}, 2011.

\bibitem{cohn2013measure}
D.L. Cohn.
\newblock {\em Measure Theory: Second Edition}.
\newblock Springer New York, 2013.

\bibitem{desai2020IfoTX}
Siddharth Desai, Ishan Durugkar, Haresh Karnan, Garrett Warnell, Josiah Hanna, and Peter Stone.
\newblock An imitation from observation approach to transfer learning with dynamics mismatch.
\newblock In {\em Advances in Neural Information Processing Systems 33 (NeurIPS)}, pages 3917--3929, 2020.

\bibitem{finn2016guided}
Chelsea Finn, Sergey Levine, and Pieter Abbeel.
\newblock Guided cost learning: Deep inverse optimal control via policy optimization.
\newblock In {\em International Conference on Machine Learning 33 (ICML)}, pages 49--58, 2016.

\bibitem{Fu2017LearningRR}
Justin Fu, Katie Luo, and Sergey Levine.
\newblock Learning robust rewards with adversarial inverse reinforcement learning.
\newblock In {\em International Conference on Learning Representations 5 (ICLR)}, 2017.

\bibitem{garg2021IQlearn}
Divyansh Garg, Shuvam Chakraborty, Chris Cundy, Jiaming Song, and Stefano Ermon.
\newblock Iq-learn: Inverse soft-q learning for imitation.
\newblock In {\em Advances in Neural Information Processing Systems 34 (NeurIPS)}, pages 4028--4039, 2021.

\bibitem{haarnoja2017rldeepenergypolicies}
Tuomas Haarnoja, Haoran Tang, Pieter Abbeel, and Sergey Levine.
\newblock Reinforcement learning with deep energy-based policies.
\newblock In {\em International Conference on Machine Learning 34 (ICML)}, volume~70, pages 1352--1361, 2017.

\bibitem{hefferon2009linearalgebra}
Jim Hefferon.
\newblock {\em Linear Algebra}.
\newblock Department of Mathematics \& Applied Mathematics/Virginia Commonwealth University, 2009.

\bibitem{ho2016generativeadversarialimitationlearning}
Jonathan Ho and Stefano Ermon.
\newblock Generative adversarial imitation learning.
\newblock In {\em Advances in Neural Information Processing Systems 29 (NeurIPS)}, 2016.

\bibitem{kim2021rewardidentification}
Kuno Kim, Shivam Garg, Kirankumar Shiragur, and Stefano Ermon.
\newblock Reward identification in inverse reinforcement learning.
\newblock In {\em International Conference on Machine Learning 38 (ICML)}, pages 5496--5505, 2021.

\bibitem{kuderer2015driving}
Markus Kuderer, Shilpa Gulati, and Wolfram Burgard.
\newblock Learning driving styles for autonomous vehicles from demonstration.
\newblock In {\em IEEE International Conference on Robotics and Automation 32 (ICRA)}, pages 2641--2646, 2015.

\bibitem{lazzati2025rel}
Filippo Lazzati and Alberto~Maria Metelli.
\newblock On the partial identifiability in reward learning: Choosing the best reward, 2025.

\bibitem{lazzati2024offline}
Filippo Lazzati, Mirco Mutti, and Alberto~Maria Metelli.
\newblock Offline inverse rl: New solution concepts and provably efficient algorithms.
\newblock In {\em International Conference on Machine Learning 41 (ICML)}, 2024.

\bibitem{lee2019introduction}
J.M. Lee.
\newblock {\em Introduction to Riemannian Manifolds}.
\newblock Springer International Publishing, 2019.

\bibitem{matas2018simtoreal}
Jan Matas, Stephen James, and Andrew~J. Davison.
\newblock Sim-to-real reinforcement learning for deformable object manipulation.
\newblock In {\em Conference on Robot Learning 2 (CoRL)}, volume~87, pages 734--743, 2018.

\bibitem{metelli2023towards}
Alberto~Maria Metelli, Filippo Lazzati, and Marcello Restelli.
\newblock Towards theoretical understanding of inverse reinforcement learning.
\newblock In {\em International Conference on Machine Learning 40 (ICML)}, pages 24555--24591, 2023.

\bibitem{metelli2021provably}
Alberto~Maria Metelli, Giorgia Ramponi, Alessandro Concetti, and Marcello Restelli.
\newblock Provably efficient learning of transferable rewards.
\newblock In {\em International Conference on Machine Learning 38 (ICML)}, volume 139, pages 7665--7676, 2021.

\bibitem{michini2012improving}
Bernard Michini and Jonathan~P. How.
\newblock Improving the efficiency of bayesian inverse reinforcement learning.
\newblock In {\em IEEE International Conference on Robotics and Automation 29 (ICRA)}, pages 3651--3656, 2012.

\bibitem{ng1999shaping}
Andrew~Y. Ng, Daishi Harada, and Stuart~J. Russell.
\newblock Policy invariance under reward transformations: Theory and application to reward shaping.
\newblock In {\em International Conference on Machine Learning 16 (ICML)}, pages 278--287, 1999.

\bibitem{ng2000algorithms}
Andrew~Y. Ng and Stuart~J. Russell.
\newblock Algorithms for inverse reinforcement learning.
\newblock In {\em International Conference on Machine Learning 17 (ICML)}, pages 663--670, 2000.

\bibitem{ni2021firl}
Tianwei Ni, Harshit Sikchi, Yufei Wang, Tejus Gupta, Lisa Lee, and Ben Eysenbach.
\newblock f-irl: Inverse reinforcement learning via state marginal matching.
\newblock In {\em Conference on Robot Learning (CoRL)}, volume 155, pages 529--551, 2021.

\bibitem{openai2019solvingrubikscuberobot}
OpenAI, Ilge Akkaya, Marcin Andrychowicz, Maciek Chociej, Mateusz Litwin, Bob McGrew, Arthur Petron, Alex Paino, Matthias Plappert, Glenn Powell, Raphael Ribas, Jonas Schneider, Nikolas Tezak, Jerry Tworek, Peter Welinder, Lilian Weng, Qiming Yuan, Wojciech Zaremba, and Lei Zhang.
\newblock Solving rubik's cube with a robot hand, 2019.

\bibitem{osa2018IL}
Takayuki Osa, Joni Pajarinen, Gerhard Neumann, J.~Andrew Bagnell, Pieter Abbeel, and Jan Peters.
\newblock An algorithmic perspective on imitation learning.
\newblock {\em Foundations and Trends® in Robotics}, 7(1-2):1--179, 2018.

\bibitem{poiani2024inversereinforcementlearningsuboptimal}
Riccardo Poiani, Gabriele Curti, Alberto~Maria Metelli, and Marcello Restelli.
\newblock Inverse reinforcement learning with sub-optimal experts, 2024.

\bibitem{pomerlau1988alvinn}
Dean~A. Pomerleau.
\newblock Alvinn: An autonomous land vehicle in a neural network.
\newblock In {\em Advances in Neural Information Processing Systems 1 (NeurIPS)}, 1988.

\bibitem{puterman1994markov}
Martin~Lee Puterman.
\newblock {\em {M}arkov Decision Processes: Discrete Stochastic Dynamic Programming}.
\newblock John Wiley \& Sons, Inc., 1994.

\bibitem{Ramachandran2007birl}
Deepak Ramachandran and Eyal Amir.
\newblock Bayesian inverse reinforcement learning.
\newblock In {\em International Joint Conference on Artifical Intelligence 20 (IJCAI)}, pages 2586--2591, 2007.

\bibitem{ratliff2006maximum}
Nathan~D. Ratliff, J.~Andrew Bagnell, and Martin~A. Zinkevich.
\newblock Maximum margin planning.
\newblock In {\em International Conference on Machine Learning 23 (ICML)}, pages 729--736, 2006.

\bibitem{reddy2019sqil}
Siddharth Reddy, Anca~D. Dragan, and Sergey Levine.
\newblock Sqil: Imitation learning via reinforcement learning with sparse rewards, 2019.

\bibitem{rolland2022identifiability}
Paul Rolland, Luca Viano, Norman Sch\"{u}rhoff, Boris Nikolov, and Volkan Cevher.
\newblock Identifiability and generalizability from multiple experts in inverse reinforcement learning.
\newblock In {\em Advances in Neural Information Processing Systems 35 (NeurIPS)}, pages 550--564, 2022.

\bibitem{russell1998learning}
Stuart Russell.
\newblock Learning agents for uncertain environments (extended abstract).
\newblock In {\em Proceedings of the Eleventh Annual Conference on Computational Learning Theory 11 (COLT)}, pages 101--103, 1998.

\bibitem{sasaki2018sample}
Fumihiro Sasaki, Tetsuya Yohira, and Atsuo Kawaguchi.
\newblock Sample efficient imitation learning for continuous control.
\newblock In {\em International Conference on Learning Representations 7 (ICLR)}, 2019.

\bibitem{schlaginhaufen2023identifiability}
Andreas Schlaginhaufen and Maryam Kamgarpour.
\newblock Identifiability and generalizability in constrained inverse reinforcement learning.
\newblock In {\em International Conference on Machine Learning 40 (ICML)}, volume 202, pages 30224--30251, 2023.

\bibitem{schlaginhaufen2024transferability}
Andreas Schlaginhaufen and Maryam Kamgarpour.
\newblock Towards the transferability of rewards recovered via regularized inverse reinforcement learning, 2024.

\bibitem{silver2010terrain}
David Silver, J.~Andrew Bagnell, and Anthony Stentz.
\newblock Learning from demonstration for autonomous navigation in complex unstructured terrain.
\newblock {\em The International Journal of Robotics Research}, 29(12):1565--1592, 2010.

\bibitem{skalse2023misspecificationinversereinforcementlearning}
Joar Skalse and Alessandro Abate.
\newblock Misspecification in inverse reinforcement learning.
\newblock In {\em AAAI Conference on Artificial Intelligence 37 (AAAI)}, 2023.

\bibitem{skalse2023invariancepolicyoptimisationpartial}
Joar Skalse, Matthew Farrugia-Roberts, Stuart Russell, Alessandro Abate, and Adam Gleave.
\newblock Invariance in policy optimisation and partial identifiability in reward learning.
\newblock In {\em International Conference on Machine Learning 40 (ICML)}, volume 202, pages 32033--32058, 2023.

\bibitem{stoll2007integration}
Michael Stoll.
\newblock Integration and manifolds, Fall 2007.

\bibitem{sutton2018RL}
Richard~S. Sutton and Andrew~G. Barto.
\newblock {\em Reinforcement Learning: An Introduction}.
\newblock A Bradford Book, 2018.

\bibitem{viano2021robust}
Luca Viano, Yu-Ting Huang, Parameswaran Kamalaruban, Adrian Weller, and Volkan Cevher.
\newblock Robust inverse reinforcement learning under transition dynamics mismatch.
\newblock In {\em Advances in Neural Information Processing Systems 34 (NeurIPS)}, pages 25917--25931, 2021.

\bibitem{wang2019ILviaexpertsupportestimation}
Ruohan Wang, Carlo Ciliberto, Pierluigi~Vito Amadori, and Yiannis Demiris.
\newblock Random expert distillation: Imitation learning via expert policy support estimation.
\newblock In {\em International Conference on Machine Learning 36 (ICML)}, volume~97, pages 6536--6544, 2019.

\bibitem{wulfmeier2016maximumentropydeepirl}
Markus Wulfmeier, Peter Ondruska, and Ingmar Posner.
\newblock Maximum entropy deep inverse reinforcement learning, 2016.

\bibitem{xie2021policyfinetuning}
Tengyang Xie, Nan Jiang, Huan Wang, Caiming Xiong, and Yu~Bai.
\newblock Policy finetuning: Bridging sample-efficient offline and online reinforcement learning.
\newblock In {\em Advances in Neural Information Processing Systems 34 (NeurIPS)}, pages 27395--27407, 2021.

\bibitem{yu2019healthcare}
Chao Yu, Jiming Liu, and Hongyi Zhao.
\newblock Inverse reinforcement learning for intelligent mechanical ventilation and sedative dosing in intensive care units.
\newblock {\em BMC Medical Informatics and Decision Making}, 19, 2019.

\bibitem{Ziebart2010ModelingPA}
Brian~D. Ziebart.
\newblock Modeling purposeful adaptive behavior with the principle of maximum causal entropy, 2010.

\bibitem{ziebart2008maxent}
Brian~D. Ziebart, Andrew Maas, J.~Andrew Bagnell, and Anind~K. Dey.
\newblock Maximum entropy inverse reinforcement learning.
\newblock In {\em AAAI Conference on Artificial Intelligence 23 (AAAI)}, volume~3, pages 1433--1438, 2008.

\end{thebibliography}

\newpage

\setlength{\abovedisplayskip}{6pt}
\setlength{\belowdisplayskip}{6pt}
\setlength{\abovedisplayshortskip}{6pt}
\setlength{\belowdisplayshortskip}{6pt}
\setlength{\textfloatsep}{15pt}
\thinmuskip=4.0mu
\medmuskip=4.0mu
\thickmuskip=4.0mu

\appendix

\section{Additional Related Work}\label{apx: additional related work}

In this appendix, we discuss additional related work that was omitted from the
main paper due to space constraints.

In Section~\ref{sec: existing methods}, we categorized existing IRL methods into
three groups, based on the strategy used to address the partial identifiability
of the expert's reward function. Here, we introduce a fourth category of
methods, which we refer to as “heuristic”
methods~\cite{ng2000algorithms,ratliff2006maximum}. These are closely related to
“best-case” methods. Specifically, heuristic methods select a reward function
from the feasible set using an application-specific heuristic, which is
generally not justified beyond the given context. For example,
\cite{ng2000algorithms,ratliff2006maximum} employ the heuristic of \emph{margin
maximization} to select a reward function from the feasible set, a choice mainly
justified within the IL setting.

The reward function obtained through OPT (see Eq. \ref{eq: centroid OPT}) can be
computed using knowledge of the expert's support $\cS_{\cM,\pi_E}$, making it
similar to expert support estimation methods in IL found in the
literature \cite{wang2019ILviaexpertsupportestimation,Brantley2020DisagreementRegularizedIL}.
These works present IL algorithms based on reward functions derived directly
from $\cS_{\cM,\pi_E}$. However, we emphasize that these approaches are
specifically tailored to IL and typically involve implicit structural
assumptions to scale to large or continuous state-action spaces (see also
\cite{reddy2019sqil,sasaki2018sample}).

Finally, we note that the reward derived for MCE in Eq.\eqref{eq: centroid MCE}
was previously discussed in Section 4 of \cite{Fu2017LearningRR} as a promising
approach for IL. However, it was shown to potentially yield errors when applied
to novel environments. This observation led~\cite{Fu2017LearningRR} to introduce
additional assumptions aimed at reducing the size of the feasible set to obtain
a more transferable reward. In contrast, our work provides a principled
justification for this reward based on the minimization of average error in the
absence of additional assumptions.

\section{Additional Notation}\label{apx: additional notation}

In this section, we provide additional notation that will be useful for the
presentation of the contents in the appendix. In addition, in Appendix \ref{apx:
rewriting fs}, we show the equivalence between $\cR^{\text{OPT}}_{\cM,\pi_E}$,
$\cR^{\text{MCE}}_{\cM,\pi_E}$, $\cR^{\text{BIRL}}_{\cM,\pi_E}$ and
$\cR^{\text{OPT},\cS_{\cM,\pi_E}}_{\cM,\pi_E}$,
$\cR^{\text{MCE},\cS_{\cM,\pi_E}}_{\cM,\pi_E}$,
$\cR^{\text{BIRL},\cS_{\cM,\pi_E}}_{\cM,\pi_E}$, while in Appendix \ref{apx:
explicit fs} we provide useful results that permit to rewrite the feasible sets
$\cR^{\text{OPT}}_{\cM,\pi_E}$, $\cR^{\text{MCE}}_{\cM,\pi_E}$,
$\cR^{\text{BIRL}}_{\cM,\pi_E}$ more explicitly.

Given a policy $\pi$ and a subset of states
$\overline{\cS}\subseteq\cS$, we define as $[\pi]_{\overline{\cS}}\coloneqq \{
\pi'\in\Delta_\cS^\cA\,|\, \forall s\in\overline{\cS}:\, \pi'(\cdot|s) =
\pi(\cdot|s) \wedge \forall s\notin\overline{\cS},\exists a\in\cA:\, \pi'(a|s)
=1\}$ the set of policies that in $\overline{\cS}$ coincide with $\pi$, and
outside $\overline{\cS}$ are deterministic.

Let $\cM=\tuple{\cS,\cA,s_0,p,\gamma}$ be an \MDPr, and let $V\in\RR^S$ be an
arbitrary vector and $A\in\RR^{S\times (A-1)}$ an arbitrary matrix. Then, for
any deterministic policy $\pi$, we define the linear operator
$T_{p,\pi}:\RR^S\times \RR^{S\times (A-1)}\to \RR^{S\times A}$ as the matrix
that takes in input a pair $(V,A)$ and return a reward $r=T_{p,\pi}(V,A)$ in the
following way. For all $(s,a)\in\SA$:
\begin{align}\label{eq: def T p pi}
  r(s,a)=V(s)-\gamma\sum\limits_{s'\in\cS}p(s'|s,a)V(s')+
  \begin{cases}
    0&\text{if }a=\pi(s),\\
    A(s,a)&\text{otherwise}.
  \end{cases}
\end{align}
A property that will be useful in various proofs is that:
$|\text{det}(T_{p,\pi})|=|\text{det}(W_{p,\pi})|$ (see Proposition \ref{prop:
rel T W}).

Similarly to $T_{p,\pi}$, for any $SA$-dimensional vector $\eta\in\RR^{SA}$, we
introduce the linear map $U_{p,\eta}:\RR^S\to\RR^{SA}$ that takes in input a
vector $V\in\RR^S$ and returns a reward $r=U_{p,\eta}(V)$ such that, for all
$(s,a)\in\SA$:
\begin{align}\label{eq: def U eta}
  r(s,a) = V(s)-\gamma\sum_{s'\in\cS}p(s'|s,a)V(s')
  +\eta(s,a).
\end{align}

For any \MDPr $\cM=\tuple{\cS,\cA,s_0,p,\gamma}$, it is useful to define a
partition of the set of all possible rewards $\RR^{SA}=\sR^{1}_p\cup\sR^{>1}_p$
into the set of rewards that induce in $\cM$ a unique optimal policy from every
state $\sR^{1}_p$, and those that do not $\sR^{>1}_p$. Formally:
\begin{align}
  \mathscr{R}^{1}_p&\coloneqq
  \Bigc{r\in\RR^{SA}\;\Big|\;|\Pi^*(r,p)|=1
  },\label{eq: def R1}\\
  \mathscr{R}^{>1}_p&\coloneqq
  \Bigc{r\in\RR^{SA}\;\Big|\;|\Pi^*(r,p)|>1
  }\label{eq: def R greater 1}.
\end{align}

\subsection{Equivalence among Sets of Rewards}
\label{apx: rewriting fs}

We now show that the feasible sets
$\cR^{\text{OPT}}_{\cM,\pi_E}$, $\cR^{\text{MCE}}_{\cM,\pi_E}$,
$\cR^{\text{BIRL}}_{\cM,\pi_E}$ coincide, respectively, with the sets
$\cR^{\text{OPT},\cS_{\cM,\pi_E}}_{\cM,\pi_E}$,
$\cR^{\text{MCE},\cS_{\cM,\pi_E}}_{\cM,\pi_E}$,
$\cR^{\text{BIRL},\cS_{\cM,\pi_E}}_{\cM,\pi_E}$.
First, we write $\cR^{\text{OPT}}_{\cM,\pi_E}$, $\cR^{\text{MCE}}_{\cM,\pi_E}$,
$\cR^{\text{BIRL}}_{\cM,\pi_E}$ formally:
\begin{align}
    &\cR^{\text{OPT}}_{\cM,\pi_E}\coloneqq\bigc{r\in\RR^{SA}\,\big|\, \pi_E\in\argmax_\pi V^\pi(s_0;p,r)},
    \\
    &\cR^{\text{MCE}}_{\cM,\pi_E}\coloneqq\bigc{r\in\RR^{SA}\,\big|\, \pi_E=\argmax_\pi V^\pi_\lambda (s_0;p,r)},
    \\
    &\cR^{\text{BIRL}}_{\cM,\pi_E}\coloneqq\bigc{r\in\RR^{SA}\,\big|\, \pi_E(a|s) \propto e^{\frac{1}{\beta}
    Q^*(s,a;p,r)}
    \; \forall s\in\cS^{\cM,\pi_E},\forall a\in\cA},
\end{align}
from which the equality $\cR^{\text{BIRL}}_{\cM,\pi_E}=
\cR^{\text{BIRL},\cS_{\cM,\pi_E}}_{\cM,\pi_E}$ is immediate.
Regarding OPT, the equality $\cR^{\text{OPT}}_{\cM,\pi_E}=
\cR^{\text{OPT},\cS_{\cM,\pi_E}}_{\cM,\pi_E}$ follows directly from Lemma E.1 of
\cite{lazzati2024offline}.
Finally, w.r.t. MCE, simply observe that the constraint $\pi_E=\argmax_\pi
V^\pi_\lambda (s_0;p,r)$ holds only for the states $s\in\cS_{\cM,\pi_E}$, i.e.,
those reachable from $s_0$ playing $\pi_E$ in $\cM$. Thus, the reward functions
$r\in \cR^{\text{MCE}}_{\cM,\pi_E}$ can take on arbitrary values in $s'\notin
\cS_{\cM,\pi_E}$. Then, the equality $\cR^{\text{MCE}}_{\cM,\pi_E}=
\cR^{\text{MCE},\cS_{\cM,\pi_E}}_{\cM,\pi_E}$ follows by using Eq. (5) of
\cite{cao2021identifiability} (or, equivalently, Eq. (6) of
\cite{haarnoja2017rldeepenergypolicies}).

\subsection{Rewriting the Feasible Sets Explicitly}
\label{apx: explicit fs}

The propositions below permit to rewrite the feasible sets
$\cR^{\text{OPT}}_{\cM,\pi_E}$, $\cR^{\text{MCE}}_{\cM,\pi_E}$,
$\cR^{\text{BIRL}}_{\cM,\pi_E}$ more explicitly when $\cS=\cS_{\cM,\pi_E}$.
Observe that, by making explicit the number of degrees of freedom in these sets,
they also allow us to realize that
$\cR^{\text{OPT},\cS}_{\cM,\pi_E}$ is $SA$-dimensional, while both
$\cR^{\text{MCE},\cS}_{\cM,\pi_E}$ and
$\cR^{\text{BIRL},\cS}_{\cM,\pi_E}$ are $S$-dimensional.

\begin{prop}\label{prop: fs OPT explicit}
  Let $\cM=\tuple{\cS,\cA,s_0,p,\gamma}$ be an \MDPr, and let $\pi$ be any
  policy. Then, for all and only the rewards $r$ in the set
  $\cR^{\text{OPT},\cS}_{\cM,\pi}$ there exist two vectors $V\in\RR^S$ and
  $A\in(-\infty,0]^{S\times A}$ such that, for any $(s,a)\in\SA$:
  \begin{align*}
    r(s,a)=V(s)-\gamma\sum_{s'}p(s'|s,a)V(s')+\begin{cases}
      0&\text{ if }\pi(a|s)>0,\\
      A(s,a)&\text{ otherwise}.
    \end{cases}
  \end{align*}
  Moreover, $V$ and $A$ coincide with, respectively, $V^*(\cdot;p,r)$ and
  $A^*(\cdot,\cdot;p,r)$.
\end{prop}
\begin{proof}
  See Lemma 3.2 of \cite{metelli2021provably}.
\end{proof}

\begin{prop}\label{prop: fs MCE explicit}
  Let $\cM=\tuple{\cS,\cA,s_0,p,\gamma}$ be an \MDPr, and let $\pi$ be any
  stochastic policy. Then, for all and only the rewards $r$ in the set
  $\cR^{\text{MCE},\cS}_{\cM,\pi}$ there exists a vector $V\in\RR^S$ such that,
  for any $(s,a)\in\SA$:
  \begin{align*}
    r(s,a)=V(s)-\gamma\sum_{s'}p(s'|s,a)V(s')+\lambda \log \pi(a|s).
  \end{align*}
  Moreover, $V(s)$ coincides with $V^*_\lambda(s;p,r)$, and $\lambda \log
  \pi(a|s)$ coincides with $A^*_\lambda(s,a;p,r)$.
\end{prop}
\begin{proof}
  See Theorem 1 of \cite{cao2021identifiability}. 
\end{proof}

\begin{prop}\label{prop: fs BIRL explicit}
  Let $\cM=\tuple{\cS,\cA,s_0,p,\gamma}$ be an \MDPr, and let $\pi$ be any
  stochastic policy. Then, for all and only the rewards $r$ in the set
  $\cR^{\text{BIRL},\cS}_{\cM,\pi}$ there exists a vector $V\in\RR^S$ such that,
  for any $(s,a)\in\SA$:
  \begin{align*}
    r(s,a)=V(s)-\gamma\sum_{s'}p(s'|s,a)V(s')+\beta 
    \log \frac{\pi(a|s)}{\max_{a'\in\cA}\pi(a'|s)}.
  \end{align*}
  Moreover, $V(s)$ coincides with $V^*(s;p,r)$, and $\beta 
    \log (\pi(a|s)/\max_{a'\in\cA}\pi(a'|s))$ coincides with $A^*(s,a;p,r)$.
\end{prop}
\begin{proof}
  We begin by showing that every reward in the set can be written in this way.
  Starting from the definition of $\cR^{\text{BIRL},\cS}_{\cM,\pi}$, for any
  $r\in \cR^{\text{BIRL},\cS}_{\cM,\pi}$, for every $(s,a)\in\SA$, it holds
  that:
  \begin{align*}
    \pi(a|s)&=\frac{e^{\frac{1}{\beta} Q^*(s,a;p,r)}}{
      \sum_{a'}e^{\frac{1}{\beta} Q^*(s,a';p,r)}
    }\\
    &=\frac{e^{\frac{1}{\beta}(Q^*(s,a;p,r)-V^*(s;p,r))}}{
      \sum_{a'}e^{\frac{1}{\beta}(Q^*(s,a';p,r)-V^*(s;p,r))}
    }\\
    &=
    \frac{e^{\frac{1}{\beta} A^*(s,a;p,r)}
    }{\sum_{a'}e^{\frac{1}{\beta} A^*(s,a';p,r)}}.
  \end{align*}
  Then, by rearranging and taking the logarithm on both parts, we get:
  \begin{align*}
    \frac{1}{\beta} A^*(s,a;p,r)=
    \log \pi(a|s)+
    \log \sum_{a'\in\cA}
    e^{\frac{1}{\beta} A^*(s,a';p,r)}.
  \end{align*}
  Now, by applying the following formula:
  \begin{align*}
    \beta \log \sum_{a'\in\cA}e^{\frac{1}{\beta} A^*(s,a';p,r)} =
    - \beta\log\max_{a'\in\cA}\pi(a'|s),
  \end{align*}
  which has been derived in the proof of Lemma C.2 of
  \cite{skalse2023invariancepolicyoptimisationpartial}, we obtain:
  \begin{align*}
    A^*(s,a;p,r)=
    \beta\log \frac{\pi(a|s)}{\max_{a'}\pi(a'|s)}.
  \end{align*}
  By using the definition of optimal advantage function:
  $A^*(s,a;p,r)=Q^*(s,a;p,r)-V^*(s;p,r)$, and also the Bellman optimality
  equation \cite{puterman1994markov}: $Q^*(s,a;p,r) = r(s,a)+
  \gamma\sum_{s'}p(s'|s,a)V^*(s';p,r)$, rearranging, and using symbol $V$ for
  $V^*(\cdot;p,r)$, we get the result:
  \begin{align*}
    r(s,a) = V(s) - \gamma\sum_{s'\in\cS}p(s'|s,a)V(s') +
    \beta\log \frac{\pi(a|s)}{\max_{a'}\pi(a'|s)}.
  \end{align*}
  To show that if a reward function $r$ can be written in the form described in
  the statement of the proposition (for some $V\in\RR^S$) then it belongs to
  $\cR^{\text{BIRL},\cS}_{\cM,\pi}$, observe that, in every $s\in\cS$, there is
  at least one action $a$ for which (observe that $\log
  (\pi(a|s)/\max_{a'}\pi(a'|s))=0$):
  \begin{align*}
    r(s,a) = V(s) - \gamma\sum_{s'\in\cS}p(s'|s,a)V(s').
  \end{align*}
  Therefore, through a reasoning analogous to that of Lemma 3.2 of
  \cite{metelli2021provably}, we realize that $V$ coincides with
  $V^*(\cdot;p,r)$, and, in addition:
  \begin{align*}
    A^*(s,a;p,r)=\beta\log \frac{\pi(a|s)}{\max_{a'}\pi(a'|s)}.
  \end{align*}
  Thus, the proof can be concluded by reversing the passages in the first part
  of this proof, to show that this implies that:
  \begin{align*}
    \pi(a|s)=\frac{e^{\frac{1}{\beta} Q^*(s,a;p,r)}}{
      \sum_{a'}e^{\frac{1}{\beta} Q^*(s,a';p,r)}
    },
  \end{align*}
  and so $r\in \cR^{\text{BIRL},\cS}_{\cM,\pi}$.
\end{proof}

\section{Additional Results and Proofs for Section \ref{sec: existing methods}}
\label{apx: proofs existing methods}

In this appendix, we collect various results and proofs for Section \ref{sec:
existing methods}. Specifically, in Appendix \ref{apx: unbounded worst case}, we
show the importance of restricting the feasible set to a bounded set in the
``worst-case'' approaches.
Next, we present the proofs of Propositions \ref{prop: centroid prior
OPT}-\ref{prop: counterexample MCE} (Appendix \ref{apx: no same vol}), and we
explain how to extend the result for OPT to state-only rewards (Appendix
\ref{apx: uniform prior state-only rewards}). Finally, in Appendix \ref{apx:
uniform transition models}, we extend these results to the setting where we are
given a uniform distribution over transition models.

\subsection{On the Importance of $\fR$ in the ``Worst-case'' Approach}
\label{apx: unbounded worst case}

In this appendix, we show that, if we do not restrict the feasible set
$\cR^m_{\cM,\pi_E}$ in Eq. \eqref{eq: worst case} to a bounded set, e.g.,
$\cR^m_{\cM,\pi_E}\cap\fR$, then there are problem instances where any feasible
policy $\pi'\in\Pi_{c,k}$ achieves the objective value
$\max_{r\in\cR^m_{\cM,\pi_E}\cap\fR} \bigr{\max_{\pi\in\Pi_{c,k}}
V^{\pi}(s_0';p',r) - V^{\pi'}(s_0';p',r)} =+\infty$.
As a consequence, any policy $\pi'$ solves the optimization problem in
Eq. \eqref{eq: worst case}.

\begin{figure}[t!]
  \centering
  \begin{minipage}[t!]{0.48\textwidth}
    \centering
  \begin{tikzpicture}
    \node[state] at (0,0) (s) {$s_0$};
    \node[state] at (-2,+2) (s1) {$s_1$};
    \node[state] at (0,+2) (s2) {$s_2$};
    \node[state] at (+2,+2) (s3) {$s_3$};
    \draw (s) edge[->, solid, left] node{$a_1$} (s1);
    \draw (s) edge[->, solid, left] node{$a_2$} (s2);
    \draw (s) edge[->, solid, right] node{$a_3$} (s3);
    \draw (s1) edge[->, solid, loop above] node{$a$} (s1);
    \draw (s2) edge[->, solid, loop above] node{$a$} (s2);
    \draw (s3) edge[->, solid, loop above] node{$a$} (s3);
  \end{tikzpicture}
\end{minipage}
\hfill
\begin{minipage}[t!]{0.48\textwidth}
  \centering
  \begin{tikzpicture}
    \node[state] at (0,0) (s) {$s_0$};
    \node[state] at (-2,+2) (s1) {$s_1$};
    \node[state] at (0,+2) (s2) {$s_2$};
    \node[state] at (+2,+2) (s3) {$s_3$};
    \draw (s) edge[->, solid, loop left, left] node{$a_1$} (s);
    \draw (s) edge[->, solid, left] node{$a_2$} (s2);
    \draw (s) edge[->, solid, right] node{$a_3$} (s3);
    \draw (s1) edge[->, solid, loop above] node{$a$} (s1);
    \draw (s2) edge[->, solid, loop above] node{$a$} (s2);
    \draw (s3) edge[->, solid, loop above] node{$a$} (s3);
  \end{tikzpicture}
\end{minipage}
  \caption{ Problem instance considered in Appendix \ref{apx: unbounded worst
  case}. }
  \label{fig: overly pessimistic worst case}
\end{figure}
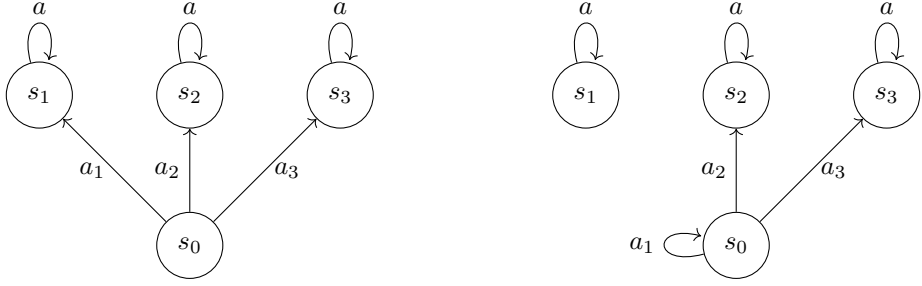

We focus on the problem instance exemplified in Fig. \ref{fig: overly
pessimistic worst case}.
Assume that the expert demonstrates policy $\pi_E$ with $\pi_E(a_1|s_0) = 1$ in
the MDP without reward $\cM$ depicted on the left of Fig. \ref{fig: overly
pessimistic worst case}, where $s_0$ is the initial state and $\gamma$ is any
scalar in$(0,1)$, and we want to ``generalize $\pi_E$'' to the MDP without
reward $\cM'$ on the right of Fig. \ref{fig: overly pessimistic worst case},
where $s_0$ is again the initial state and $\gamma'$ is any value in $[0,1)$.
Assume that there are no constraints (e.g., $c(s,a)=0$ and $k=1$) and that the
model of behavior is $m=$ OPT.
Then, the robust approach of \cite{lazzati2025rel} without restricting the
feasible set $\cR^{\text{OPT}}_{\cM,\pi_E}$ to the hypercube $\fR$ involves the
optimization of:
\begin{align*}
  \widehat{\pi}_{\text{W}}\in\argmin\limits_{\pi'\in\Pi} 
  \max\limits_{r\in\cR^{\text{OPT}}_{\cM,\pi_E}}
  \Bigr{V^*(s_0';p',r)
  - V^{\pi'}(s_0';p',r)}.
\end{align*}
Now, we show that, for any policy $\pi\in\Pi=\Delta^3$, there exists a reward
function $\overline{r}\in\cR^{\text{OPT}}_{\cM,\pi_E}$ that makes the objective
diverge:
\begin{align*}
  \max\limits_{\pi\in\Pi} V^{\pi}(s_0';p',\overline{r})
  - V^{\pi'}(s_0';p',\overline{r})\to+\infty.
\end{align*}
We begin with the deterministic policy $\pi$ such that $\pi(a_1|s_0)=1$. Observe
that the reward $r_1$ such that:
\begin{align*}
  r_1(s,a)=\begin{cases}
    0&\text{if }(s,a)=(s_0,a_1)\vee (s_0,a_2)\vee (s_0,a_3),\\
    +k'&\text{otherwise},
  \end{cases}
\end{align*}
for any $k'\ge 0$ belongs to the feasible set $\cR^{\text{OPT}}_{\cM,\pi_E}$
(because it makes all policies to be optimal), and provides a suboptimality
value for $\pi$ in $\cM'$ of:
\begin{align*}
  V^*(s_0';p',r_1) - V^{\pi}(s_0';p',r_1)=0+\gamma\frac{k'}{1-\gamma}-\frac{0}{1-\gamma}=\frac{\gamma}{1-\gamma}k'
  \to+\infty,
\end{align*}
for $k'\to\infty$.
If $\pi$ plays deterministically $a_2$, i.e., $\pi(a_2|s_0)=1$, then we find
reward $r_2$ such that:
\begin{align*}
  r_2(s,a)=\begin{cases}
    +k'&\text{if }(s,a)=(s_0,a_1)\vee (s_1,a),\\
    0&\text{otherwise},
  \end{cases}
\end{align*}
for any $k'\ge 0$ belongs to the feasible set $\cR^{\text{OPT}}_{\cM,\pi_E}$, and provides a suboptimality
value for $\pi$ in $\cM'$ of:
\begin{align*}
  V^*(s_0';p',r_2) - V^{\pi}(s_0';p',r_2)=\frac{k'}{1-\gamma}-\frac{0}{1-\gamma}=\frac{k'}{1-\gamma}
  \to+\infty,
\end{align*}
for $k'\to\infty$.
A similar derivation can be carried out also for the policy that plays $a_3$
deterministically.
Note that, using $r_2$, we can show the same result also for any stochastic
policy. Let $\pi$ be any stochastic policy in $\Delta^3$. Then, its
suboptimality w.r.t. $r_2$ is:
\begin{align*}
  V^*(s_0';p',r_2) - V^{\pi}(s_0';p',r_2)=\frac{k'}{1-\gamma}-\pi(a_1|s_0)\frac{k'}{1-\gamma}=\frac{1-\pi(a_1|s_0)}{1-\gamma}k'
  \to+\infty,
\end{align*}
for $k'\to\infty$, as $\pi(a_1|s_0)<1$ since $\pi$ is stochastic.

\subsection{The Hypercube is \emph{Biased} in the Space of Policies}
\label{apx: no same vol}

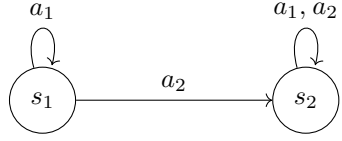
\begin{figure}[t!]
  \centering
  \begin{tikzpicture}[node distance=3.5cm]
    \node[state] at (0,0) (s1) {$s_1$};
    \node[state] at (3.5,0) (s2) {$s_2$};
    \draw (s1) edge[->, solid, loop above] node{$a_1$} (s1);
    \draw (s1) edge[->, solid, above] node{$a_2$} (s2);
    \draw (s2) edge[->, solid, loop above] node{$a_1,a_2$} (s2);
  \end{tikzpicture}
  \caption{\MDPr used in the proof of Proposition \ref{prop: counterexample
  OPT}.}
  \label{fig:  counterexample OPT}
\end{figure}

\counterexampleOPT*
\begin{proof}
  Consider the \MDPr $\cM$ depicted in Fig. \ref{fig: counterexample OPT}, and
  let $\pi_1$ be the policy that plays action $a_1$ in both states $s_1,s_2$.
  Note that there are four deterministic policies in this \MDPr, thus, we can
  prove the result by showing that:
  \begin{align*}
    \frac{\text{vol}\bigr{\fR\cap\cR^{\text{OPT},\cS}_{\cM,\pi_1}}}{\text{vol}\bigr{\fR}}
    \neq \frac{1}{4}.
  \end{align*}
  Since $\fR=[-1,+1]^{SA}=[-1,+1]^{4}$, then $\text{vol}\bigr{\fR}=2^4=16$, thus
  the result follows by showing that:
  \begin{align*}
    \text{vol}\bigr{\fR\cap\cR^{\text{OPT},\cS}_{\cM,\pi_1}}\neq 4.
  \end{align*}
  Let us now compute the integration region $\fR\cap\cR^{\text{OPT},\cS}_{\cM,\pi_1}$.
  For any $r\in\RR^{SA}$, let us abbreviate $r_{ij}\coloneqq r(s_i,a_j)$. We can
  compute the value and $Q$-functions of policy $\pi_1$ under any $r$ as:
  \begin{align*}
    &V^{\pi_1}(s_1;p,r)=\frac{r_{11}}{1-\gamma},\\
    &V^{\pi_1}(s_2;p,r)=\frac{r_{21}}{1-\gamma},\\
    &Q^{\pi_1}(s_1,a_2;p,r)=r_{12}+\gamma V^{\pi_1}(s_2;p,r),\\
    &Q^{\pi_1}(s_2,a_2;p,r)=r_{22}+\gamma V^{\pi_1}(s_2;p,r).
  \end{align*}
  Using the policy improvement theorem \cite{puterman1994markov}, we have that a
  reward $r$ belongs to $\cR^{\text{OPT},\cS}_{\cM,\pi_1}$ if and only if:
  \begin{align*}
    \begin{cases}
      &Q^{\pi_1}(s_1,a_2;p,r)\le V^{\pi_1}(s_1;p,r),\\
      &Q^{\pi_1}(s_2,a_2;p,r)\le V^{\pi_1}(s_2;p,r),
    \end{cases}
    \qquad\iff\qquad
    \begin{cases}
      &r_{12}+\gamma \frac{r_{21}}{1-\gamma}\le \frac{r_{11}}{1-\gamma},\\
      &r_{22}+\gamma \frac{r_{21}}{1-\gamma}\le \frac{r_{21}}{1-\gamma},
    \end{cases}
  \end{align*}
  which can be rearranged to:
  \begin{align*}
    \begin{cases}
      &r_{11}\ge (1-\gamma)r_{12}+\gamma r_{21},\\
      &r_{21}\ge r_{22}.
    \end{cases}    
  \end{align*}
  To simplify the calculations, let $\gamma\to1$, so that the integration region
  reduces to:
  \begin{align*}
    \fR\cap\cR^{\text{OPT},\cS}_{\cM,\pi_1}=\Bigc{
      r\in [-1,+1]^{4}\,|\,
      r_{11}\ge r_{21}\ge r_{22}
    }.
  \end{align*}
  Thus, we can write:
  \begin{align*}
    \int_{\fR\cap\cR^{\text{OPT},\cS}_{\cM,\pi_1}}
    dr_{11}dr_{12}dr_{21}dr_{22}&=
    \int_{-1}^{+1}dr_{12}
    \int_{-1}^{+1}\int_{r_{22}}^{+1}
    \int_{r_{21}}^{+1}
    dr_{11}dr_{21}dr_{22}\\
    &=2\cdot \int_{-1}^{+1}\int_{r_{22}}^{+1}
    \int_{r_{21}}^{+1}
    dr_{11}dr_{21}dr_{22}\\
    &=2\cdot \int_{-1}^{+1}\int_{r_{22}}^{+1}
    (1-r_{21})dr_{21}dr_{22}\\
    &=2\cdot \int_{-1}^{+1}\frac{1}{2}
    (1-r_{22})^2 dr_{22}\\
    &=2\cdot \frac{4}{3}=\frac{8}{3}.
  \end{align*}
  Since $\frac{8}{3}\neq 4$, this concludes the proof.
\end{proof}

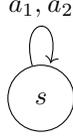
\begin{figure}[t!]
  \centering
  \begin{tikzpicture}[node distance=3.5cm]
    \node[state] at (0,0) (s1) {$s$};
    \draw (s1) edge[->, solid, loop above] node{$a_1,a_2$} (s1);
  \end{tikzpicture}
  \caption{\MDPr used in the proof of Proposition \ref{prop: counterexample
  MCE}.}
  \label{fig: counterexample MCE}
\end{figure}

\counterexampleMCE*
\begin{proof}
  We begin with MCE. Take $\lambda=1$ for simplicity.
  Consider the \MDPr with one state and two actions reported in Fig. \ref{fig:
  counterexample MCE}. By applying Proposition \ref{prop: fs MCE explicit}, we
  have that, for any policy $\pi$ and for any reward $r\in
  \cR^{\text{MCE},\cS}_{\cM,\pi}$, there exists a scalar $V$ such that:
  \begin{align*}
      &r(s,a_1) = (1-\gamma)V +\log\pi(a_1|s),\\
      &r(s,a_2) = (1-\gamma)V +\log\pi(a_2|s).
  \end{align*}
  By rearranging the second equation, we have:
  \begin{align*}
    V =\frac{r(s,a_2)-\log\pi(a_2|s)}{1-\gamma},
  \end{align*}
  and replacing this value of $V$ into the first one, we get:
  \begin{align*}
    r(s,a_1) = r(s,a_2)+\log\frac{\pi(a_1|s)}{1-\pi(a_1|s)}.
  \end{align*}
  Therefore, by denoting $x \coloneqq r(s,a_2)$, we
  realize that the integration region $\fR\cap \cR^{\text{MCE},\cS}_{\cM,\pi}$
  is:
  \begin{align*}
    \fR\cap \cR^{\text{MCE},\cS}_{\cM,\pi}=
    \Bigc{x\in[-1,+1]\,|\, -1-\log\frac{\pi(a_1|s)}{1-\pi(a_1|s)}\le 
    x\le +1-\log\frac{\pi(a_1|s)}{1-\pi(a_1|s)}}.
  \end{align*}
  Clearly, this is a 1-dimensional region embedded in $\RR^2$, whose
  1-dimensional volume is:
  \begin{align*}
    \text{vol}(\fR\cap \cR^{\text{MCE},\cS}_{\cM,\pi})&=\int_{
      \max\{-1,-1-\log\frac{\pi(a_1|s)}{1-\pi(a_1|s)} \}
    }^{\min\{+1,+1-\log\frac{\pi(a_1|s)}{1-\pi(a_1|s)}\}}dx\\
    &=\min\Bigc{+1,+1-\log\frac{\pi(a_1|s)}{1-\pi(a_1|s)}}
    -\max\Bigc{-1,-1-\log\frac{\pi(a_1|s)}{1-\pi(a_1|s)}}.
  \end{align*}
  If we take policies $\pi_1$ and $\pi_2$ such that $\pi_1(a_1|s)=1/2$ and
  $\pi_2(a_1|s)=1/3$, we obtain the volumes:
  \begin{align*}
    &\text{vol}(\fR\cap \cR^{\text{MCE},\cS}_{\cM,\pi_1})=2,\\
    &\text{vol}(\fR\cap \cR^{\text{MCE},\cS}_{\cM,\pi_2})=1-(-1+\log 2)=2-\log 2,
  \end{align*}
  which are different.
  
  Regarding BIRL, the proof is analogous to that for MCE. Specifically, set
  $\beta=1$ and consider the same \MDPr in Fig. \ref{fig: counterexample MCE}.
  Thanks to Proposition \ref{prop: fs BIRL explicit}, we see that, for any
  policy $\pi$ and for any reward $r\in \cR^{\text{BIRL},\cS}_{\cM,\pi}$, there
  exists a scalar $V$ such that (take $\pi(a_2|s)\ge \pi(a_1|s)$ w.l.o.g.):
  \begin{align*}
      &r(s,a_1) = (1-\gamma)V +\log\frac{\pi(a_1|s)}{1-\pi(a_1|s)},\\
      &r(s,a_2) = (1-\gamma)V.
  \end{align*}
  By isolating $V$ in the second expression, and replacing in the first one, we
  get:
  \begin{align*}
    r(s,a_1)=r(s,a_2)+\log\frac{\pi(a_1|s)}{1-\pi(a_1|s)},
  \end{align*}
  as in MCE. Therefore, for this simple example, the integration regions
  $\fR\cap \cR^{\text{MCE},\cS}_{\cM,\pi}= \fR\cap
  \cR^{\text{BIRL},\cS}_{\cM,\pi}$ coincide. Thus, the same choices of
  $\pi_1$, $\pi_2$ as in MCE allow to prove the result.
\end{proof}

\subsection{Uniform Prior for State-only Rewards}
\label{apx: uniform prior state-only rewards}

We present here an analogous result to Proposition \ref{prop: counterexample
OPT} for state-only rewards.

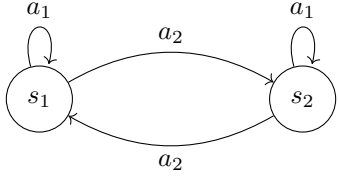
\begin{figure}[t!]
  \centering
  \begin{tikzpicture}[node distance=3.5cm]
    \node[state] at (0,0) (s1) {$s_1$};
    \node[state] at (3.5,0) (s2) {$s_2$};
    \draw (s1) edge[->, solid, loop above] node{$a_1$} (s1);
    \draw (s1) edge[->, solid, above, bend left = 30] node{$a_2$} (s2);
    \draw (s2) edge[->, solid, loop above] node{$a_1$} (s2);
    \draw (s2) edge[->, solid, below, bend left=30] node{$a_2$} (s1);
  \end{tikzpicture}
  \caption{\MDPr used in the proof of Proposition \ref{prop: counterexampleOPT
  state only}.}
  \label{fig: counterexample OPT state-only}
\end{figure}

\begin{prop}\label{prop: counterexampleOPT state only}
  Let $R_{\max}>0$ be arbitrary, and consider state-only rewards. There exists
  an \MDPr $\cM$ and two deterministic policies $\pi_1,\pi_2$ such that:
  \begin{align*}
      \text{vol}\bigr{[-1,+1]^S\cap\cR^{\text{OPT},\cS}_{\cM,\pi_1}} \neq
      \text{vol}\bigr{[-1,+1]^S\cap\cR^{\text{OPT},\cS}_{\cM,\pi_2}}.
  \end{align*}
\end{prop}
\begin{proof}
  The proof is analogous to that of Proposition \ref{prop: counterexample OPT},
  but we make use of the \MDPr in Fig. \ref{fig: counterexample OPT
  state-only}.
  Let $\pi_1$ be the policy that plays action $a_2$ in both states $s_1,s_2$.
  Note that there are four deterministic policies in this \MDPr, thus, we can
  prove the result by showing that:
  \begin{align*}
    \frac{\text{vol}\bigr{[-1,+1]^S\cap\cR^{\text{OPT},\cS}_{\cM,\pi_1}}}{\text{vol}\bigr{[-1,+1]^S}}
    \neq \frac{1}{4}.
  \end{align*}
  Since $[-1,+1]^S=[-1,+1]^{2}$, then $\text{vol}\bigr{[-1,+1]^S}=2^2=4$, thus
  the result follows by showing that:
  \begin{align*}
    \text{vol}\bigr{[-1,+1]^S\cap\cR^{\text{OPT},\cS}_{\cM,\pi_1}}\neq 1.
  \end{align*}
  Let us now compute the integration region
  $[-1,+1]^S\cap\cR^{\text{OPT},\cS}_{\cM,\pi_1}$. For any $r\in\RR^{SA}$, let
  us abbreviate $r_{i}\coloneqq r(s_i)$ (the reward is state-only). We can
  compute the value and $Q$-functions of policy $\pi_1$ under any state-only reward $r$
  as:
  \begin{align*}
    &V^{\pi_1}(s_1;p,r)=r_1+\gamma V^{\pi_1}(s_2;p,r),\\
    &V^{\pi_1}(s_2;p,r)=r_2+\gamma V^{\pi_1}(s_1;p,r),\\
    &Q^{\pi_1}(s_1,a_1;p,r)=r_1+\gamma V^{\pi_1}(s_1;p,r),\\
    &Q^{\pi_1}(s_2,a_1;p,r)=r_2+\gamma V^{\pi_1}(s_2;p,r),
  \end{align*}
  thus, using the policy improvement theorem \cite{puterman1994markov}, $r$
  belongs to $\cR^{\text{OPT},\cS}_{\cM,\pi_1}$ if and only if:
  \begin{align*}
    \begin{cases}
      &Q^{\pi_1}(s_1,a_1;p,r)\le V^{\pi_1}(s_1;p,r),\\
      &Q^{\pi_1}(s_2,a_1;p,r)\le V^{\pi_1}(s_2;p,r),
    \end{cases}
    \quad\iff\quad
    \begin{cases}
      &r_1+\gamma V^{\pi_1}(s_1;p,r)\le r_1+\gamma V^{\pi_1}(s_2;p,r),\\
      &r_2+\gamma V^{\pi_1}(s_2;p,r)\le r_2+\gamma V^{\pi_1}(s_1;p,r),
    \end{cases}
  \end{align*}
  which can be simplified to:
  \begin{align*}
    \begin{cases}
      &V^{\pi_1}(s_1;p,r)\le V^{\pi_1}(s_2;p,r),\\
      &V^{\pi_1}(s_2;p,r)\le V^{\pi_1}(s_1;p,r).
    \end{cases}\qquad\iff\qquad
    V^{\pi_1}(s_1;p,r)= V^{\pi_1}(s_2;p,r).
  \end{align*}
  Since the value functions can be computed as:
  \begin{align*}
    &V^{\pi_1}(s_1;p,r)=\frac{r_1+\gamma r_2}{1-\gamma^2},\\
    &V^{\pi_1}(s_2;p,r)=\frac{r_2+\gamma r_1}{1-\gamma^2},
  \end{align*}
  the constraint imposes that:
  \begin{align*}
    r_1+\gamma r_2=r_2+\gamma r_1\qquad\iff\qquad
    r_1=r_2.
  \end{align*}
  Therefore, the integration region:
  \begin{align*}
    [-1,+1]^S\cap\cR^{\text{OPT},\cS}_{\cM,\pi_1}=
    \Bigc{r\in [-1,+1]^S\,|\, r_1=r_2},
  \end{align*}
  has zero volume, which concludes the proof.
\end{proof}

\subsection{A Uniform Prior over Transition Models Too}
\label{apx: uniform transition models}

In this appendix we consider the setting in which the \MDPr $\cM$ is not fixed,
but its transition model $p$ is sampled from a uniform distribution over
$\Delta_{\SA}^\cS$. Interestingly, it turns out that the uniform prior over the
hypercube (see Eq. \ref{eq: uniform prior}) assigns equal weight to all the deterministic
policies on the \emph{average} of the transition models. However, for MCE and
BIRL, this does not hold.

\begin{restatable}{prop}{hypercubefineallenvsOPT}
  Let $\cM=(\cS,\cA,s_0,p,\gamma)$ be an \MDPr whose transition model $p$ is a
  random variable uniformly distributed over $\Delta_{\SA}^\cS$.
  Then, for any pair of deterministic policies $\pi_1,\pi_2$, it holds that:
  \begin{align*}
      \E_p[\text{vol}\bigr{\fR\cap\cR^{\text{OPT},\cS}_{\cM,\pi_1}}]= 
      \E_p[\text{vol}\bigr{\fR\cap\cR^{\text{OPT},\cS}_{\cM,\pi_2}}],
  \end{align*}
  where $\E_p$ denotes the expectation w.r.t. $p$.
\end{restatable}
\begin{proof}
  The result can be proved by showing that, for any possible pair of policies
  $\pi_1,\pi_2$ and any transition model $p'$, there always exists another
  transition model $p''$ for which:
  \begin{align*}
    \text{vol}\bigr{\fR\cap\cR^{\text{OPT},\cS}_{p',\pi_1}}
    =\text{vol}\bigr{\fR\cap\cR^{\text{OPT},\cS}_{p'',\pi_2}},
  \end{align*}
  and:
  \begin{align*}
    \text{vol}\bigr{\fR\cap\cR^{\text{OPT},\cS}_{p'',\pi_1}}
    =\text{vol}\bigr{\fR\cap\cR^{\text{OPT},\cS}_{p',\pi_2}}.
  \end{align*}
  where we used $p',p''$ in place of $(\cS,\cA,s_0,p',\gamma)$ and
  $(\cS,\cA,s_0,p'',\gamma)$. To this aim, fix two policies $\pi_1,\pi_2$ and a
  transition model $p'$, and construct the transition model $p''$ such that, for
  every $s,s'\in\cS$:
  \begin{align*}
    &p''(s'|s,\pi_1(s))=p'(s'|s,\pi_2(s)),\\
    &p''(s'|s,\pi_2(s))=p'(s'|s,\pi_1(s)),\\
    &p''(s'|s,a)=p'(s'|s,a)\qquad\text{if }a\neq \pi_1(s)\wedge\pi_2(s).
  \end{align*}
  Clearly, $p''$ guarantees that the role of $\pi_2$ is the same as that of
  $\pi_1$ in $p'$, and that the role of $\pi_1$ is the same as that of $\pi_2$
  in $p'$. Therefore:
  \begin{align*}
    \text{vol}\bigr{\fR\cap\cR^{\text{OPT},\cS}_{p',\pi_1}}
    =\text{vol}\bigr{\fR\cap\cR^{\text{OPT},\cS}_{p'',\pi_2}},
  \end{align*}
  and:
  \begin{align*}
    \text{vol}\bigr{\fR\cap\cR^{\text{OPT},\cS}_{p'',\pi_1}}
    =\text{vol}\bigr{\fR\cap\cR^{\text{OPT},\cS}_{p',\pi_2}}.
  \end{align*}
  This concludes the proof.
\end{proof}

\begin{restatable}{prop}{hypercubefineallenvsMCEBIRL}
  There exist a state space $\cS$, action space $\cA$, discount factor $\gamma$,
 initial state $s_0$, coefficients $\lambda,\beta\ge0$, and a pair of stochastic
 policies $\pi_1,\pi_2$ such that, if we construct $\cM=(\cS,\cA,s_0,p,\gamma)$
 as an \MDPr whose transition model $p$ is a random variable uniformly
 distributed over $\Delta_{\SA}^\cS$, then:
  \begin{align*}
      &\E_p[\text{vol}\bigr{\fR\cap\cR^{\text{MCE},\cS}_{\cM,\pi_1}}]\neq 
      \E_p[\text{vol}\bigr{\fR\cap\cR^{\text{MCE},\cS}_{\cM,\pi_2}}]\quad \wedge\\
      &\E_p[\text{vol}\bigr{\fR\cap\cR^{\text{BIRL},\cS}_{\cM,\pi_1}}]\neq 
      \E_p[\text{vol}\bigr{\fR\cap\cR^{\text{BIRL},\cS}_{\cM,\pi_2}}],
  \end{align*}
  where $\E_p$ denotes the expectation w.r.t. $p$.
\end{restatable}
\begin{proof}
  Simply, observe that, for $S=1$, the expectation over transition models
  reduces to evaluate the single transition model that brings every action back
  to the current state. For this problem instance, we have already shown that
  there exists a pair of policies $\pi_1,\pi_2$ for which the volumes are
  different (see the proof of Proposition \ref{prop: counterexample MCE}).
\end{proof}

\section{Additional Results and Proofs for Section \ref{sec: new prior}}
\label{apx: proofs our approach}

In this appendix, we provide the missing proofs for Proposition \ref{prop: rel
hypercube} (Appendix \ref{apx: rel with hypercube}) and for Propositions \ref{prop:
same vol OPT}-\ref{prop: same vol BIRL} (Appendix \ref{apx: new priors
unbiased}). In addition, we show that priors $w_\cM^m$ might be unbiased in new
environments $\cM''$ (Appendix \ref{apx: no same vol new envs}).

\subsection{Relation with the Hypercube}
\label{apx: rel with hypercube}

To prove Proposition \ref{prop: rel hypercube}, we use the following three
lemmas.

\begin{restatable}{prop}{relhypercubeOPT}\label{prop: rel hypercube OPT}
  Let $\cM$ be any \MDPr. For any $C_1^{\text{OPT}},C_2^{\text{OPT}}>0$, it holds that:
  \begin{align*}
    \sR_\cM^{\text{OPT}}\supseteq\Bigs{-(1-\gamma)\min\Bigc{\frac{C_1^{\text{OPT}}}{1+\gamma},C_2^{\text{OPT}}},
      +(1-\gamma)\min\Bigc{\frac{C_1^{\text{OPT}}}{1+\gamma},C_2^{\text{OPT}}}}^{SA},
  \end{align*}
  and that:
  \begin{align*}
    \sR_\cM^{\text{OPT}}\subseteq
    \Bigs{-\frac{1+\gamma}{1-\gamma}C_1^{\text{OPT}}-C_2^{\text{OPT}},+\frac{1+\gamma}{1-\gamma}C_1^{\text{OPT}}}^{SA}.
  \end{align*}
\end{restatable}
\begin{proof}
  %
  For any (stochastic) policy $\pi\in\Pi$, define the set:
\begin{align*}
  \mathscr{R}^{\text{OPT}}_{\cM,\pi}&\coloneqq
  \Bigc{r\in\RR^{SA}\;\Big|\;\pi\in\Pi^*(p,r)\wedge \forall s,a:\;
  |V^\pi(s;p,r)|\le C_1^{\text{OPT}} k_\pi\wedge 
  |A^\pi(s,a;p,r)|\le C_2^{\text{OPT}}
  },
\end{align*}
as the set of all the rewards that make policy $\pi$ optimal at every possible
state $s\in\cS$ and whose value and advantage functions are bounded respectively
by $C_1^{\text{OPT}} k_\pi$ and $C_2^{\text{OPT}}$.

We show that, for any \MDPr $\cM$ and policy $\pi\in\Pi$, it holds that:
\begin{align}\label{eq1}
  \mathscr{R}^{\text{OPT}}_{\cM,\pi}\supseteq
  [-(1-\gamma)\min\{C_1^{\text{OPT}}k_\pi,C_2^{\text{OPT}}\},
  +(1-\gamma)\min\{C_1^{\text{OPT}}k_\pi,C_2^{\text{OPT}}\}]^{SA}.
\end{align}
To this aim, observe that, for any $r\in[-r_{\max},+r_{\max}]^{SA}$ (with
$r_{\max}>0$), the value and advantage functions $V^\pi(s;p,r)$ and
$A^\pi(s,a;p,r)$ induced by policy $\pi$ in $\cM_r$ are bounded by
$\frac{r_{\max}}{1-\gamma}$ \cite{puterman1994markov}. Therefore, if we set
$r_{\max}= (1-\gamma)\min\{C_1^{\text{OPT}}k_\pi,C_2^{\text{OPT}}\}$, we have
that $|V^\pi(s;p,r)|\le C_1^{\text{OPT}} k_\pi$ and $|A^\pi(s,a;p,r)|\le
C_2^{\text{OPT}}$, and therefore Eq. \eqref{eq1} holds.

Now, let us show that, for any \MDPr $\cM$ and policy $\pi\in\Pi$, it holds
that:
\begin{align}\label{eq2}
  \mathscr{R}^{\text{OPT}}_{\cM,\pi}\subseteq
  [-(1+\gamma)C_1^{\text{OPT}} k_\pi-C_2^{\text{OPT}},
  +(1+\gamma)C_1^{\text{OPT}} k_\pi]^{SA}.
\end{align}
  This can be done by first observing that any $r\in
  \mathscr{R}^{\text{OPT}}_{\cM,\pi}$, since $\pi\in\Pi^*(p,r)$ by definition,
  can be rewritten through Proposition \ref{prop: fs OPT explicit} as:
  \begin{align*}
      r(s,a)=V(s)-\gamma\sum_{s'\in\cS}p(s'|s,a)V(s')-\begin{cases}
        0&\text{if }\pi(a|s)>0,\\
        A(s,a)&\text{otherwise},
      \end{cases}
  \end{align*}
  where $V\in\RR^{S}$ and $A\in(-\infty,0]^{S\times A}$ coincide with
  $V^*(\cdot;p,r)$ and $A^*(\cdot,\cdot;p,r)$ (see Proposition \ref{prop: fs OPT
  explicit}). However, by definition of $\mathscr{R}^{\text{OPT}}_{\cM,\pi}$ we
  know that $|V^\pi(s;p,r)|\le C_1^{\text{OPT}} k_\pi$ and $|A^\pi(s,a;p,r)|\le
  C_2^{\text{OPT}}$, therefore, by applying triangle's inequality, we have that:
  \begin{align*}
    \max\limits_{(s,a)\in\SA}|r(s,a)|&\le |V(s)|+\gamma\sum_{s'\in\cS}p(s'|s,a)|V(s')|+\begin{cases}
      0&\text{if }\pi(a|s)>0,\\
      |A(s,a)|&\text{otherwise},
    \end{cases}&\\
    &\le (1+\gamma)C_1^{\text{OPT}} k_\pi+C_2^{\text{OPT}}.
  \end{align*}
  Thus, Eq. \eqref{eq2} holds. 

  To conclude the proof, given the relationship between
  $\mathscr{R}^{\text{OPT}}_{\cM}$ and all the sets
  $\mathscr{R}^{\text{OPT}}_{\cM,\pi}$, it suffices to show that, for any \MDPr
  $\cM$ and policy $\pi\in\Pi$, it holds that:
  \begin{align*}
      k_\pi\in\Bigs{\frac{1}{1+\gamma},\frac{1}{1-\gamma}}.
  \end{align*}
  To do so, we write:
  \begin{align*}
      (k_\pi)^{-S}&\markref{(1)}{=}|\text{det}(W_{p,\pi})|\\
      &\markref{(2)}{=}|\text{det}(I-\gamma P^\pi)|\\
      &\markref{(3)}{=}\Biga{\prod\limits_{i=1}^S (1-\gamma\lambda_i)},
  \end{align*}
  where at (1) we use that the determinant of the inverse of a matrix coincides
  with the inverse of the determinant of the matrix
  \cite{bernstein2011matrixmathematics}, at (2) we use the definition of
  $W_{p,\pi}$, where $I\in\RR^{S\times S}$ is the identity matrix, and
  $P^\pi\in\RR^{S\times S}$ associates, to every pair of states $s,s'\in\cS$,
  the value $\sum_{a\in\cA}\pi(a|s) p(s'|s,a)$, at (3) we use the fact that the
  characteristic polynomial  satisfies $\text{det}(xI- P^\pi)=
  \prod_{i=1}^S (x-\lambda_i)$ for any $x\in\RR$, where $\lambda_i$ are the
  eigenvalues of $P^\pi$, and then we replaced $x=1/\gamma$.
  
  By noting that $P^\pi$ is a right stochastic matrix, we have that its
  eigenvalues satisfy $|\lambda_i|\le 1$ for all $i\in\dsb{S}$, and so:
  \begin{align*}
      \Bigr{\frac{1}{k_\pi}}^S\in\Bigs{(1-\gamma)^S,(1+\gamma)^S}
      \iff k_\pi\in\Bigs{\frac{1}{1+\gamma},\frac{1}{1-\gamma}}.
  \end{align*}
  This concludes the proof.
\end{proof}

\begin{restatable}{prop}{relhypercubeMCE}\label{prop: rel hypercube MCE}
    Let $\cM$ be any \MDPr. For any $C_1^{\text{MCE}},C_2^{\text{MCE}}>0$,
    it holds that:
    \begin{align*}
      \sR_\cM^{\text{MCE}}\subseteq
      [-(1+\gamma)C_1^{\text{MCE}}-C_2^{\text{MCE}},+(1+\gamma)C_1^{\text{MCE}}]^{SA}.
    \end{align*}
    Moreover, define $k_1\coloneqq (1-\gamma)C_1^{\text{MCE}}-\lambda\log A$ and
    $k_2\coloneqq (1-\gamma)C_2^{\text{MCE}}/2-\lambda/2\log A$. Then, if
    $\min\{k_1,k_2\}>0$, we also have that:
    \begin{align*}
      \sR_\cM^{\text{MCE}}\supseteq[-\min\{k_1,k_2\},+\min\{k_1,k_2\}]^{SA}.
    \end{align*}
\end{restatable}
\begin{proof}
  For the first part of the statement, for any $r\in \sR_\cM^{\text{MCE}}$ we
  can apply Proposition \ref{prop: fs MCE explicit} to rewrite it as:
  \begin{align*}
    r(s,a)&=V(s)-\gamma\sum_{s'}p(s'|s,a)V(s')+\lambda \log \pi(a|s)\\
    &=V^*_\lambda(s;p,r)-\gamma\sum_{s'}p(s'|s,a)V^*_\lambda(s';p,r)+A^*_\lambda(s,a;p,r).
  \end{align*}
  The result follows by applying triangle's inequality and noting that
  $V^*_\lambda(s;p,r)$ and $A^*_\lambda(s,a;p,r)$ are bounded since $r\in
  \sR_\cM^{\text{MCE}}$.

  For the second part of the proposition, let $r\in[-r_{\max},+r_{\max}]^{SA}$
  (for some $r_{\max}>0$). Then, we can show that, for every $s,a$:
  \begin{align*}
      &|V^*_\lambda(s;p,r)|\le \frac{r_{\max}+\lambda\log A}{1-\gamma},\\
      &|A^*_\lambda(s,a;p,r)|\le \frac{2r_{\max}+\lambda\log A}{1-\gamma}.
  \end{align*}
  To do so, for any $s\in\cS$, we can write (see also
  \cite{cao2021identifiability} and \cite{haarnoja2017rldeepenergypolicies}):
  \begin{align*}
    V^*_\lambda(s;p,r)&=\lambda \log \sum_{a\in\cA} e^{\frac{1}{\lambda}Q^*_\lambda(s,a;p,r)}\\
    &\markref{(1)}{\le} \max_{a\in\cA} Q^*_\lambda (s,a;p,r)+\lambda\log A\\
    &= \max_{a\in\cA} r(s,a)+\gamma\sum_{s'\in\cS}p(s'|s,a)V^*_\lambda(s';p,r)+\lambda\log A\\
    &\markref{(2)}{\le} r_{\max}+\gamma\max_{s'\in\cS}V^*_\lambda(s';p,r)+\lambda\log A\\
    &\markref{(3)}{\le} r_{\max}+\lambda\log A+\gamma(\popblue{r_{\max}+\lambda\log A+\gamma\max_{s'\in\cS}V^*_\lambda(s';p,r)})\\
    &\markref{(4)}{\le}\frac{r_{\max}+\lambda\log A}{1-\gamma},
  \end{align*}
  where at (1) we used the log-sum-exp inequality: $\log\sum_{i\in\dsb{n}}
  e^{x_i}\le \max_i x_i + \log n$, at (2) we use that $r$ is in the hypercube by
  hypothesis, at (3) we unfold the recursion, obtaining the expression in (4)
  since it is a geometric series.

  Regarding the advantage function, we begin by lower bounding the soft
  $Q^*_\lambda$:
  \begin{align*}
    Q^*_\lambda(s,a;p,r)&=r(s,a)+\gamma\sum_{s'\in\cS}p(s'|s,a)V^*_\lambda(s';p,r)\\
    &\markref{(5)}{\ge}-r_{\max}+\gamma\min_{s'\in\cS}V^*_\lambda(s';p,r)\\
    &=-r_{\max}+\gamma\min_{s'\in\cS}\lambda \log \sum_{a'\in\cA} e^{\frac{1}{\lambda}Q^*_\lambda(s',a';p,r)}\\
    &\markref{(6)}{\ge}-r_{\max}+\gamma\min_{s'\in\cS}\max_{a'\in\cA}Q^*_\lambda(s',a';p,r)\\
    &\ge-r_{\max}+\gamma(\popblue{-r_{\max}+\gamma\min_{s'\in\cS}\max_{a'\in\cA}Q^*_\lambda(s',a';p,r)})\\
    &\markref{(7)}{=}-\frac{r_{\max}}{1-\gamma},
  \end{align*}
  where at (5) we use that $r$ is in the hypercube and bound the expectation
  with the minimum, at (6) we used the log-sum-exp inequality: $\log\sum_i
  e^{x_i}\ge \max_i x_i$, and then we have unfolded the recursion, obtaining
  the expression in (7).

  Then, we can bound the soft advantage combining these bounds:
  \begin{align*}
    A^*_\lambda(s,a;p,r)&\coloneqq Q^*_\lambda(s,a;p,r)-V^*_\lambda(s;p,r)\\
    &\ge -\frac{r_{\max}}{1-\gamma} - \frac{r_{\max}+\lambda\log A}{1-\gamma}\\
    &=- \frac{2r_{\max}+\lambda\log A}{1-\gamma}.
  \end{align*}
  Therefore, a sufficient condition for the reward $r$ to belong to set
  $\sR_\cM^{\text{MCE}}$ is that:
  \begin{align*}
    \begin{cases}
      \frac{r_{\max}+\lambda\log A}{1-\gamma}\le C_1^{\text{MCE}},\\
      \frac{2r_{\max}+\lambda\log A}{1-\gamma}\le C_2^{\text{MCE}},
    \end{cases}\iff 
    \begin{cases}
      r_{\max}\le (1-\gamma)C_1^{\text{MCE}}-\lambda\log A,\\
      r_{\max}\le (1-\gamma)C_2^{\text{MCE}}/2-\lambda/2\log A.
    \end{cases}
  \end{align*}
  By taking the minimum and imposing that $r_{\max}\ge 0$, we get the result.
\end{proof}

\begin{restatable}{prop}{relhypercubeBIRL}\label{prop: rel hypercube BIRL}
  Let $\cM$ be any \MDPr. For any $C_1^{\text{BIRL}},C_2^{\text{BIRL}}>0$,
  it holds that:
  \begin{align*}
    \sR_\cM^{\text{BIRL}}\subseteq
    [-(1+\gamma)C_1^{\text{BIRL}}-C_2^{\text{BIRL}},+(1+\gamma)C_1^{\text{BIRL}}]^{SA},
  \end{align*}
  and also that:
  \begin{align*}
    \sR_\cM^{\text{BIRL}}\supseteq[-(1-\gamma)\min\{C_1^{\text{BIRL}},C_2^{\text{BIRL}}\},
    +(1-\gamma)\min\{C_1^{\text{BIRL}},C_2^{\text{BIRL}}\}]^{SA}.
  \end{align*}
\end{restatable}
\begin{proof}
  For the first part, observe that any $r\in\sR_\cM^{\text{BIRL}}$ can be
  rewritten using Proposition \ref{prop: fs BIRL explicit} as:
  \begin{align*}
    r(s,a)&=V(s)-\gamma\sum_{s'}p(s'|s,a)V(s')+\beta 
    \log \frac{\pi(a|s)}{\max_{a'\in\cA}\pi(a'|s)}\\
    &=V^*(s;p,r)-\gamma\sum_{s'}p(s'|s,a)V^*(s';p,r)+A^*(s,a;p,r).
  \end{align*}
  Then, the result follows by applying triangle's inequality and the bounds on
  $V^*(s;p,r)$ and $A^*(s,a;p,r)$ in the definition of $\sR_\cM^{\text{BIRL}}$.

  For the second part, observe that, as in the proof of Proposition \ref{prop:
  rel hypercube OPT}, if $r\in[-r_{\max},+r_{\max}]$ (with $r_{\max}>0$), then
  its optimal value and advantage functions are bounded by
  $\frac{r_{\max}}{1-\gamma}$. By asking that this quantity is smaller than both
  $C_1^{\text{BIRL}}$ and $C_2^{\text{BIRL}}$, we get the result.
\end{proof}

We are now ready to prove Proposition \ref{prop: rel hypercube}.

\relhypercube*
\begin{proof}
  Apply Proposition \ref{prop: rel hypercube OPT},\ref{prop: rel hypercube MCE}
  and \ref{prop: rel hypercube BIRL} by choosing the values that guarantee that
  the sets in the lower inclusions contain the unit hypercube $\fR$.
\end{proof}

\subsection{The New Priors are Unbiased}
\label{apx: new priors unbiased}

Before proving the results, we need the following proposition.

\begin{prop}\label{prop: rel T W}
  Let $\cM=\tuple{\cS,\cA,s_0,p,\gamma}$ be an \MDPr and let $\pi$ be any
  deterministic policy. Then, it holds that:
  \begin{align*}
    |\text{det}(T_{p,\pi})|=|\text{det}(W_{p,\pi})|.
  \end{align*}
\end{prop}
\begin{proof} 
  Assume that $\pi$ prescribes action $a_1$ in every possible state. Then,
  $T_{p,\pi}$ can be written as:
  \begin{align*}
    \setcounter{MaxMatrixCols}{14}
    T_{p,\pi} = \scalebox{0.7}{$  \displaystyle\begin{bmatrix}
      1-\gamma p(s_1|s_1,a_1) &
      -\gamma p(s_2|s_1,a_1) & \dotsc &
      -\gamma p(s_S|s_1,a_1)
      & 0 & \dotsc & 0 & 0 & \dotsc & 0 & \dotsc &
      0 & \dotsc & 0\\
      1-\gamma p(s_1|s_1,a_2) &
      -\gamma p(s_2|s_1,a_2) & \dotsc &
      -\gamma p(s_S|s_1,a_2)& 1 & \dotsc & 0 & 0 & \dotsc & 0 & \dotsc &
      0 & \dotsc & 0\\
      \dotsc & \dotsc & \dotsc & \dotsc& \dotsc & \dotsc & \dotsc & \dotsc & \dotsc & \dotsc & \dotsc &
      \dotsc & \dotsc & \dotsc\\
      1-\gamma p(s_1|s_1,a_A) &
      -\gamma p(s_2|s_1,a_A) & \dotsc &
      -\gamma p(s_S|s_1,a_A)& 0 & \dotsc & 1 & 0 & \dotsc & 0 & \dotsc &
      0 & \dotsc & 0\\
      -\gamma p(s_1|s_2,a_1) &
      1-\gamma p(s_2|s_2,a_1) & \dotsc &
      -\gamma p(s_S|s_2,a_1)& 0 & \dotsc & 0 & 0 & \dotsc & 0 & \dotsc &
      0 & \dotsc & 0\\
      -\gamma p(s_1|s_2,a_2) &
      1-\gamma p(s_2|s_2,a_2) & \dotsc &
      -\gamma p(s_S|s_2,a_2)& 0 & \dotsc & 0 & 1 & \dotsc & 0 & \dotsc &
      0 & \dotsc & 0\\
      \dotsc & \dotsc & \dotsc & \dotsc& \dotsc & \dotsc & \dotsc & \dotsc & \dotsc & \dotsc & \dotsc &
      \dotsc & \dotsc & \dotsc\\
      -\gamma p(s_1|s_2,a_A) &
      1-\gamma p(s_2|s_2,a_A) & \dotsc &
      -\gamma p(s_S|s_2,a_A)& 0 & \dotsc & 0 & 0 & \dotsc & 1 & \dotsc &
      0 & \dotsc & 0\\
      -\gamma p(s_1|s_S,a_1) &
      -\gamma p(s_2|s_S,a_1) & \dotsc &
      1-\gamma p(s_S|s_S,a_1)& 0 & \dotsc & 0 & 0 & \dotsc & 0 & \dotsc &
      0 & \dotsc & 0\\
      -\gamma p(s_1|s_S,a_2) &
      -\gamma p(s_2|s_S,a_2) & \dotsc &
      1-\gamma p(s_S|s_S,a_2)& 0 & \dotsc & 0 & 0 & \dotsc & 0 & \dotsc &
      1 & \dotsc & 0\\
      \dotsc & \dotsc & \dotsc & \dotsc& \dotsc & \dotsc & \dotsc & \dotsc & \dotsc & \dotsc & \dotsc &
      \dotsc & \dotsc & \dotsc\\
      -\gamma p(s_1|s_S,a_A) &
      -\gamma p(s_2|s_S,a_A) & \dotsc &
      1-\gamma p(s_S|s_S,a_A)& 0 & \dotsc & 0 & 0 & \dotsc & 0 & \dotsc &
      0 & \dotsc & 1\\
      \end{bmatrix}$},
  \end{align*}
  where on the right part of the matrix we do not have any 1 in all and only the
  rows corresponding to the action $a_1$, played by $\pi$.
  To compute $\text{det}(T_{p,\pi})$, let us apply the Laplace's formula (see
  Theorem 1.5 of \cite{hefferon2009linearalgebra}) by iteratively expanding by
  cofactors from the last column $SA$ up to column $S+1$ included. Clearly,
  since these columns are all zeros with exactly one 1, then after these
  iterations we obtain that, for some $x\in\{0,1\}$:
  \begin{align*}
    \text{det}(T_{p,\pi})=(-1)^x\text{det}\begin{bmatrix}
      1-\gamma p(s_1|s_1,a_1) &
      -\gamma p(s_2|s_1,a_1) & \dotsc &
      -\gamma p(s_S|s_1,a_1)\\
      -\gamma p(s_1|s_2,a_1) &
      1-\gamma p(s_2|s_2,a_1) & \dotsc &
      -\gamma p(s_S|s_2,a_1)\\
      \dotsc & \dotsc & \dotsc & \dotsc\\
      -\gamma p(s_1|s_S,a_1) &
      -\gamma p(s_2|s_S,a_1) & \dotsc &
      1-\gamma p(s_S|s_S,a_1)\\
      \end{bmatrix},
  \end{align*}
  since all the rows corresponding to the actions non prescribed by $\pi$ are
  neglected. However, this matrix is exactly $W_{p,\pi}$, therefore, the result
  follows when $\pi(s)=a_1$ in every state.
  
  Generalizing to arbitrary deterministic policies $\pi$ is immediate once
  realizing that $W_{p,\pi}$ contains all and only the rows that, in
  $T_{p,\pi}$, are followed by only zeros.
\end{proof}

\allsamevolOPT*
\begin{proof}
For any deterministic policy $\pi$ in any \MDPr
$\cM$, we can write:
\begin{align*}
  \text{vol}(\sR^{\text{OPT}}_\cM\cap
  \cR^{\text{OPT},\cS}_{\cM,\pi})
  &\markref{(1)}{=}
  \text{vol}(\sR^{\text{OPT}}_\cM\cap
  \cR^{\text{OPT},\cS}_{\cM,\pi}\cap\popblue{\sR^{1}_p})
  +\text{vol}(\sR^{\text{OPT}}_\cM\cap
  \cR^{\text{OPT},\cS}_{\cM,\pi}\cap\popblue{\sR^{>1}_p})\\
  &\markref{(2)}{=}
  \text{vol}(\sR^{\text{OPT}}_\cM\cap
  \cR^{\text{OPT},\cS}_{\cM,\pi}\cap\sR^{1}_p)\\
  &\markref{(3)}{=}
  |\text{det}(T_{p,\pi})|\text{vol}(\cV_{\pi})\\
  &\markref{(4)}{=}
  |\popblue{\text{det}(W_{p,\pi})}|\text{vol}(\cV_{\pi})\\
  &\markref{(5)}{=}
  \popblue{k_\pi^{-S}}\text{vol}(\cV_{\pi}),
\end{align*}
where at (1) we use the additivity of the integral (see Eq. \eqref{eq: def R
greater 1} for a definition of $\sR^{1}_p$ and $\sR^{>1}_p$), at (2) we realize
that the set $\sR^{\text{OPT}}_\cM\cap
\cR^{\text{OPT},\cS}_{\cM,\pi}\cap\sR^{>1}_p$, containing all the rewards that
induce at least two optimal policies from every state in $\cM$ and with bounded
value and advantage functions, has zero $SA$-dimensional volume since the
optimality of another policy $\pi'$ beyond $\pi$ represents an additional
constraint (the considered volume is $SA$-dimensional because of Proposition
\ref{prop: fs OPT explicit}, where it is shown that the OPT feasible set has
$SA$ degrees of freedom). At (3) we use the change of variables formula in
Proposition 6.1.3 of \cite{cohn2013measure}, after having observed that
matrix\footnote{The definition of $T_{p,\pi}$ is provided in Eq. \eqref{eq: def
T p pi}.} $T_{p,\pi}$ (along with Proposition \ref{prop: fs OPT explicit})
permits to reparametrize set $\sR^{\text{OPT}}_\cM\cap
\cR^{\text{OPT},\cS}_{\cM,\pi}\cap\sR^{1}_p$ using set $\cV_\pi$ as:
  \begin{align*}
    \sR^{\text{OPT}}_\cM\cap
    \cR^{\text{OPT},\cS}_{\cM,\pi}\cap\sR^{1}_p=\Bigc{r\in\RR^{SA}\,|\,
    \exists (V,A)\in\cV_\pi:\,
    r=T_{p,\pi}(V,A)},
  \end{align*}
  where 
  \begin{align*}
    \cV_\pi\coloneqq\Bigc{V\in[-C_1^{\text{OPT}}k_\pi,+C_1^{\text{OPT}}k_\pi]^S,
    A\in [-C_2^{\text{OPT}},0)^{S\times(A-1)}}.
  \end{align*}
  At (4) we use Proposition \ref{prop: rel T W} and at (5) we use the definition
  of $k_\pi$.

  We can compute the volume of $\cV_{\pi}$ as:
    \begin{align*}
      \text{vol}(\cV_{\pi})&\coloneqq\int_{\cV^{\pi}}dVdA\\
      &=\int_{[-C_1^{\text{OPT}}k_\pi,+C_1^{\text{OPT}}k_\pi]^S}dV \cdot
      \int_{[-C_2^{\text{OPT}},0)^{S\times(A-1)}}dA\\
      &=\prod\limits_{s\in\cS}\int_{[-C_1^{\text{OPT}}k_\pi,+C_1^{\text{OPT}}k_\pi]}dV(s)
      \cdot \prod\limits_{s\in\cS,a\in\cA\setminus\{\pi(s)\}}
      \int_{[-C_2^{\text{OPT}},0)^{S\times(A-1)}}dA(s,a)\\
      &=2^Sk_\pi^S(C_1^{\text{OPT}})^S (C_2^{\text{OPT}})^{S(A-1)}.
    \end{align*}
    Thus, replacing into the previous expression, we obtain:
    \begin{align*}
      \text{vol}(\sR^{\text{OPT}}_\cM\cap
      \cR^{\text{OPT},\cS}_{\cM,\pi})&=2^S(C_1^{\text{OPT}})^S (C_2^{\text{OPT}})^{S(A-1)}.
    \end{align*}
    Since this quantity is constant for all the deterministic policies $\pi$,
    this concludes the proof.
  \end{proof}

\allsamevolMCE*
\begin{proof}
  For any \MDPr $\cM$, observe that, for any stochastic policy $\pi$, if it does
  not hold that:
  \begin{align*}
  \min_{(s,a)\in\SA}A^\pi_\lambda(s,a;p,r)\ge -C_2^{\text{MCE}},
  \end{align*}
  then the intersection $\sR^{\text{MCE}}_\cM\cap
  \cR^{\text{MCE},\cS}_{\cM,\pi}$ is empty.
  Since, by hypothesis, $\min_{(s,a)\in\SA}\pi(a|s)\ge\pi_{\min}$, and since we
  know that \cite{cao2021identifiability,haarnoja2017rldeepenergypolicies}:
  \begin{align*}
    A^*_\lambda(s,a;p,r)=\lambda\log\pi(a|s),
  \end{align*}
  then the above condition can be imposed to hold as long as:
  \begin{align*}
    C_2^{\text{MCE}} \ge \lambda \log\frac{1}{\pi_{\min}}.
  \end{align*}  
  For this reason, let $\pi$ be a
  stochastic policy satisfying that
  $C_2^{\text{MCE}}\ge\lambda\log\frac{1}{\pi_{\min}}$. Define
  $\eta(s,a)=\lambda\log\pi(a|s)$ for all $s,a$ for conveniency.
  We can write:
  \begin{align*}
    \text{vol}(\sR^{\text{MCE}}_\cM\cap
    \cR^{\text{MCE},\cS}_{\cM,\pi})
    &\markref{(1)}{=}
    \sqrt{|\text{det}(J_{p,\eta}^\intercal J_{p,\eta})|}\text{vol}_{S}(\cV_{\pi}')\\
    &\markref{(2)}{=}\sqrt{|\text{det}(J_{p,\eta}^\intercal J_{p,\eta})|}
    \popblue{2^S (C_1^{\text{MCE}})^S},
  \end{align*}
  where at (1) we use the change of variables formula in
  \cite{boothby1986manifolds,stoll2007integration}, after having observed that
  the \emph{rectangular}\footnote{The set $\sR^{\text{MCE}}_\cM\cap
  \cR^{\text{MCE},\cS}_{\cM,\pi}$ is $S$-dimensional but is embedded into
  $\RR^{SA}$, so we cannot apply Proposition 6.1.3 of \cite{cohn2013measure}.}
  matrix $U_{p,\eta}$, defined in Eq. \eqref{eq: def U eta}, along with
  Proposition \ref{prop: fs MCE explicit}, permits to reparametrize the set
  $\sR^{\text{MCE}}_\cM\cap \cR^{\text{MCE},\cS}_{\cM,\pi}$ using set
  $\cV_{\pi}'$ as:
    \begin{align*}
      \sR^{\text{MCE}}_\cM\cap
    \cR^{\text{MCE},\cS}_{\cM,\pi}=\Bigc{r\in\RR^{SA}\,|\,\exists V\in\cV_\pi':\,
    r=U_{p,\eta}(V)},
    \end{align*}
    where 
    \begin{align*}
      \cV_\pi'\coloneqq[-C_1^{\text{MCE}},+C_1^{\text{MCE}}]^S.
    \end{align*}
    Specifically, we use symbol $J_{p,\eta}(V)\in\RR^{SA\times S}$ to denote the
    Jacobian of the operator $U_{p,\eta}$ evaluated in any $V\in\RR^S$; however,
    since $J_{p,\eta}(V)$ is a constant function of $V$ (as shown below), then
    we use $J_{p,\eta}$. At (2) we compute the volume of set $\cV_\pi'$, which
    is an hypercube.

    Observe that $2^S (C_1^{\text{MCE}})^S$ does not depend on the policy $\pi$.
    The proof is concluded after having observed that the Jacobian $J_{p,\eta}$
    is a constant function. Indeed, for any $V\in\RR^S$, we have:
    \begin{align*}
      J_{p,\eta}(V)&\coloneqq\begin{bmatrix}
        \frac{\partial [U_{p,\eta}(V)]_{s_1,a_1}}{\partial V(s_1)} &
        \frac{\partial [U_{p,\eta}(V)]_{s_1,a_1}}{\partial V(s_2)} & \dotsc &
        \frac{\partial [U_{p,\eta}(V)]_{s_1,a_1}}{\partial V(s_S)}\\
        \frac{\partial [U_{p,\eta}(V)]_{s_1,a_2}}{\partial V(s_1)} &
        \frac{\partial [U_{p,\eta}(V)]_{s_1,a_2}}{\partial V(s_2)} & \dotsc &
        \frac{\partial [U_{p,\eta}(V)]_{s_1,a_2}}{\partial V(s_S)}\\
        \dotsc & \dotsc & \dotsc & \dotsc\\
        \frac{\partial [U_{p,\eta}(V)]_{s_S,a_A}}{\partial V(s_1)} &
        \frac{\partial [U_{p,\eta}(V)]_{s_S,a_A}}{\partial V(s_2)} & \dotsc &
        \frac{\partial [U_{p,\eta}(V)]_{s_S,a_A}}{\partial V(s_S)}\\
        \end{bmatrix}\\
        &=
        \begin{bmatrix}
          1-\gamma p(s_1|s_1,a_1) &
          -\gamma p(s_2|s_1,a_1) & \dotsc &
          -\gamma p(s_S|s_1,a_1)\\
          1-\gamma p(s_1|s_1,a_2) &
          -\gamma p(s_2|s_1,a_2) & \dotsc &
          -\gamma p(s_S|s_1,a_2)\\
          \dotsc & \dotsc & \dotsc & \dotsc\\
          -\gamma p(s_1|s_S,a_A) &
          -\gamma p(s_2|s_S,a_A) & \dotsc &
          1-\gamma p(s_S|s_S,a_A)\\
          \end{bmatrix}.
    \end{align*}
\end{proof}

\allsamevolBIRL*
\begin{proof}
  For any \MDPr $\cM$, observe that, for any stochastic policy $\pi$, if it does
  not hold that:
  \begin{align*}
  \min_{(s,a)\in\SA}A^\pi(s,a;p,r)\ge -C_2^{\text{BIRL}},
  \end{align*}
  then the intersection $\sR^{\text{BIRL}}_\cM\cap
  \cR^{\text{BIRL},\cS}_{\cM,\pi}$ is empty.
  Since, by hypothesis, $\min_{(s,a)\in\SA}\pi(a|s)\ge\pi_{\min}$, and since we
  know that (see \cite{skalse2023invariancepolicyoptimisationpartial} and
  Proposition \ref{prop: fs BIRL explicit}):
  \begin{align*}
    A^*(s,a;p,r)=\beta\log\frac{\pi(a|s)}{\max_{a'\in\cA}\pi(a'|s)},
  \end{align*}
  then the above condition can be imposed to hold as long as:
  \begin{align*}
    C_2^{\text{BIRL}} \ge \beta \log\frac{1}{A\pi_{\min}},
  \end{align*}
  where we also used that $\log\max_{a}\pi(a|s)\ge\log\frac{1}{A}$.
Having made this condition to hold, the proof is then analogous to that of
    Proposition \ref{prop: same vol MCE}, with the only difference that we
    define $\eta(s,a)=\beta\log(\pi(a|s)/\max_{a'\in\cA}\pi(a'|s))$ at all
    $s,a$.
\end{proof}

\subsection{The New Priors can be Biased in new Environments}
\label{apx: no same vol new envs}

We now show that the new priors $w_\cM^m$ can be biased in environments
different from $\cM$.

\begin{restatable}{prop}{nosamevolnewenvOPT}\label{prop: no same vol new env}
  There exists a pair of \MDPr $\cM=\tuple{\cS,\cA,s_0,p,\gamma}$ and
  $\cM'=\tuple{\cS,\cA,s_0,p,\gamma'}$ with two states $s_1,s_2$ and two actions
  $a_1,a_2$ in $s_1$ and one action $a_1$ in $s_2$, and a pair of deterministic
  policies $\pi_1,\pi_2$ such that, if we set $C_1^{\text{OPT}}=1$ and if we let
  $\gamma\to 1$, then for any choice of $C_2^{\text{OPT}}>0$, we have:
  \begin{align*}
    \text{vol}(\sR^{\text{OPT}}_\cM\cap \cR^{\text{OPT},\cS}_{\cM',\pi_1})\neq
    \text{vol}(\sR^{\text{OPT}}_\cM\cap \cR^{\text{OPT},\cS}_{\cM',\pi_2}).
  \end{align*}
\end{restatable}
\begin{proof}
  
  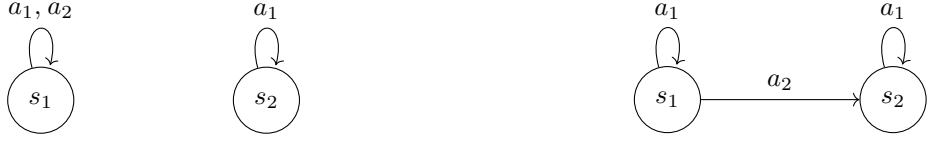
\begin{figure}[t!]
    \centering
    \begin{minipage}[t!]{0.4\textwidth}
      \centering
    \begin{tikzpicture}[node distance=3.5cm]
      \node[state] at (0,0) (s1) {$s_1$};
      \node[state] at (3,0) (s2) {$s_2$};
      \draw (s1) edge[->, solid, loop above] node{$a_1,a_2$} (s1);
      \draw (s2) edge[->, solid, loop above] node{$a_1$} (s2);
    \end{tikzpicture}
  \end{minipage}
  \hfill
  \begin{minipage}[t!]{0.4\textwidth}
    \centering
    \begin{tikzpicture}[node distance=3.5cm]
      \node[state] at (0,0) (s1) {$s_1$};
      \node[state] at (3,0) (s2) {$s_2$};
      \draw (s1) edge[->, solid, loop above] node{$a_1$} (s1);
      \draw (s1) edge[->, solid, above] node{$a_2$} (s2);
      \draw (s2) edge[->, solid, loop above] node{$a_1$} (s2);
    \end{tikzpicture}
  \end{minipage}
    \caption{\MDPr considered in the proof of Proposition \ref{prop: no same vol
    new env}. In the left there is $\cM$, while on the right there is $\cM'$.}
    \label{fig: per proof no same vol new env}
  \end{figure}

  For the proof, we consider the $\cM,\cM'$ in Fig. \ref{fig: per proof no
  same vol new env}, and we define $x\coloneqq r(s_1,a_1), y\coloneqq
  r(s_1,a_2), z\coloneqq r(s_1,a_1)$ for simplicity. Let $\pi_1,\pi_2$ be the
  deterministic policies that play actions $a_1,a_2$ in state $s_1$.
  Now, let us make explicit the sets $\sR^{\text{OPT}}_\cM$,
  $\cR^{\text{OPT},\cS}_{\cM',\pi_1}$ and $\cR^{\text{OPT},\cS}_{\cM',\pi_2}$.
  It is immediate to observe that, in $\cM'$, for $\gamma\to 1$, policy $\pi_1$
  is optimal if $x\ge z$, and viceversa $\pi_2$ is optimal if $x\le z$:
  \begin{align*}
    &\cR^{\text{OPT},\cS}_{\cM',\pi_1}=\Bigc{x,y,z\in\RR\,|\,x\ge z},\\
    &\cR^{\text{OPT},\cS}_{\cM',\pi_2}=\Bigc{x,y,z\in\RR\,|\,x\le z}.
  \end{align*}
  Regarding $\sR^{\text{OPT}}_\cM$, we need to compute $k_\pi$ for any possible
  (stochastic) policy $\pi\in\Pi$. We have:
  \begin{align*}
    (k_{\pi})^{-S}= |\text{det}(W_{p,\pi})|= |\text{det}\begin{bmatrix}
      1-\gamma & 0 \\
      0 & 1-\gamma\\
      \end{bmatrix}|=(1-\gamma)^2.
  \end{align*}
  Since we have two states, $S=2$ and so:
  \begin{align*}
    k_{\pi}=\frac{1}{1-\gamma},
  \end{align*}
  for every possible policy $\pi$.\footnote{To be precise, we have never shown
  in the paper that this choice of set gives the same volume
  $\text{vol}(\sR^{\text{OPT}}_\cM\cap \cR^{\text{OPT},\cS}_{\cM,\pi_1})\neq
  \text{vol}(\sR^{\text{OPT}}_\cM\cap \cR^{\text{OPT},\cS}_{\cM,\pi_2})$ in any
  $\cM$ for any possible pair of policies $\pi_1,\pi_2$ even though some states
  might have a number of actions different than that of other states.
  Nevertheless, this can be proved rather easily analogously to Proposition
  \ref{prop: same vol OPT}. Of course, an analogous of Proposition \ref{prop: fs
  OPT explicit} will be needed.}
  To represent set $\sR^{\text{OPT}}_\cM$ explicitly as a function of $x,y,z$,
  observe that it can be rewritten as:
  \begin{align*}
    \sR^{\text{OPT}}_\cM = (\sR^{\text{OPT}}_\cM\cap \cR^{\text{OPT},\cS}_{\cM,\pi_1}\cap\sR^1_p)
    \cup (\sR^{\text{OPT}}_\cM\cap \cR^{\text{OPT},\cS}_{\cM,\pi_2}\cap\sR^1_p)
    \cup (\sR^{\text{OPT}}_\cM\cap \sR^{>1}_p),
  \end{align*}
  where set $\sR^{\text{OPT}}_\cM\cap \sR^{>1}_p$ has zero $3$-dimensional
  volume and so is uninteresting. By using Proposition \ref{prop: fs OPT
  explicit} and also that $(1-\gamma)k_\pi=1$ for any $\pi$ and $\gamma$, it can
  be shown that:
  \begin{align*}
    &\sR^{\text{OPT}}_\cM\cap \cR^{\text{OPT},\cS}_{\cM,\pi_1}\cap\sR^1_p
    =\Bigc{x,y,z\in\RR\,|\,x,z\in[-C_1^{\text{OPT}},+C_1^{\text{OPT}}]
    \wedge y\in[x-C_2^{\text{OPT}},x]},\\
    &\sR^{\text{OPT}}_\cM\cap \cR^{\text{OPT},\cS}_{\cM,\pi_2}\cap\sR^1_p
    =\Bigc{x,y,z\in\RR\,|\,y,z\in[-C_1^{\text{OPT}},+C_1^{\text{OPT}}]
    \wedge x\in[y-C_2^{\text{OPT}},y]}.
  \end{align*}

  Since $\sR^{\text{OPT}}_\cM\cap \sR^{>1}_p$ has zero $3$-dimensional
  volume, we can write:
  \begin{align*}
    \text{vol}(\sR^{\text{OPT}}_\cM\cap \cR^{\text{OPT},\cS}_{\cM',\pi_1})&=
    \text{vol}(\sR^{\text{OPT}}_\cM\cap \cR^{\text{OPT},\cS}_{\cM,\pi_1}\cap
    \sR^1_p\cap \cR^{\text{OPT},\cS}_{\cM',\pi_1})\\
    &\qquad\qquad+
    \text{vol}(\sR^{\text{OPT}}_\cM\cap \cR^{\text{OPT},\cS}_{\cM,\pi_2}\cap
    \sR^1_p\cap \cR^{\text{OPT},\cS}_{\cM',\pi_1})\\
    &=\text{vol}(\sR^{\text{OPT}}_\cM\cap \cR^{\text{OPT},\cS}_{\cM,\pi_1}\cap
    \sR^1_p\cap \popblue{\cR^{\text{OPT},\cS}_{\cM',\pi_2}})\\
    &\qquad\qquad+
    \text{vol}(\sR^{\text{OPT}}_\cM\cap \cR^{\text{OPT},\cS}_{\cM,\pi_2}\cap
    \sR^1_p\cap \cR^{\text{OPT},\cS}_{\cM',\pi_1}),
  \end{align*}
  where in the second passage we have observed that $\sR^{\text{OPT}}_\cM\cap
    \cR^{\text{OPT},\cS}_{\cM,\pi_1}\cap \sR^1_p$ allows $x,z$ to vary freely in
    the same interval, therefore, using the formulas computed earlier for
    $\cR^{\text{OPT},\cS}_{\cM',\pi_1}$ and $\cR^{\text{OPT},\cS}_{\cM',\pi_2}$,
    we realize that by symmetry $\text{vol}(\sR^{\text{OPT}}_\cM\cap \cR^{\text{OPT},\cS}_{\cM,\pi_1}\cap
    \sR^1_p\cap \cR^{\text{OPT},\cS}_{\cM',\pi_1}) = 
    \text{vol}(\sR^{\text{OPT}}_\cM\cap \cR^{\text{OPT},\cS}_{\cM,\pi_1}\cap
    \sR^1_p\cap \cR^{\text{OPT},\cS}_{\cM',\pi_2})$.

    For this reason, the claim in the proposition can be proved by showing that:
    \begin{align*}
      \text{vol}(\sR^{\text{OPT}}_\cM\cap \cR^{\text{OPT},\cS}_{\cM,\pi_2}\cap
    \sR^1_p\cap \cR^{\text{OPT},\cS}_{\cM',\pi_1})\neq 
    \text{vol}(\sR^{\text{OPT}}_\cM\cap \cR^{\text{OPT},\cS}_{\cM,\pi_2}\cap
    \sR^1_p\cap \cR^{\text{OPT},\cS}_{\cM',\pi_2}).
    \end{align*}
 To make things even easier, we will prove it by first computing
    $\text{vol}(\sR^{\text{OPT}}_\cM\cap \cR^{\text{OPT},\cS}_{\cM,\pi_2}\cap
    \sR^1_p)$, then computing $\text{vol}(\sR^{\text{OPT}}_\cM\cap
    \cR^{\text{OPT},\cS}_{\cM,\pi_2}\cap \sR^1_p\cap
    \cR^{\text{OPT},\cS}_{\cM',\pi_1})$, and finally showing that their ratio is
    different from $1/2$:
    \begin{align*}
      \frac{\text{vol}(\sR^{\text{OPT}}_\cM\cap
      \cR^{\text{OPT},\cS}_{\cM,\pi_2}\cap \sR^1_p\cap
      \cR^{\text{OPT},\cS}_{\cM',\pi_1})}{\text{vol}(\sR^{\text{OPT}}_\cM
      \cap \cR^{\text{OPT},\cS}_{\cM,\pi_2}\cap
      \sR^1_p)}\neq \frac{1}{2}.
    \end{align*}

  We begin with the volume of $\sR^{\text{OPT}}_\cM \cap
  \cR^{\text{OPT},\cS}_{\cM,\pi_2}\cap \sR^1_p$. We can write:
  \begin{align*}
    \text{vol}(\cR_{p,\pi_2})&=\int\limits_{-C_1^{\text{OPT}}}^{+C_1^{\text{OPT}}}
    \int\limits_{-C_1^{\text{OPT}}}^{+C_1^{\text{OPT}}}\int\limits_{y-C_2^{\text{OPT}}}^{y}dxdydz\\
    &=\int\limits_{-C_1^{\text{OPT}}}^{+C_1^{\text{OPT}}}
    \int\limits_{-C_1^{\text{OPT}}}^{+C_1^{\text{OPT}}}C_2^{\text{OPT}} dydz\\
    &=4C_1^{\text{OPT}}C_2^{\text{OPT}}=4C_2^{\text{OPT}},
  \end{align*}
  by recalling that $C_1^{\text{OPT}}=1$ by hypothesis.

  Now, concerning the volume of $\sR^{\text{OPT}}_\cM\cap
  \cR^{\text{OPT},\cS}_{\cM,\pi_2}\cap \sR^1_p\cap
  \cR^{\text{OPT},\cS}_{\cM',\pi_1}$, note that we can write this region more
  explicitly (modulo some zero-volume sets) as:
  \begin{align*}
    \sR^{\text{OPT}}_\cM\cap
      \cR^{\text{OPT},\cS}_{\cM,\pi_2}\cap \sR^1_p\cap
      \cR^{\text{OPT},\cS}_{\cM',\pi_1}=\Big\{x,y,z\in\RR\,|\,
    -1\le &y\le +1\;\wedge\\
    -1\le &z\le +1\;\wedge\\
    y-C_2^{\text{OPT}}\le &x\le y\;\wedge\\
    &x\ge z\Big\}.
  \end{align*}
  We partition it into 4 regions $\cR\coloneqq \sR^{\text{OPT}}_\cM\cap
      \cR^{\text{OPT},\cS}_{\cM,\pi_2}\cap \sR^1_p\cap
      \cR^{\text{OPT},\cS}_{\cM',\pi_1}=I_1\cup I_2\cup I_3\cup I_4$, so that
      the resulting volume will be the sum of the volumes of these regions:
  \begin{align*}
    &I_1\coloneqq \cR\cap\{z\ge y-C_2^{\text{OPT}}\;\wedge\; z\ge 1-C_2^{\text{OPT}}\},\\
    &I_2\coloneqq \cR\cap\{z\ge y-C_2^{\text{OPT}}\;\wedge\; z\le 1-C_2^{\text{OPT}}\},\\
    &I_3\coloneqq \cR\cap\{z\le y-C_2^{\text{OPT}}\;\wedge\; z\le -1-C_2^{\text{OPT}}\},\\
    &I_4\coloneqq \cR\cap\{z\le y-C_2^{\text{OPT}}\;\wedge\; z\ge -1-C_2^{\text{OPT}}\}.
  \end{align*}
  We remark that, for $I_1,I_2$, it holds that $z\le y$ (otherwise they are
  empty).

  Now, if we denote $m\coloneqq \max\{-1,1-C_2^{\text{OPT}}\}$, the volume of
  $I_1$ is (we must enforce $z\le y$ otherwise this becomes the signed volume):
  \begin{align*}
    I_1&=\int\limits_{m}^{+1}
    \int\limits_{z}^{+1}\int\limits_{z}^{y}dxdydz\\
    &=\int\limits_{m}^{+1}
    \int\limits_{z}^{+1}(y-z)dydz\\
    &=\int\limits_{m}^{+1}
    (\frac{z^2}{2}-z+\frac{1}{2})dydz\\
    &=\frac{1}{6}-\frac{m^3}{6}+\frac{m^2}{2}-\frac{m}{2}.
  \end{align*}
  Depending on $C_2^{\text{OPT}}$, we get different values:
  \begin{align*}
    I_1=\begin{cases}
      \frac{(C_2^{\text{OPT}})^3}{6}&\text{if }C_2^{\text{OPT}}\le 2,\\
      \frac{4}{3}&\text{otherwise},
    \end{cases}
  \end{align*}
  where we used that, for $C_2^{\text{OPT}}\ge 2$, we have $m=-1$, and replaced,
  and similarly for the other term:
  \begin{align*}
    \frac{1}{6}-\frac{(1-C)^3}{6}+\frac{(1-C)^2}{2}-\frac{(1-C)}{2}&=
    \frac{(C_2^{\text{OPT}})^3}{6}.
  \end{align*}

  W.r.t. $I_2$, we can write:
  \begin{align*}
    I_2&=\int\limits_{-1}^{1-C_2^{\text{OPT}}}
    \int\limits_{z}^{z+C_2^{\text{OPT}}}\int\limits_{z}^{y}dxdydz\\
    &=\int\limits_{-1}^{1-C_2^{\text{OPT}}}
    \int\limits_{z}^{z+C_2^{\text{OPT}}}(y-z)dydz\\
    &=\int\limits_{-1}^{1-C_2^{\text{OPT}}}(\frac{y^2}{2}\bigg|_z^{z+C_2^{\text{OPT}}} -z C_2^{\text{OPT}})dz\\
    &=\frac{(C_2^{\text{OPT}})^2}{2}\int\limits_{-1}^{1-C_2^{\text{OPT}}}dz\\
    &=\frac{(C_2^{\text{OPT}})^2}{2}(2-C_2^{\text{OPT}}),
  \end{align*}
  where we are assuming that $C_2^{\text{OPT}}\le 2$. Otherwise, $I_2=0$.

  Regarding $I_3$, simply note that it is empty, so the volume is 0, because we
  are asking that $z\le-1-C_2^{\text{OPT}}\le -1$.

  Thus, w.r.t. $I_4$, we can write (the upper bound $1-C_2^{\text{OPT}}$ on $z$ arises because
  we ask $z+C_2^{\text{OPT}}\le 1$):
  \begin{align*}
    I_4&=\int\limits_{-1}^{1-C_2^{\text{OPT}}}
    \int\limits_{z+C_2^{\text{OPT}}}^{1}\int\limits_{y-C_2^{\text{OPT}}}^{y}dxdydz\\
    &=\int\limits_{-1}^{1-C_2^{\text{OPT}}}
    \int\limits_{z+C_2^{\text{OPT}}}^{1}Cdydz\\
    &=\int\limits_{-1}^{1-C_2^{\text{OPT}}}
    (1-z-C_2^{\text{OPT}})Cdz\\
    &=\frac{C_2^{\text{OPT}}}{2}(C_2^{\text{OPT}}-2)^2,
  \end{align*}
  which holds assuming $C_2^{\text{OPT}}\le 2$, otherwise it is 0.

  Finally, we get that the volume of $\cR= \sR^{\text{OPT}}_\cM\cap
  \cR^{\text{OPT},\cS}_{\cM,\pi_2}\cap \sR^1_p\cap
  \cR^{\text{OPT},\cS}_{\cM',\pi_1}$ is (we sum all contributions as a function
  of $C_2^{\text{OPT}}$):
  \begin{align*}
    \text{vol}(\cR)&=I_1+I_2+I_3+I_4\\
    &=\begin{cases}
      \frac{(C_2^{\text{OPT}})^3}{6}+\frac{(C_2^{\text{OPT}})^2}{2}(2-C_2^{\text{OPT}})+0+
      \frac{C_2^{\text{OPT}}}{2}(C_2^{\text{OPT}}-2)^2&\text{if }C_2^{\text{OPT}}\le 2,\\
      \frac{4}{3}+0+0+0&\text{otherwise},
    \end{cases}\\
    &=\begin{cases}
      \frac{(C_2^{\text{OPT}})^3}{6}-(C_2^{\text{OPT}})^2
      +2C_2^{\text{OPT}}
      &\text{if }C_2^{\text{OPT}}\le 2,\\
      \frac{4}{3}&\text{otherwise}.
    \end{cases}
  \end{align*}
  By taking the ratio, we get:
  \begin{align*}
    \frac{\text{vol}(\sR^{\text{OPT}}_\cM\cap
      \cR^{\text{OPT},\cS}_{\cM,\pi_2}\cap \sR^1_p\cap
      \cR^{\text{OPT},\cS}_{\cM',\pi_1})}{\text{vol}(\sR^{\text{OPT}}_\cM
      \cap \cR^{\text{OPT},\cS}_{\cM,\pi_2}\cap
      \sR^1_p)}&=\begin{cases}
        \frac{1}{3C_2^{\text{OPT}}}
      &\text{if }C_2^{\text{OPT}}\ge 2,\\
      \frac{(C_2^{\text{OPT}})^2}{24}-\frac{C_2^{\text{OPT}}}{4}
      +\frac{1}{2}&\text{if }0<C_2^{\text{OPT}}<2,
      \end{cases}
  \end{align*}
  which is different from $\frac{1}{2}$ for every $C_2^{\text{OPT}}>0$.

  This concludes the proof.
\end{proof}

\section{Additional Results and Proofs for Section \ref{sec: centroids}}
\label{apx: centroids}

In Appendix \ref{apx: rewrite avg case}, we prove Proposition \ref{prop: avg
case rewritten}. In Appendix \ref{apx: centroids our priors}, we provide the
missing proofs for Theorems \ref{thr: centroid OPT}, \ref{thr: centroid MCE} and
\ref{thr: centroid BIRL}, while in Appendix \ref{apx: proofs IL} we prove
Proposition \ref{prop: int IL} and Proposition \ref{prop: birl opt tx}. Finally,
Appendix \ref{apx: additional centroids OPT} contains additional centroid
calculations for OPT.

\subsection{Derivation of Eq. \eqref{eq: avg case rewritten}}
\label{apx: rewrite avg case}

\avgcaserewritten*
\begin{proof}
The derivation is analogous to the proof of Theorem 3 of
\cite{Ramachandran2007birl}. Starting from Eq. \eqref{eq: avg case}, we can
write:
\begin{align*}
  \widehat{\pi}_w&\in \argmin\limits_{\pi'\in\Pi_{c,k}} 
  \int_{\cR_{\cM,\pi_E}^m} w(r)\Bigr{\max\limits_{\pi\in\Pi_{c,k}} V^{\pi}(s_0';p',r)
  - V^{\pi'}(s_0';p',r)}dr\\
  &\markref{(1)}{=}\argmin\limits_{\pi'\in\Pi_{c,k}}\biggr{
  \int_{\cR_{\cM,\pi_E}^m} w(r)\max\limits_{\pi\in\Pi_{c,k}} V^{\pi}(s_0';p',r)dr
  -\int_{\cR_{\cM,\pi_E}^m} w(r)V^{\pi'}(s_0';p',r)dr}\\
  &\markref{(2)}{=}
    \argmin\limits_{\pi'\in\Pi_{c,k}}-\int_{\cR_{\cM,\pi_E}^m} w(r)V^{\pi'}(s_0';p',r)dr\\
  &\markref{(3)}{=}
  \argmax\limits_{\pi'\in\Pi_{c,k}}\int_{\cR_{\cM,\pi_E}^m} w(r)V^{\pi'}(s_0';p',r)dr\\
  &\markref{(4)}{=}
  \argmax\limits_{\pi'\in\Pi_{c,k}}V^{\pi'}(s_0';p',r_{w,\pi_E}),
\end{align*}
where at (1) we use the linearity of the integral operator, at (2) we realize
that the first term does not depend on $\pi'$, and, as such, it can be removed,
at (3) we remove the minus by replacing argmin with argmax, finally, at (4) we
realize that $V^{\pi'}(s_0';p',r)$ is a linear function of $r$ (it can be seen
as the dot product between the occupancy measure of $\pi'$ and $r$
\cite{puterman1994markov}), thus, by defining $r_{w,\pi_E}\coloneqq
\int_{\cR_{\cM,\pi_E}^m} w(r)dr$, and recalling that $\cS$ and $\cA$ are finite,
we get the result.
\end{proof}

\subsection{The Centroids of the Feasible Sets w.r.t. our New Priors}
\label{apx: centroids our priors}

\centroidOPT*
\begin{proof}
  We can write:
    \begin{align*}
      r^{\text{OPT}}_{\cM,\pi_E}
       &\markref{(1)}{\coloneqq}\int_{\sR^{\text{OPT}}_\cM\cap \cR_{\cM,\pi_E}^{\text{OPT}}}
       rdr\\
       &\markref{(2)}{=}\int_{\sR^{\text{OPT}}_\cM\cap \cR_{\cM,\pi_E}^{\text{OPT}}\cap\sR^{1}_p}
       rdr+\int_{\sR^{\text{OPT}}_\cM\cap \cR_{\cM,\pi_E}^{\text{OPT}}\cap\sR^{>1}_p}
       rdr\\
       &\markref{(3)}{=}\int_{\sR^{\text{OPT}}_\cM\cap \cR_{\cM,\pi_E}^{\text{OPT}}\cap\sR^{1}_p}
       rdr\\
       &\markref{(4)}{=}\popblue{\sum\limits_{\pi\in[\pi_E]_{\cS_{\cM,\pi_E}}}}
       \int_{\sR^{\text{OPT}}_\cM\cap \popblue{\cR_{\cM,\pi}^{\text{OPT},\cS}}\cap\sR^{1}_p}
       rdr\\
       &\markref{(5)}{=}\sum\limits_{\pi\in[\pi_E]_{\cS_{\cM,\pi_E}}}
       \popblue{\int_{\cV_\pi}
       |\text{det}(T_{p,\pi})|
       T_{p,\pi}(V,A)dVdA}\\
       &\markref{(6)}{=}\sum\limits_{\pi\in[\pi_E]_{\cS_{\cM,\pi_E}}}
       \popblue{|\text{det}(W_{p,\pi})|}\int_{\cV_\pi}
       T_{p,\pi}(V,A)dVdA\\
       &\markref{(7)}{=}\sum\limits_{\pi\in[\pi_E]_{\cS_{\cM,\pi_E}}}
       |\text{det}(W_{p,\pi})|\popblue{T_{p,\pi}(\overline{V}_\pi,\overline{A}_\pi)},
    \end{align*}
    where at (1) we use the definition, at (2) we use the additivity property of
    the integral for partitioning in rewards that induce a single optimal policy
    $\sR^{1}_p$ and those that do not (see Eq. \ref{eq: def R1}), at (3) we
    recognize that set $\sR^{\text{OPT}}_\cM\cap
    \cR_{\cM,\pi_E}^{\text{OPT}}\cap\sR^{>1}_p$ has zero volume as $\sR^{>1}_p$
    imposes an additional constraint on the rewards in this set (using
    Proposition \ref{prop: fs OPT explicit} it can be seen that this set is
    described by at most $SA-1<SA$ parameters), at (4) we apply Corollary D.1 in
    \cite{lazzati2024offline} to observe that\footnote{See Appendix \ref{apx:
    additional notation} for a definition of $[\cdot]_{\overline{\cS}}$ for any
    set $\overline{\cS}\subseteq\cS$.} $\cR_{\cM,\pi_E}^{\text{OPT}}=
    \bigcup_{\pi\in[\pi_E]_{\cS_{\cM,\pi_E}}}\cR_{\cM,\pi}^{\text{OPT},\cS}$,
    and then we use the distributive property of intersection over union and the
    additivity of the integral,\footnote{To be precise, the sets
    $\sR^{\text{OPT}}_\cM\cap \cR_{\cM,\pi_E}^{\text{OPT}}\cap\sR^{1}_p$ are
    open, so this passage cannot be done as-is, but additional formalism should
    be adopted. Nevertheless, we avoid introducing it here to keep the
    exposition from becoming unnecessarily cumbersome.} at (5), similarly to the
    proof of Proposition \ref{prop: same vol OPT}, we use the change of
    variables formula in Theorem 6.1.7 of \cite{cohn2013measure}, after having
    observed that matrix\footnote{The definition of $T_{p,\pi}$ is provided in
    Eq. \eqref{eq: def T p pi}.} $T_{p,\pi}$ (along with Proposition \ref{prop:
    fs OPT explicit}) permits to reparametrize set $\sR^{\text{OPT}}_\cM\cap
    \cR^{\text{OPT},\cS}_{\cM,\pi}\cap\sR^{1}_p$ using set $\cV_\pi$ as:
      \begin{align*}
        \sR^{\text{OPT}}_\cM\cap
        \cR^{\text{OPT},\cS}_{\cM,\pi}\cap\sR^{1}_p=\Bigc{r\in\RR^{SA}\,|\,
        \exists (V,A)\in\cV_\pi:\,
        r=T_{p,\pi}(V,A)},
      \end{align*}
      where 
      \begin{align*}
        \cV_\pi\coloneqq\Bigc{V\in[-C_1^{\text{OPT}}k_\pi,+C_1^{\text{OPT}}k_\pi]^S,
        A\in [-C_2^{\text{OPT}},0)^{S\times(A-1)}}.
      \end{align*}
      At (6) we apply Proposition \ref{prop: rel T W}, and note that
      $|\text{det}(W_{p,\pi})|$ does not depend on $V,A$, at (7) we applied the
      linearity of operator $T_{p,\pi}$, and defined $\overline{V}_\pi\coloneqq
      \int_{\cV_\pi} VdVdA$ and $\overline{A}_\pi\coloneqq \int_{\cV_\pi}
      AdVdA$.
      
      Now, we turn to the computation of $\overline{V}_\pi$ and
      $\overline{A}_\pi$. For any state $\overline{s}\in\cS$, we have:
      \begin{align*}
        \overline{V}_\pi(\overline{s})&\coloneqq\int_{\cV^{\pi}}V(\overline{s})dVdA\\
        &=\int_{[-C_1^{\text{OPT}}k_\pi,+C_1^{\text{OPT}}k_\pi]^S}
        V(\overline{s})dV \cdot
      \int_{[-C_2^{\text{OPT}},0)^{S\times(A-1)}}dA\\
        &=\prod\limits_{s\in\cS\setminus\{\overline{s}\}}
        \int_{[-C_1^{\text{OPT}}k_\pi,+C_1^{\text{OPT}}k_\pi]}dV(s)
        \cdot\int_{[-C_1^{\text{OPT}}k_\pi,+C_1^{\text{OPT}}k_\pi]}
        V(\overline{s})dV(\overline{s})\\
        &\qquad\qquad
        \cdot \prod\limits_{s\in\cS,a\in\cA\setminus\{\pi(s)\}}
        \int_{[-C_2^{\text{OPT}},0)}dA(s,a)\\
        &=(2 C_1^{\text{OPT}}k_\pi)^{S-1}\cdot (C_2^{\text{OPT}})^{S(A-1)}
        \cdot \biggs{\frac{V(\overline{s})^2}{2}}_{-C_1^{\text{OPT}}k_\pi}^{+C_1^{\text{OPT}}k_\pi}\\
        &=0.
      \end{align*}
      In a similar manner, for any $\overline{s}\in\cS$,
      $\overline{a}\in\cA\setminus\{\pi(\overline{s})\}$, we have:
      \begin{align*}
        \overline{A}_\pi(\overline{s},\overline{a})&\coloneqq
        \int_{\cV^{\pi}}A(\overline{s},\overline{a})dVdA\\
        &=\int_{[-C_1^{\text{OPT}}k_\pi,+C_1^{\text{OPT}}k_\pi]^S}
        dV \cdot
      \int_{[-C_2^{\text{OPT}},0)^{S\times(A-1)}}A(\overline{s},\overline{a})dA\\
        &=\prod\limits_{s\in\cS}\int_{[-C_1^{\text{OPT}}k_\pi,+C_1^{\text{OPT}}k_\pi]}dV(s)
        \cdot \prod\limits_{s\in\cS,a\in\cA\setminus\{\pi(s)\}\wedge (s,a)\neq
        (\overline{s},\overline{a})}
        \int_{[-C_2^{\text{OPT}},0)}dA(s,a)\\
        &\qquad\qquad\cdot 
        \int_{[-C_2^{\text{OPT}},0)}A(\overline{s},\overline{a})dA(\overline{s},\overline{a})\\
        &=(2 C_1^{\text{OPT}}k_\pi)^{S} \cdot (C_2^{\text{OPT}})^{S(A-1)-1}\cdot
        \biggs{\frac{A(\overline{s},\overline{a})^2}{2}}_{-C_2^{\text{OPT}}}^{0}\\
        &=\popblue{-}2^{\popblue{S-1}}\cdot(C_1^{\text{OPT}}k_\pi)^{S}\cdot
        (C_2^{\text{OPT}})^{\popblue{S(A-1)+1}}.
      \end{align*}
      Replacing the values of $\overline{V}_\pi$ and $\overline{A}_\pi$ into the
      previous expression, we get that:
      \begin{align*}
        r^{\text{OPT}}_{\cM,\pi_E}(s,a)&\markref{(8)}{=}\begin{cases}
          0&\text{ if }s\in\cS_{\cM,\pi_E},a=\pi_E(s),\\
          -\sum\limits_{\pi\in[\pi_E]_{\cS_{\cM,\pi_E}}}
        |\text{det}(W_{p,\pi})|x_\pi
        &\text{ if }s\in\cS_{\cM,\pi_E},a\neq\pi_E(s),\\
        -\sum\limits_{\pi\in[\pi_E]_{\cS_{\cM,\pi_E}}:\pi(s)\neq a}
        |\text{det}(W_{p,\pi})|x_\pi
        &\text{ otherwise}
        \end{cases}\\
      &\markref{(9)}{\propto}\begin{cases}
        0&\text{ if }s\in\cS_{\cM,\pi_E},a=\pi_E(s),\\
          -\sum\limits_{\pi\in[\pi_E]_{\cS_{\cM,\pi_E}}}
          1
        &\text{ if }s\in\cS_{\cM,\pi_E},a\neq\pi_E(s),\\
        -\sum\limits_{\pi\in[\pi_E]_{\cS_{\cM,\pi_E}}:\pi(s)\neq a}
        1
        &\text{ otherwise}
      \end{cases}\\
      &\markref{(10)}{=}\begin{cases}
        0&\text{ if }s\in\cS_{\cM,\pi_E},a=\pi_E(s),\\
        -A^{S-|\cS_{\cM,\pi_E}|}
        &\text{ if }s\in\cS_{\cM,\pi_E},a\neq\pi_E(s),\\
        -A^{S-|\cS_{\cM,\pi_E}|-1}(A-1)
        &\text{ otherwise}
      \end{cases}\\
      &\markref{(11)}{\propto}\begin{cases}
        1&\text{ if }s\in\cS_{\cM,\pi_E},a=\pi_E(s),\\
        0
        &\text{ if }s\in\cS_{\cM,\pi_E},a\neq\pi_E(s),\\
        \frac{1}{A}
        &\text{ otherwise},
      \end{cases}
      \end{align*}
      where at (8) we used that
      $T_{p,\pi}(\overline{V}_\pi,\overline{A}_\pi)(s,a)=0$ at all
      $s\notin\cS_{\cM,\pi_E}$, for the policies
      $\pi\in[\pi_E]_{\cS_{\cM,\pi_E}}$ that prescribe $\pi(s)= a$, and we also
      defined $x_\pi\coloneqq 2^{S-1}\cdot(C_1^{\text{OPT}}k_\pi)^{S}\cdot
      (C_2^{\text{OPT}})^{S(A-1)+1}$ for necessity of space, at (9) we use that
      $(k_\pi)^S=1/|\text{det}(W_{p,\pi})|$, and we divided everything by
      $2^{S-1}\cdot(C_1^{\text{OPT}})^{S}\cdot (C_2^{\text{OPT}})^{S(A-1)+1}$
      (recall that we provide a rescaled result), at (10) we use that
      $[\pi_E]_{\cS_{\cM,\pi_E}}$ contains exactly $A^{S-|\cS_{\cM,\pi_E}|}$
      policies, while set $\{\pi\in[\pi_E]_{\cS_{\cM,\pi_E}}:\pi(s)\neq a\}$
      contains exactly $A^{S-|\cS_{\cM,\pi_E}|-1}(A-1)$ policies, and at (11) we
      rescale everything by first summing $A^{S-|\cS_{\cM,\pi_E}|}$, and then
      dividing by it.

      This concludes the proof.
\end{proof}

\centroidMCE*
\begin{proof}
  The condition on $C_2^{\text{MCE}}$ is made for the same reason as in the
  proof Proposition \ref{prop: same vol MCE}.

  Define $\eta(s,a)=\lambda\log{\pi_E}(a|s)$ for all $s,a$.
  First, we remark that, since set $\sR^{\text{MCE}}_\cM\cap
  \cR_{\cM,\pi_E}^{\text{MCE}}$ is $S$-dimensional (immediate from Proposition
  \ref{prop: same vol MCE}), then the volume elements below are $S$-dimensional.
  We can write:
  \begin{align*}
    r^{\text{MCE}}_{\cM,\pi_E}
     &\markref{(1)}{\coloneqq}\int_{\sR^{\text{MCE}}_\cM\cap \cR_{\cM,\pi_E}^{\text{MCE}}}
     rdr\\
     &\markref{(2)}{=}\int_{\sR^{\text{MCE}}_\cM\cap \popblue{\cR_{\cM,\pi_E}^{\text{MCE},\cS}}}
     rdr\\
     &\markref{(3)}{=}
     \int_{\cV_{\pi_E}'}
     \sqrt{|\text{det}(J_{p,\eta}^\intercal J_{p,\eta})|}
     U_{p,\eta}(V)dV\\
     &\markref{(4)}{=}
     \popblue{\sqrt{|\text{det}(J_{p,\eta}^\intercal J_{p,\eta})|}}
     \int_{\cV_{\pi_E}'}
     U_{p,\eta}(V)dV\\
     &\markref{(5)}{\propto}\;
     \int_{\cV_{\pi_E}'}
     U_{p,\eta}(V)dV\\
     &\markref{(6)}{=}
     \popblue{U_{p,\eta}}\biggr{\int_{\cV_{\pi_E}'}
     VdV}\\
     &\markref{(7)}{=}
     U_{p,\eta}(\popblue{0})\\
     &\markref{(8)}{=}\eta,
  \end{align*}
  where at (1) we use the definition of $r^{\text{MCE}}_{\cM,\pi_E}$, at (2) we exploit the fact that, by hypothesis, $\cS_{\cM,\pi_E}=\cS$,
  at (3) we use the change of variables formula in
  \cite{boothby1986manifolds,stoll2007integration}, after having observed that
  the \emph{rectangular}\footnote{The set $\sR^{\text{MCE}}_\cM\cap
  \cR^{\text{MCE},\cS}_{\cM,{\pi_E}}$ is $S$-dimensional but is embedded into
  $\RR^{SA}$, so we cannot apply Proposition 6.1.3 of \cite{cohn2013measure}.}
  matrix $U_{p,\eta}$, defined in Eq. \eqref{eq: def U eta}, along with
  Proposition \ref{prop: fs MCE explicit}, permits to reparametrize the set
  $\sR^{\text{MCE}}_\cM\cap \cR^{\text{MCE},\cS}_{\cM,{\pi_E}}$ using set
  $\cV_{{\pi_E}}'$ as:
    \begin{align*}
      \sR^{\text{MCE}}_\cM\cap
    \cR^{\text{MCE},\cS}_{\cM,{\pi_E}}=\Bigc{r\in\RR^{SA}\,|\,\exists V\in\cV_{\pi_E}':\,
    r=U_{p,\eta}(V)},
    \end{align*}
    where 
    \begin{align*}
      \cV_{\pi_E}'\coloneqq[-C_1^{\text{MCE}},+C_1^{\text{MCE}}]^S.
    \end{align*}
    Specifically, we use symbol $J_{p,\eta}\in\RR^{SA\times S}$ to denote the
    Jacobian of the operator $U_{p,\eta}$, and we omit the dependence on $V$
    since, as observed in the proof of Proposition \ref{prop: same vol MCE}, it
    is constant in $V$. At (4), we bring it outside because it does not depend
    on $V$, at (5) we drop it as we will rescale, at (6) we use the linearity of
    $U_{p,\eta}$, at (7) we use that $\int_{\cV_{\pi_E}'} VdV=0\in\RR^S$, as the
    domain is symmetric, and at (8) we use the defintion of operator $U_{p,\eta}$.

    The result can be obtained by rescaling the computed reward by $\lambda$.
\end{proof}

\centroidBIRL*
\begin{proof}
  The condition on $C_2^{\text{BIRL}}$ is made for the same reason as in the
  proof Proposition \ref{prop: same vol BIRL}.

The proof is analogous to that of Theorem \ref{thr: centroid MCE}, with the only
    difference that we define
    $\eta(s,a)=\beta\log(\pi_E(a|s)/\max_{a'\in\cA}\pi_E(a'|s))$ at all $s,a$, and
    that we use Proposition \ref{prop: fs BIRL explicit} in place of Proposition
    \ref{prop: fs MCE explicit}.
\end{proof}

\subsection{Proofs for Additional Results on the Centroids}
\label{apx: proofs IL}

\intIL*
\begin{proof}
  Regarding MCE and BIRL, thanks to Propositions \ref{prop: fs MCE explicit} and
  \ref{prop: fs BIRL explicit}, it is immediate to notice that
  $r^{\text{MCE}}_{\cM,\pi_E}$ and $r^{\text{BIRL}}_{\cM,\pi_E}$ are obtained
  from $r_E$ through reward shaping \cite{ng1999shaping} (the state-only
  function is the value function). Therefore, they induce the same optimal
  policies.

  Regarding OPT, observe that
  $V^{\pi_E}(s_0;p,r^{\text{OPT}}_{\cM,\pi_E})=\frac{1}{1-\gamma}$, which is the
  maximum possible, thus $\pi^*=\pi_E$. Moreover, observe also that $\pi_E$ is
  the unique optimal policy in $\cM_{r^{\text{OPT}}_{\cM,\pi_E}}$ as playing a
  non-expert's action gives a reward smaller than 1. For these reasons, since
  $\pi_E$ is optimal for $r_E$ (by hypothesis of OPT), then the result follows. 
\end{proof}

%
The next result shows that, in the absence of additional constraints, reward
$r^{\text{BIRL}}_{\cM,\pi_E}$ transfers to new $\cM'$ the optimal policy w.r.t.
$r_E$ in $\cM$. Moreover, $r^{\text{OPT}}_{\cM,\pi_E}$ does so as well when
$\cS_{\cM,\pi_E}=\cS$.
\begin{restatable}{prop}{birlopttx}\label{prop: birl opt tx}
  Assume that $\pi_E$ and $r_E$ are related by BIRL in an \MDPr
  $\cM=\tuple{\cS,\cA,s_0,p,\gamma}$, and that $\cS_{\cM,\pi_E}=\cS$.
  Then, for any \MDPr $\cM'=\tuple{\cS,\cA,\gamma',s_0',p'}$, it holds that:
  \begin{align*}
    \argmax_{\pi\in\Pi} V^\pi(s_0;p',r^{\text{BIRL}}_{\cM,\pi_E})=
    \argmax_{\pi\in\Pi} V^\pi(s_0;p,r_E).
  \end{align*}
  Moreover, if $\pi_E$ is deterministic and, still, $\cS_{\cM,\pi_E}=\cS$, then
  the same result holds also for OPT.
\end{restatable}
\begin{proof}
  Regarding the first statement, observe that reward
  $r^{\text{BIRL}}_{\cM,\pi_E}$, assuming that $\cS_{\cM,\pi_E}=\cS$, associates
  to at least one action in every possible state the reward 0. Indeed, if we
  take $\overline{a}\in\argmax_{a'\in\cA} \pi_E(a'|s)$, we have:
  \begin{align*}
    \log\frac{\pi_E(\overline{a}|s)}{\max_{a'\in\cA}\pi_E(a'|s)}=\log 1 =0.
  \end{align*}
  Therefore, any greedy policy $\pi^*$ that plays such action/s $\overline{a}$
  achieves, for any possible $p'$ and $\gamma'$:
  \begin{align*}
    V^{\pi^*}(s_0;p,r^{\text{BIRL}}_{\cM,\pi_E})=0.
  \end{align*}
  Since playing at least once an action different from $\overline{a}$ gives a
  smaller expected return (since its corresponding reward is $\le0$), then we
  have that such policy $\pi^*$ is optimal in any $\cM'$.
  In particular, $\pi^*$ is optimal also in $\cM$, and thanks to Proposition
  \ref{prop: int IL} this means that it is optimal also w.r.t. $r_E$. This
  concludes the proof of the first claim.

  For the second claim, simply observe that, if $\cS_{\cM,\pi_E}=\cS$, then:
  \begin{align*}
    r^{\text{OPT}}_{\cM,\pi_E}(s,a)=\indic{a=\pi_E(s)},
  \end{align*}
  that makes $\pi_E$ the optimal policy in any $\cM'$. Since, by hypothesis of
  OPT, $\pi_E$ is optimal also w.r.t. $r_E$, then the result follows.
\end{proof}

\subsection{Additional Centroid Calculations for OPT}
\label{apx: additional centroids OPT}

We now focus on OPT and we show that the centroid of the prior
$w_\cM^{\text{OPT}}$, after rescaling is 0, which is an additional interesting
property of this prior.

\begin{prop}\label{prop: centroid prior OPT}
  For any \MDPr $\cM$, define:
  \begin{align*}
    r^{\text{OPT}}_\cM\coloneqq \int_{\RR^{SA}} w_\cM^{\text{OPT}}(r)rdr.
  \end{align*}
  Then, after rescaling, it holds that: $r^{\text{OPT}}_\cM=0$.
\end{prop}
\begin{proof}
  The proof is completely analogous to that of Theorem \ref{thr: centroid OPT},
  with the difference that, instead of summing over the policies in
  $[\pi_E]_{\cS_{\cM,\pi_E}}$, we sum over all the possible deterministic
  policies $\Pi^{\text{D}}\coloneqq \cA^S$. Then, all the state-action pairs
  $s,a$ are treated equally, and we obtain:
  \begin{align*}
    r^{\text{OPT}}_\cM(s,a)&=\sum\limits_{\pi\in\Pi^{\text{D}}}
    |\text{det}(W_{p,\pi})|T_{p,\pi}(\overline{V}_\pi,\overline{A}_\pi)(s,a)\\
    &=-\sum\limits_{\pi\in\Pi^{\text{D}}:\pi(s)\neq a}
    |\text{det}(W_{p,\pi})|2^{S-1}\cdot(C_1^{\text{OPT}}k_\pi)^{S}\cdot
    (C_2^{\text{OPT}})^{S(A-1)+1}\\
    &=-\sum\limits_{\pi\in\Pi^{\text{D}}:\pi(s)\neq a}
    2^{S-1}\cdot(C_1^{\text{OPT}})^{S}\cdot
    (C_2^{\text{OPT}})^{S(A-1)+1}\\
    &=2^{S-1}\cdot(C_1^{\text{OPT}})^{S}\cdot
    (C_2^{\text{OPT}})^{S(A-1)+1}\cdot A^{S-1}(A-1),
  \end{align*}
  which is the same for every $s,a$. Thus, by subtracting
    $2^{S-1}\cdot(C_1^{\text{OPT}})^{S}\cdot (C_2^{\text{OPT}})^{S(A-1)+1}\cdot
    A^{S-1}(A-1)$ by $r^{\text{OPT}}_\cM$ elementwise, we get the zero vector:
    \begin{align*}
      r^{\text{OPT}}_\cM(s,a)=0\qquad\forall (s,a)\in\SA.
    \end{align*}
    This concludes the proof.
\end{proof}

We now show what happens if we replace prior $w_\cM^{\text{OPT}}$ with a new
prior $w_{\cM,q}$ that is not uniform on $\sR^{\text{OPT}}_\cM$ anymore, but
depends on a vector $q\in\Delta^{\Pi^{\text{D}}}$ in the same way:
\begin{align*}
  w_{\cM,q}(r)\coloneqq \begin{cases}
    q(\pi)/\text{vol}(\sR^{\text{OPT}}_\cM)&\text{if }r\in \sR_\cM^{\text{OPT}}\wedge
    \pi\in\Pi^*(p,r)\wedge |\Pi^*(p,r)|,\\
    0&\text{otherwise}.
  \end{cases}
\end{align*}
Simply put, we are assigning different weights on the rewards in
$\sR^{\text{OPT}}_\cM$ depending on the induced optimal policy (we neglect
rewards that induce more than a single optimal policy because such set has zero
volume). It can be shown that:
\begin{restatable}[OPT expert]{thr}{centroidOPT2}
  For any \MDPr $\cM$ and deterministic expert's policy $\pi_E$, define:
  \begin{align*}
    r^{\text{OPT}}_{\cM,\pi_E,q}\coloneqq \int_{\cR_{\cM,\pi_E}^{\text{OPT}}}
    w_{\cM,q}(r)rdr.
  \end{align*}
  Then, after rescaling, it holds that:
  \begin{align*}
    r^{\text{OPT}}_{\cM,\pi_E,q}(s,a)=\begin{cases}
      1&\text{if }s\in\cS_{\cM,\pi_E},\pi_E(s)=a,\\
      0&\text{if }s\in\cS_{\cM,\pi_E},\pi_E(s)\neq a,\\
      \frac{\sum_{\pi\in[\pi_E]_{\cS_{\cM,\pi_E}}:\pi(s)\neq a} q(\pi)}{\sum_{\pi\in[\pi_E]_{\cS_{\cM,\pi_E}}}
      q(\pi)}&\text{if }s\notin\cS_{\cM,\pi_E}.
    \end{cases}
  \end{align*}
\end{restatable}
\begin{proof}
  Again, the proof is completely analogous to that of Theorem \ref{thr: centroid
  OPT}. The only difference is that, after the first passages, we end up with:
  \begin{align*}
    \sum\limits_{\pi\in[\pi_E]_{\cS_{\cM,\pi_E}}}
    r^{\text{OPT}}_{\cM,\pi_E,q}=\sum\limits_{\pi\in[\pi_E]_{\cS_{\cM,\pi_E}}}
    \popblue{q(\pi)}
    |\text{det}(W_{p,\pi})|T_{p,\pi}(\overline{V}_\pi,\overline{A}_\pi),
  \end{align*}
  as $w_{\cM,q}$ is not ``uniform'' over the policies. For this reason, later on
  in the proof, we obtain that, for all $(s,a)\in\SA$:
  \begin{align*}
    r^{\text{OPT}}_{\cM,\pi_E,q}(s,a)\propto \begin{cases}
      0&\text{ if }s\in\cS_{\cM,\pi_E},a=\pi_E(s),\\
        -\sum\limits_{\pi\in[\pi_E]_{\cS_{\cM,\pi_E}}}
        \popblue{q(\pi)}
      &\text{ if }s\in\cS_{\cM,\pi_E},a\neq\pi_E(s),\\
      -\sum\limits_{\pi\in[\pi_E]_{\cS_{\cM,\pi_E}}:\pi(s)\neq a}
      \popblue{q(\pi)}
      &\text{ otherwise},
    \end{cases}
  \end{align*}
  from which, after rescaling, we get the result.
\end{proof}

\section{Sample Complexity}\label{apx: sample complexity}

This appendix is organized into four parts. In Appendix \ref{apx: proof sample
OPT}, \ref{apx: proof sample MCE} and \ref{apx: proof sample BIRL} we provide
the proofs of, respectively, Theorems \ref{thr: sample OPT}, \ref{thr: sample
MCE} and \ref{thr: sample BIRL}. Finally, in Appendix \ref{apx: proof bound diff
r} we provide the proof of Proposition \ref{prop: bound diff r}.

\subsection{Proof of Theorem \ref{thr: sample OPT}}\label{apx: proof sample OPT}

Estimating $r^{\text{OPT}}_{\cM,\pi_E}$ correctly is equivalent to correctly
estimating the support $\cS_{\cM,\pi_E}$ and the set of state-action pairs
$\cZ_{\cM,\pi_E}$ visited by $\pi_E$ in $\cM$, which can be written as:
\begin{align*}
  &\cS_{\cM,\pi_E}=\Bigc{s\in\cS\;\Big|\;\exists t\in\Nat:\, (p^{\pi_E})^t(s|s_0)>0},\\
  &\cZ_{\cM,\pi_E}\coloneqq\Bigc{(s,a)\in\SA\;\Big|\;s\in\cS_{\cM,\pi_E}\wedge \pi_E(s)=a},
\end{align*}
where $p^{\pi_E}\in\RR^{S\times S}$ is the right stochastic matrix that, to each
pair of states $s,s'$, associates $p(s'|s,\pi_E(s))$, thus
$(p^{\pi_E})^t(s|s_0)$ represents the probability of being in state $s$ at stage
$t$ when starting from $s_0$ \cite{puterman1994markov}.
As reported in Algorithm \ref{alg: opt}, our estimators for sets
$\cS_{\cM,\pi_E}$ and $\cZ_{\cM,\pi_E}$ are:
\begin{align*}
  &\widehat{\cS}\coloneqq\Bigc{s\in\cS\;\Big|\;\exists i\in\dsb{N},
  \exists h\in\dsb{H}:\, s_h^i=s},\\
  &\widehat{\cZ}\coloneqq\Bigc{(s,a)\in\SA\;\Big|\;\exists i\in\dsb{N},
  \exists h\in\dsb{H}:\, s_h^i=s\wedge a_h^i=a}.
\end{align*}
We want to find a minimum number of samples $NH$ that guarantees that, w.h.p.,
$\widehat{\cS}=\cS_{\cM,\pi_E}$ and $\widehat{\cZ}=\cZ_{\cM,\pi_E}$. To this
aim, we must first show that taking $H\ge S$ is sufficient for guaranteeing that
the dataset $\cD_E$, made of trajectories long $H$-stages, can contain all the
states and state-action pairs in the support, i.e., $p^{\min,H}_{\cM,\pi_E}>0$:
\begin{lemma}\label{lemma: positive pmin}
  If $H\ge S$, then $p^{\min,H}_{\cM,\pi_E}>0$.
\end{lemma}
\begin{proof}
  Observe that, since $H\ge S$, then:
  \begin{align*}
    p^{\min,H}_{\cM,\pi_E}&\coloneqq \min\limits_{s\in\cS_{\cM,\pi_E}} \P_{\cM,\pi_E}\bigr{
    \exists t\in\dsb{H}:\, s_t=s}\\
    &\ge  \min\limits_{s\in\cS_{\cM,\pi_E}} \P_{\cM,\pi_E}\bigr{
    \popblue{\exists t\in\dsb{S}}:\, s_t=s}\\
    &\ge \min\limits_{s\in\cS_{\cM,\pi_E}}
    \popblue{\min\limits_{t\in\dsb{S}}
    (p^{\pi_E})^t(s|s_0)}.
  \end{align*}
  Therefore, the result can be proved by showing that, for every state $s\in
  \cS_{\cM,\pi_E}$, there exists a stage $t\in\dsb{S}$ for which the probability
  $(p^{\pi_E})^t(s|s_0)$ of being in $s$ at $t$ is strictly greater than 0. But
  this is exactly the claim of Lemma \ref{lemma: H at least S}. This concludes
  the proof.
\end{proof}
\begin{lemma}\label{lemma: H at least S}
  For every state $s\in\cS_{\cM,\pi_E}$, there exists a $t\in\dsb{S}$ such that
  $(p^{\pi_E})^t(s|s_0)>0$.
\end{lemma}
\begin{proof}
  By hypothesis, since $s\in\cS_{\cM,\pi_E}$, then there always exists a
  $t\in\Nat$ such that: $(p^{\pi_E})^t(s|s_0)>0$. If $t\le S$, then the result
  follows immediately, so we restrict to the case where $t > S$. By
  computing matrix $(p^{\pi_E})^t$, we can write:
  \begin{align*}
    (p^{\pi_E})^t(s|s_0)=\sum\limits_{s^{(2)},s^{(3)},\dotsc,s^{(t-1)}\in\cS}
    p^{\pi_E}(s^{(2)}|s_0)p^{\pi_E}(s^{(3)}|s^{(2)})\cdot\dotsc\cdot p^{\pi_E}(s|s^{(t-1)})>0,
  \end{align*}
  which is strictly positive as long as at least one specific path
  $s_0,\overline{s}^{(2)},\dotsc, \overline{s}^{(t-1)},s$ has strictly positive
  probability, which is the hypothesis made by contradiction:
  \begin{align*}
    (p^{\pi_E})^t(s|s_0)>0 &\implies\\ &\exists
    \overline{s}^{(2)},\overline{s}^{(3)},\dotsc, \overline{s}^{(t-1)}:\;
    p^{\pi_E}(\overline{s}^{(2)}|s_0)p^{\pi_E}(\overline{s}^{(3)}|\overline{s}^{(2)})
    \cdot\dotsc\cdot p^{\pi_E}(s|\overline{s}^{(t-1)})>0.
  \end{align*}
  But this means that all these probabilities are strictly positive:
  \begin{align*}
    p^{\pi_E}(\overline{s}^{(i)}|\overline{s}^{(i-1)})>0\;\forall i \in\{2,t\},
  \end{align*}
  where $\overline{s}^{(1)}=s_0, \overline{s}^{(t)}=s$.

  Now, if all the states $\overline{s}^{(i)}$ are distinct, then $t\le S$,
  because there are at most $S$ states. Otherwise, there exist $i\neq i'$ such
  that $\overline{s}^{(i)}=\overline{s}^{(i')}$, and so we can get rid of all
  the terms $p^{\pi_E}(\overline{s}^{(j)}| \overline{s}^{(j-1)})$ for
  $j\in\{i+1, i'\}$. Then, iterating this reasoning, we will end up with all
  distinct states, representing a path long $t'\le S$ for which
  $(p^{\pi_E})^{t'}(s|s_0)>0$.

  This concludes the proof.
\end{proof}
We are now ready to upper bound the sample complexity necessity for estimating
$\cS_{\cM,\pi_E}$ and $\cZ_{\cM,\pi_E}$:
\begin{lemma}\label{lemma: bound support opt expert}
  Let $\delta\in(0,1)$ and $H\ge S$. If $\pi_E$ is deterministic, then it holds that:
  \begin{align*}
    \P_{\cM,\pi_E}\Bigr{\widehat{\cS}=\cS_{\cM,\pi_E} \wedge
    \widehat{\cZ}=\cZ_{\cM,\pi_E}}\ge 1-\delta,
  \end{align*}
  with a number of trajectories:
  \begin{align*}
    N\le \frac{\log\frac{|\cS_{\cM,\pi_E}|}{\delta}}{p^{\min,H}_{\cM,\pi_E}}.
  \end{align*}
\end{lemma}
\begin{proof}
  If $\pi_E$ is deterministic, then observe that $\widehat{\cS}=\cS_{\cM,\pi_E}$
  implies condition $\widehat{\cZ}=\cZ_{\cM,\pi_E}$. Thus, we will focus on the
  former equality. We bound the opposite event:
\begin{align*}
  \P_{\cM,\pi_E}\Bigr{\widehat{\cS}\neq \cS_{\cM,\pi_E}}&=
  \P_{\cM,\pi_E}\Bigr{\exists s\in\cS_{\cM,\pi_E},\forall i\in\dsb{N},\not \exists h\in\dsb{H}: s_h^i=s}\\
  &\markref{(1)}{\le}\sum\limits_{s\in\cS_{\cM,\pi_E}}
  \P_{\cM,\pi_E}\Bigr{\forall i\in\dsb{N},\not \exists h\in\dsb{H}: s_h^i=s}\\
  &\markref{(2)}{=}\sum\limits_{s\in\cS_{\cM,\pi_E}}
  \prod\limits_{i=1}^N \P_{\cM,\pi_E}\Bigr{\not \exists h\in\dsb{H}: s_h^i=s}\\
  &=\sum\limits_{s\in\cS_{\cM,\pi_E}}
  \prod\limits_{i=1}^N \biggr{1-\P_{\cM,\pi_E}\Bigr{\exists h\in\dsb{H}: s_h^i=s}}\\
  &\markref{(3)}{\le}\sum\limits_{s\in\cS_{\cM,\pi_E}}
  \prod\limits_{i=1}^N (1-p^{\min,H}_{\cM,\pi_E})\\
  &=|\cS_{\cM,\pi_E}| (1-p^{\min,H}_{\cM,\pi_E})^N,
\end{align*}
where at (1) we apply a union bound, at (2) we use the fact that different
trajectories are independent of each other, at (3) we lower bound
$\P_{\cM,\pi_E}\Bigr{\exists h\in\dsb{H}: s_h^i=s}$ using
$p^{\min,H}_{\cM,\pi_E}$, which is $>0$ because of $H\ge S$ and Lemma
\ref{lemma: positive pmin}.

Forcing this probability to be at most $\delta$, and solving w.r.t. $N$, we get:
\begin{align*}
  |\cS_{\cM,\pi_E}| (1-p^{\min,H}_{\cM,\pi_E})^N\le\delta&\iff
  N\log(1-p^{\min,H}_{\cM,\pi_E})\le -\log\frac{|\cS_{\cM,\pi_E}|}{\delta}\\
  &\iff N\ge \frac{\log\frac{|\cS_{\cM,\pi_E}|}{\delta}}{\log\frac{1}{1-p^{\min,H}_{\cM,\pi_E}}}.
\end{align*}
By using that, for any $x\in(0,1)$, we have: $\log\frac{1}{1-x}\ge x$, we get
the result.
\end{proof}
There are some interesting observations that can be done on this sample
complexity result. Specifically, note that there is a trade-off in the choice of
$H$. If $H$ is large, then $p^{\min,H}_{\cM,\pi_E}$ is also large, and so $N$
gets small. However, at the same time, increasing $H$ we are using more samples
since the trajectories are longer. The number of samples is:
\begin{align*}
  NH \le \frac{H}{p^{\min,H}_{\cM,\pi_E}}\log\frac{|\cS_{\cM,\pi_E}|}{\delta}.
\end{align*}
This expression can be optimized by $H$ depending on the problem at stake, by
solving:
\begin{align*}
  \argmin\limits_{H\in\Nat: H\ge S}\frac{H}{p^{\min,H}_{\cM,\pi_E}}\log\frac{|\cS_{\cM,\pi_E}|}{\delta}
  = \argmin\limits_{H\in\Nat: H\ge S} \frac{H}{p^{\min,H}_{\cM,\pi_E}}.
\end{align*}
Finally, we are ready to prove the main theorem.
\sampleOPT*
\begin{proof}
    Simply, apply Lemma \ref{lemma: bound support opt expert}, and realize
    from the definition of $r^{\text{OPT}}_{\cM,\pi_E}$ that
    $\widehat{\cS}=\cS_{\cM,\pi_E} \wedge \widehat{\cZ}=\cZ_{\cM,\pi_E}$
    suffices for having $\widehat{r}^{\text{OPT}}=r^{\text{OPT}}_{\cM,\pi_E}$.
\end{proof}

\subsection{Proof of Theorem \ref{thr: sample MCE}}\label{apx: proof sample MCE}

We begin by proving how the error in the estimation of the policy propagates to
the estimation of the centroid:
  \begin{lemma}\label{lemma: upper bound rewards mce}
    Assume that $\min_{(s,a)\in\SA}\pi_E(a|s)\ge\pi_{\min}'$. Let $\widehat{\pi}$
    be the estimate of policy computed in Algorithm \ref{alg: mce}.
Then, it holds that:
    \begin{align*}
      \|\widehat{r}^{\text{MCE}}-r^{\text{MCE}}_{\cM,\pi_E}\|_\infty
      \le \frac{1}{\pi_{\min}}\|\widehat{\pi}-\pi_E\|_\infty.
    \end{align*}
  \end{lemma}
  \begin{proof}
    We can write:
    \begin{align*}
      \|\widehat{r}^{\text{MCE}}-r^{\text{MCE}}_{\cM,\pi_E}\|_\infty&\coloneqq
      \max\limits_{(s,a)\in\SA}\Biga{
        \widehat{r}^{\text{MCE}}(s,a)-r^{\text{MCE}}_{\cM,\pi_E}(s,a)
      }\\
      &=\max\limits_{(s,a)\in\SA}\Biga{
        \log\widehat{\pi}(a|s)
        -\log\pi_E(a|s)}\\
      &= \max\limits_{(s,a)\in\SA}\Biga{
        \log\frac{\widehat{\pi}(a|s)}{\pi_E(a|s)}}
        \\
      &\markref{(1)}{\le} \max\limits_{(s,a)\in\SA}
      \frac{\biga{\widehat{\pi}(a|s)-\pi_E(a|s)}}{\min\{\widehat{\pi}(a|s),\pi_E(a|s)\}}\\
      &\markref{(2)}{\le} \frac{1}{\pi_{\min}'}\max\limits_{(s,a)\in\SA}
      \biga{\widehat{\pi}(a|s)-\pi_E(a|s)},
    \end{align*}
    where at (1) we use that $\log (1+x)\le x$ for $x\ge 0$. Indeed, when
    $\widehat{\pi}(a|s)\ge \pi_E(a|s)$, then $\widehat{\pi}(a|s)/ \pi_E(a|s)\ge
    1$, and so: $\log\frac{\widehat{\pi}(a|s)}{\pi_E(a|s)} =
    \log\Bigr{1+\frac{\widehat{\pi}(a|s)}{\pi_E(a|s)}-1} \le
    (\widehat{\pi}(a|s)-\pi_E(a|s))/\pi_E(a|s)$. Doing similarly in the opposite
    case, we get the minimum at denominator. At (2) we use that
    $\min_{(s,a)\in\SA}\pi_E(a|s)\ge\pi_{\min}'$ by hypothesis and the definition
    of $\widehat{\pi}$.
  \end{proof}
  The next step consists in upper bounding the estimation error of $\pi_E$: 
  \begin{lemma}\label{lemma: upper bound policies mce}
    Let $\epsilon,\delta\in(0,1)$ and assume that
    $\min_{(s,a)\in\SA}\pi_E(a|s)\ge\pi_{\min}'$. Then, it holds that:
    \begin{align*}
      \P_{\cM,\pi_E}\Bigr{\|\widehat{\pi}-\pi_E\|_\infty\le\epsilon}\ge 1-\delta,
    \end{align*}
    with a number of trajectories:
    \begin{align*}
      N\le \frac{16\log\frac{4SA}{\delta}\log\frac{2SA}{\delta}}{\epsilon^2 p^{\min,H}_{\cM,\pi_E}} .
    \end{align*}
  \end{lemma}
  \begin{proof}
    For all $(s,a)\in\SA$, we can write:
  \begin{align*}
    \biga{\widehat{\pi}(a|s)-\pi_E(a|s)}&=
    \bigga{\max\Bigc{\pi_{\min}', \frac{N(s,a)}{\max\{1,N(s)\}}}-\pi_E(a|s)}\\
    &\markref{(1)}{=}
    \bigga{\max\Bigc{\pi_{\min}', \frac{N(s,a)}{\max\{1,N(s)\}}}-\max\Bigc{\pi_{\min}', \pi_E(a|s)}}\\
    &\markref{(2)}{\le}
    \max\biggc{\biga{\pi_{\min}'-\pi_{\min}'}, \Biga{\frac{N(s,a)}{\max\{1,N(s)\}}-\pi_E(a|s)}}\\
    &=\Biga{\frac{N(s,a)}{\max\{1,N(s)\}}-\pi_E(a|s)},\\
    &\markref{(3)}{\le}
    \sqrt{\frac{2\log\frac{4}{\delta}}{\max\{1,N(s)\}}}\\
    &\markref{(4)}{\le}
    \sqrt{\frac{16\log\frac{4}{\delta}\log\frac{2}{\delta}}{N p^{\min,H}_{\cM,\pi_E}}},
  \end{align*}
  where at (1) we use that $\min_{(s,a)\in\SA}\pi_E(a|s)\ge\pi_{\min}'$ by
  hypothesis, at (2) we use that the max operator is 1-Lipschitz, at (3) we use
  that, when $N(s)=0$, then $|\pi_E(a|s)|\le 1\le \sqrt{2\log\frac{4}{\delta}}$,
  while when $N(s)\ge 1$, we apply the double-sided Azuma-Hoeffding's inequality
  (that must be applied because $N(s)$ is a random variable itself), which holds
  w.p. $1-\delta/2$. At (4) we note that each $N(s)\sim\text{Bin}(N,
  \P_{\cM,\pi_E}(\bigcup_{t=0}^H \{s_t=s\}))$ is Binomially distributed, so we
  can apply Lemma A.1 in \cite{xie2021policyfinetuning} w.p. $\delta/2$ to get
  that: $\frac{1}{\max\{1,N(s)\}}\le \frac{8\log\frac{2}{\delta}}{N
  \P_{\cM,\pi_E}(\bigcup_{t=0}^H \{s_t=s\})} \le \frac{8\log\frac{2}{\delta}}{N
  p^{\min,H}_{\cM,\pi_E}}$, where we used the definition of
  $p^{\min,H}_{\cM,\pi_E}$ in Eq. \eqref{eq: def p min}.
  
  By applying a union bound, we have that the above bound for $s,a$ holds w.p.
  $1-\delta$. In addition, the error above is bounded by $\epsilon$ as long as:
  \begin{align*}
    \sqrt{\frac{16\log\frac{4}{\delta}\log\frac{2}{\delta}}{N p^{\min,H}_{\cM,\pi_E}}}\le\epsilon\iff
    N\ge \frac{16\log\frac{4}{\delta}\log\frac{2}{\delta}}{\epsilon^2 p^{\min,H}_{\cM,\pi_E}}.
  \end{align*}  
  The proof is concluded through the application of a union bound over all the
  $(s,a)\in\SA$.
  \end{proof}
  
  We conclude with the main result:
  \sampleMCE*
  \begin{proof}
    Combining Lemma \ref{lemma: upper bound rewards mce} with Lemma \ref{lemma:
    upper bound policies mce}, we get the result:
    \begin{align*}
      N\le \frac{16\log\frac{4SA}{\delta}\log\frac{2SA}{\delta}
      }{\epsilon^2 p^{\min,H}_{\cM,\pi_E}\pi_{\min}'}.
    \end{align*}
  \end{proof}

\subsection{Proof of Theorem \ref{thr: sample BIRL}}\label{apx: proof sample BIRL}

The proof of Theorem \ref{thr: sample BIRL} is very similar to that of Theorem
\ref{thr: sample MCE}. We begin by proving how the error in the estimation of
the policy propagates to the estimation of the centroid:
\begin{lemma}\label{lemma: upper bound rewards birl}
  Assume that $\min_{(s,a)\in\SA}\pi_E(a|s)\ge\pi_{\min}'$. Let $\widehat{\pi}$
  be the estimate of policy computed in Algorithm \ref{alg: birl}.
Then, it holds that, for all $(s,a)\in\SA$:
  \begin{align*}
    |\widehat{r}^{\text{BIRL}}(s,a)-r^{\text{BIRL}}_{\cM,\pi_E}(s,a)|
    \le \frac{2}{\pi_{\min}'}\|\widehat{\pi}-\pi_E\|_\infty+\indic{N(s)=0}\log\frac{1}{\pi_{\min}'}.
  \end{align*}
\end{lemma}
\begin{proof}
  For all the states $s\in\cS$ with $N(s)>0$, for all $a\in\cA$, we can write:
  \begin{align*}
    |\widehat{r}^{\text{BIRL}}(s,a)-r^{\text{BIRL}}_{\cM,\pi_E}(s,a)|
    &=\Biga{
      \log\frac{\widehat{\pi}(a|s)}{\max_{a'}\widehat{\pi}(a'|s)}
      -\log\frac{\pi_E(a|s)}{\max_{a'}\pi_E(a'|s)}}\\
    &\le\Biga{
        \log\frac{\widehat{\pi}(a|s)}{\pi_E(a|s)}}+\Biga{
          \log\frac{\max_{a'}\widehat{\pi}(a'|s)}{\max_{a'}\pi_E(a'|s)}}\\
    &\markref{(1)}{\le}2 \max\limits_{a'\in\cA}\Biga{
      \log\frac{\widehat{\pi}(a'|s)}{\pi_E(a'|s)}}
      \\
    &\markref{(2)}{\le} 2\max\limits_{a'\in\cA}
    \frac{\biga{\widehat{\pi}(a'|s)-\pi_E(a'|s)}}{\min\{\widehat{\pi}(a'|s),\pi_E(a'|s)\}}\\
    &\markref{(3)}{\le} \frac{2}{\pi_{\min}'}\max\limits_{a'\in\cA}
    \biga{\widehat{\pi}(a'|s)-\pi_E(a'|s)},
  \end{align*}
  where at (1) we use that $|\max_x f(x)-\max_x g(x)|\le\max_x|f(x)-g(x)|$, at
  (2) we use that $\log (1+x)\le x$ for $x\ge 0$. Indeed, when
  $\widehat{\pi}(a|s)\ge \pi_E(a|s)$, then $\widehat{\pi}(a|s)/ \pi_E(a|s)\ge
  1$, and so: $\log\frac{\widehat{\pi}(a|s)}{\pi_E(a|s)} =
  \log\Bigr{1+\frac{\widehat{\pi}(a|s)}{\pi_E(a|s)}-1} \le
  (\widehat{\pi}(a|s)-\pi_E(a|s))/\pi_E(a|s)$. Doing similarly in the opposite
  case, we get the minimum at denominator. At (2) we use that
  $\min_{(s,a)\in\SA}\pi_E(a|s)\ge\pi_{\min}'$ by hypothesis and the definition
  of $\widehat{\pi}$.

  Similarly, for all the remaining states $s\in\cS$ satisfying $N(s)=0$, for all
  $a\in\cA$, we can write:
  \begin{align*}
    |\widehat{r}^{\text{BIRL}}(s,a)-r^{\text{BIRL}}_{\cM,\pi_E}(s,a)|
    &=\Biga{
      \log\pi_{\min}'
      -\log\frac{\pi_E(a|s)}{\max_{a'}\pi_E(a'|s)}}\\
    &\le\log\frac{1}{\pi_{\min}'}+
    \Biga{\log\frac{\pi_E(a|s)}{\max_{a'}\pi_E(a'|s)}}\\
    &\markref{(4)}{=}\log\frac{1}{\pi_{\min}'}+
    \Biga{\popblue{\log\frac{\widehat{\pi}(a|s)}{\max_{a'}\widehat{\pi}(a'|s)}}
    -\log\frac{\pi_E(a|s)}{\max_{a'}\pi_E(a'|s)}},
  \end{align*}
  where at (4) we use that
  $\log(\widehat{\pi}(a|s)/\max_{a'}\widehat{\pi}(a'|s))= \log
  (\pi_{\min}'/\pi_{\min}') = 0$ when $N(s)=0$.

  Then, carrying out the same passages as for $N(s)>1$, we get the result.
\end{proof}
Now, we can upper bound $\|\widehat{\pi}^E-\pi_E\|_\infty$.
\begin{lemma}\label{lemma: upper bound policies birl}
Let $\epsilon,\delta\in(0,1)$ and assume that
$\min_{(s,a)\in\SA}\pi_E(a|s)\ge\pi_{\min}'$. Then, it holds that:
\begin{align*}
  \P_{\cM,\pi_E}\Bigr{\|\widehat{\pi}-\pi_E\|_\infty\le\epsilon}\ge 1-\delta,
\end{align*}
with a number of trajectories:
\begin{align*}
  N\le \frac{16\log\frac{4SA}{\delta}\log\frac{2SA}{\delta}}{\epsilon^2 p^{\min,H}_{\cM,\pi_E}} .
\end{align*}
\end{lemma}
\begin{proof}
  Observe that Algorithms \ref{alg: birl} and \ref{alg: mce} estimate $\pi_E$ in
  the same manner. Thus, the proof of this lemma is the same as that of Lemma
  \ref{lemma: upper bound policies mce}.
\end{proof}
Finally, we can prove the main result:
\sampleBIRL*
\begin{proof}
  First, observe that, by applying the same reasoning adopted for proving Lemma
  \ref{lemma: bound support opt expert}, it can be shown that, with probability
  at least $1-\delta$, $\indic{N(s)=0}=0$ for all $s\in\cS$, using a number of
  trajectories:
  \begin{align*}
    N\le \frac{\log\frac{S}{\delta}}{p^{\min,H}_{\cM,\pi_E}},
  \end{align*}
  where we also used that $\cS_{\cM,\pi_E}=\cS$ by hypothesis.
  
  Then, combining this result with Lemma \ref{lemma: upper bound rewards birl}
  and Lemma \ref{lemma: upper bound policies birl} through a union bound, we get
  the result with a number of trajectories bounded by:
  \begin{align*}
    N\le \frac{32\log\frac{8SA}{\delta}\log\frac{4SA}{\delta}
    }{\epsilon^2 p^{\min,H}_{\cM,\pi_E}\pi_{\min}'}
    + \frac{\log\frac{2S}{\delta}}{p^{\min,H}_{\cM,\pi_E}}
    \le \frac{\popblue{33}\log\frac{8SA}{\delta}\log\frac{4SA}{\delta}
    }{\epsilon^2 p^{\min,H}_{\cM,\pi_E}\pi_{\min}'}.
  \end{align*}
\end{proof}

\subsection{Proof of Proposition \ref{prop: bound diff r}}\label{apx: proof
bound diff r}

\bounddiffr*
\begin{proof}
  We can write:
  \begin{align*}
    \Big|\max\limits_{\pi\in\Pi_{c,k}}V^\pi(s_0';p',\widehat{r}^m)
    -\max\limits_{\pi\in\Pi_{c,k}}V^\pi(s_0';p',r_{\cM,\pi_E}^m)\Big|&\markref{(1)}{\le}
    \popblue{\max\limits_{\pi\in\Pi_{c,k}}}\Big|V^\pi(s_0';p',\widehat{r}^m)
    -V^\pi(s_0';p',r_{\cM,\pi_E}^m)\Big|\\
    &\markref{(2)}{\le}\frac{1}{1-\gamma}
    \|\widehat{r}^m-r_{\cM,\pi_E}^m\|_\infty,
  \end{align*}
where at (1) we use the property of the maximum operator that $|\max_x
f(x)-\max_x g(x)|\le \max_x|f(x)-g(x)|$, at (2) we use Holder's inequality after
having realized that $V^\pi(s_0';p',\widehat{r}^m)$ and
$V^\pi(s_0';p',r_{\cM,\pi_E}^m)$ can be written as the dot product of some
vector $d\in\RR^{SA}$ and $\widehat{r}^m$ or $r_{\cM,\pi_E}^m$, where $d$
represents the occupancy measure of $\pi$ in $\cM'$ \cite{puterman1994markov},
and, in particular, it sums to $1/(1-\gamma)$.
\end{proof}

\section{Experimental Details}\label{apx: experimental details}

In this appendix, we provide additional experimental details and simulations
that have been dropped from the main paper due to reasons of space.
In Appendix \ref{apx: match occupancy measure}, we provide a formal explanation
of how we defined and computed $\pi^{\text{MIMIC}}_{\cM,\pi_E,\cM'}$ in the
experiments. In Appendix \ref{apx: compare with MCE}, we report the results of
the experiments conducted for MCE analogous to those provided in Figs. \ref{fig:
experiments OPT} and \ref{fig: experiments BIRL} in Section \ref{sec:
experiments}.
%
%
Next, in Appendix \ref{apx: exps more difference}, we provide some additional
considerations on the difference between the policies $\pi^m_{\cM}$
found through our methods and $\pi^{\text{MIMIC}}_{\cM,\pi_E,\cM'}$.
Finally, we provide a comparison with directly transferring the expert's policy
$\pi_E$ to the new environment $\cM'$ and with a ``best-case'' approach
(Appendix \ref{apx: BC policy}), and we conclude with some additional
simulations in Appendix \ref{apx: other simulations}.

Before diving into the presentation of the numerous sections of this appendix,
we mention that each experiment has been conducted in just a few seconds on a
laptop with processor ``AMD Ryzen 5 5500U with Radeon Graphics 2.10 GHz'' with
8GB of RAM. Moreover, we remark that each policy $\pi^m_{\cM}$ is
defined as any policy in:
\begin{align*}
  \pi^m_{\cM,\pi_E,\cM'}\in\argmax\limits_{\pi'\in\Pi_{c,k}} 
  V^{\pi'}(s_0';p',r^m_{\cM,\pi_E}),
\end{align*}
and we mention that, for the simulations involving MCE and BIRL, we set
$\pi_{\min}= 10^{-6}$.

\subsection{Definition of $\pi^{\text{MIMIC}}_{\cM,\pi_E,\cM'}$}\label{apx: match occupancy measure}

We computed the policy $\pi^{\text{MIMIC}}_{\cM,\pi_E,\cM'}$ as the policy with
state-action occupancy measure $d_{\cM',\pi^{\text{MIMIC}}_{\cM,\pi_E,\cM'}}$ in
$\cM'$ closest to that of the expert $d_{\cM,\pi_E}$ in $\cM$. To measure the
distance between the occupancy measures, we adopted the $\ell_1$-norm. Formally,
we computed $\pi^{\text{MIMIC}}_{\cM,\pi_E,\cM'}$ by first solving the
mathematical program:
\begin{align*}
  \min\limits_{d\in\RR^{SA}}&\|d-d_{\cM,\pi_E}\|_1\\
  \text{s.t.: }&
  \sum\limits_{a\in\cA}d(s,a)=(1-\gamma')\indic{s=s_0'}+\gamma'
  \sum\limits_{s'\in\cS}\sum\limits_{a'\in\cA}
  p'(s|s',a')d(s',a')\qquad\forall s\in\cS,\\
  & \sum\limits_{(s,a)\in\SA}d(s,a)c(s,a)\le k,\\
  &d(s,a)\ge 0\qquad\forall (s,a)\in\SA,\\
  &\sum\limits_{(s,a)\in\SA}d(s,a)=1,
\end{align*}
using a Linear Programming (LP) solver,\footnote{The $\ell_1$-norm minimization
problem can be recast as a LP problem \cite{boyd2004convex}.} and then we found
$\pi^{\text{MIMIC}}_{\cM,\pi_E,\cM'}$ explicitly from its state-action occupancy
measure $d$ as (see \cite{puterman1994markov}):
\begin{align*}
  \pi^{\text{MIMIC}}_{\cM,\pi_E,\cM'}(a|s)=\frac{d(s,a)}{\sum\limits_{a'\in\cA}d(s,a')}.
\end{align*}

\subsection{Results for MCE}\label{apx: compare with MCE}

\begin{figure}[!t]
  \centering
  \begin{minipage}[t!]{0.195\textwidth}
    \centering
    \includegraphics[width=0.98\linewidth]{MCE_BIRL_9_samep__piE.pdf}
\end{minipage}
\begin{minipage}[t!]{0.195\textwidth}
  \centering
  \includegraphics[width=0.98\linewidth]{MCE_BIRL_9_7_eq_C_match_sa.pdf}
\end{minipage}
\begin{minipage}[t!]{0.195\textwidth}
    \centering
    \includegraphics[width=0.98\linewidth]{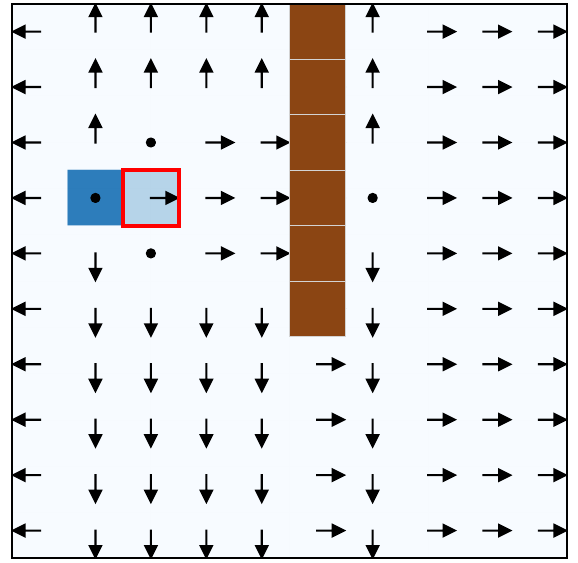}
\end{minipage}
\begin{minipage}[t!]{0.195\textwidth}
  \centering
  \includegraphics[width=0.98\linewidth]{MCE_BIRL_9_99_eq_C_match_sa.pdf}
\end{minipage}
\begin{minipage}[t!]{0.195\textwidth}
    \centering
    \includegraphics[width=0.98\linewidth]{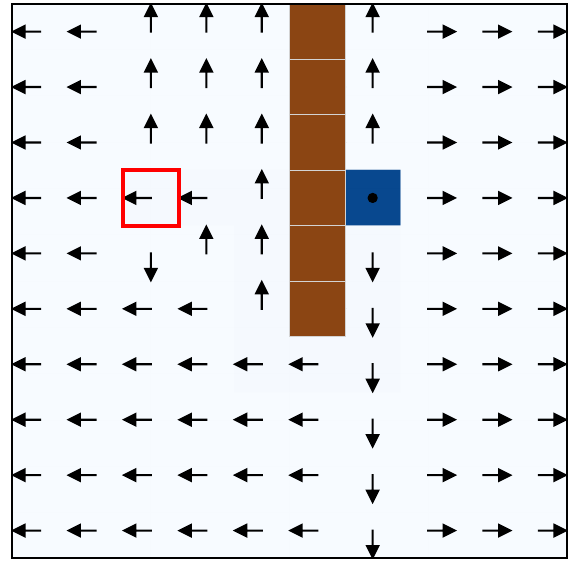}
\end{minipage}
\caption{From left to right: (a) $\pi_E$ in $\cM$ with $\gamma=0.9$, (b)-(c)
$\pi^{\text{MIMIC}}_{\cM,\pi_E,\cM'}$ and $\pi^{\text{MCE}}_{\cM}$
in $\cM'$ with $\gamma'=0.7$, (d)-(e)
$\pi^{\text{MIMIC}}_{\cM,\pi_E,\cM'}$ and $\pi^{\text{MCE}}_{\cM}$
in $\cM'$ with $\gamma'=0.99$. }
\label{fig: experiments MCE}
 \end{figure}

 In Fig. \ref{fig: experiments MCE}, we provide the analogous of the experiments
 reported in Fig. \ref{fig: experiments OPT} and \ref{fig: experiments BIRL} for
 MCE. We remark that the expert's policy $\pi_E$ in Fig. \ref{fig: experiments
 MCE} (a) is the same as the policy used for BIRL, plotted in Fig. \ref{fig:
 experiments BIRL} (a).
 Looking at the images, we realize that $\pi^{\text{MCE}}_{\cM}$ is
 quite similar to $\pi^{\text{BIRL}}_{\cM}$.

\subsection{Additional Considerations on the Difference between $\pi^m_{\cM,\pi_E,\cM'}$ and $\pi^{\text{MIMIC}}_{\cM,\pi_E,\cM'}$}
\label{apx: exps more difference}

\begin{figure}[!t]
  \centering
  \begin{minipage}[t!]{0.195\textwidth}
      \centering
      \includegraphics[width=0.98\linewidth]{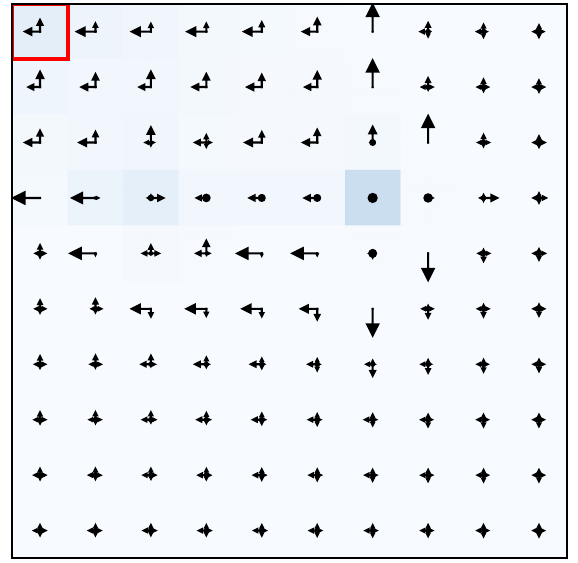}
  \end{minipage}
  \begin{minipage}[t!]{0.195\textwidth}
    \centering
    \includegraphics[width=0.98\linewidth]{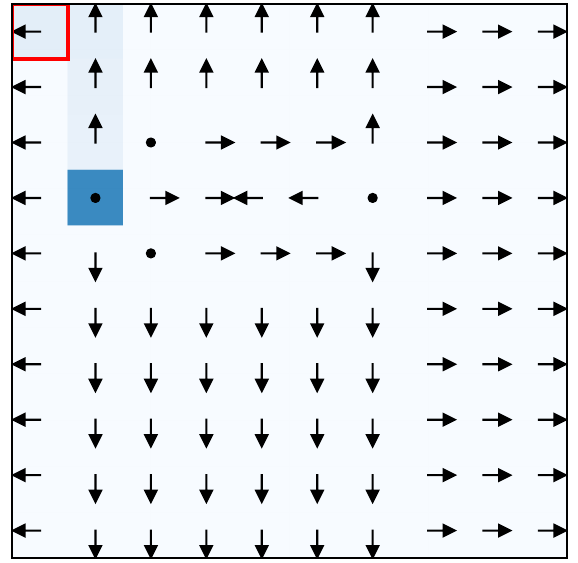}
  \end{minipage}
  \begin{minipage}[t!]{0.195\textwidth}
    \centering
    \includegraphics[width=0.98\linewidth]{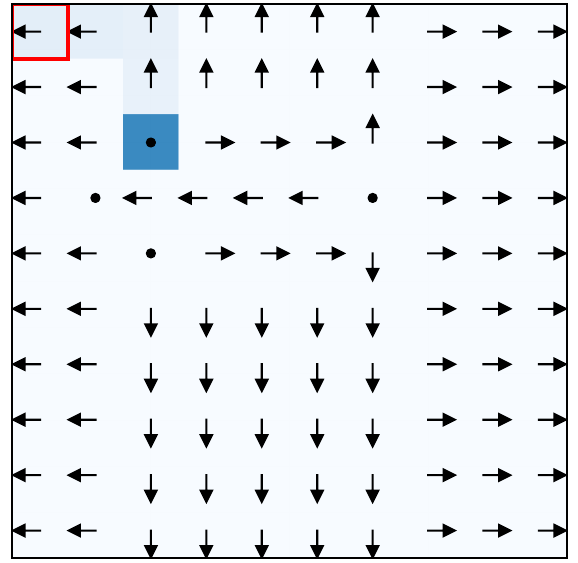}
\end{minipage}
\begin{minipage}[t!]{0.195\textwidth}
  \centering
  \includegraphics[width=0.98\linewidth]{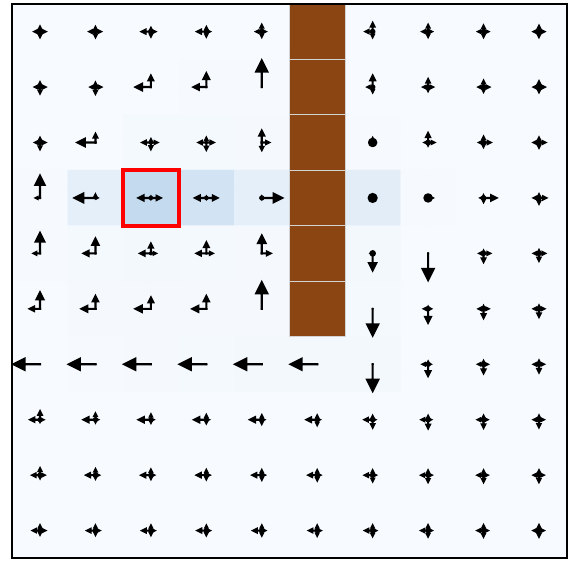}
\end{minipage}
\begin{minipage}[t!]{0.195\textwidth}
    \centering
    \includegraphics[width=0.98\linewidth]{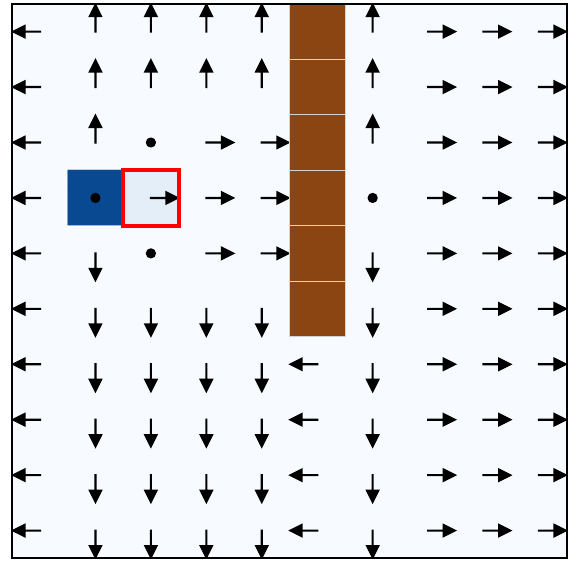}
\end{minipage}
\caption{In this series of images, we consider the $\pi_E$ in Fig. \ref{fig:
experiments BIRL} (a), where $\gamma=0.9$. 
From left to right: (a) $\pi^{\text{MIMIC}}_{\cM,\pi_E,\cM'}$ with
$\gamma'=0.9$, (b) $\pi^{\text{BIRL}}_{\cM}$ with $\gamma'=0.9$, (c)
$\pi^{\text{MCE}}_{\cM}$ with $\gamma'=0.9$, (d)
$\pi^{\text{MIMIC}}_{\cM,\pi_E,\cM'}$ with constraints and $\gamma'=0.9$, (e)
$\pi^{\text{MCE}}_{\cM}$ with constraints and $\gamma'=0.9$.
}
\label{fig: on differences}
 \end{figure}

The difference between our policies $\pi^m_{\cM,\pi_E,\cM'}$ and the policy
$\pi^{\text{MIMIC}}_{\cM,\pi_E,\cM'}$ that tries to match the occupancy measure
of the expert can be better visualized by looking at Fig. \ref{fig: on
differences}.

Let $\pi_E$ be the expert's policy in Fig. \ref{fig: experiments BIRL} (a), with a
support $\cS_{\cM,\pi_E}\subset \cS$ strictly. If we consider the problem of
generalizing to an \MDPr $\cM'$ with an initial state $s_0'$ falling outside
$\cS_{\cM,\pi_E}$, then both the policy $\pi^{\text{MIMIC}}_{\cM,\pi_E,\cM'}$
(Fig. \ref{fig: on differences} (a)) and our policies
$\pi^{\text{MCE}}_{\cM}$ (Fig. \ref{fig: on differences} (b)) and
$\pi^{\text{BIRL}}_{\cM}$ (Fig. \ref{fig: on differences} (c)) try to
get back to the support $\cS_{\cM,\pi_E}$. However, they do so differently. As
shown in Fig. \ref{fig: on differences} (a),
$\pi^{\text{MIMIC}}_{\cM,\pi_E,\cM'}$ cares mostly on getting back to the states
in the support that have been visited the most by the expert, i.e., those with
more intense blue in Fig. \ref{fig: experiments BIRL} (a). Instead, both
$\pi^{\text{MCE}}_{\cM}$ and $\pi^{\text{BIRL}}_{\cM}$ are
satisfied by an \emph{arbitrary} state in $\cS_{\cM,\pi_E}$, as shown by Fig.
\ref{fig: on differences} (b)-(c).
This behavior can be justified mathematically by observing that Algorithms
\ref{alg: mce} and \ref{alg: birl} assign to all the state-action pairs outside
$\cS_{\cM,\pi_E}$ the minimum possible reward values, namely:
\begin{align*}
  &r^{\text{MCE}}_{\cM,\pi_E}(s,a)=r^{\text{BIRL}}_{\cM,\pi_E}(s,a) = \log \pi_{\min}\approx -14,
\end{align*}
where we replaced the value $\pi_{\min}=10^{-6}$ used in the simulations.
As a consequence, $\pi^{\text{MCE}}_{\cM}$ and
$\pi^{\text{BIRL}}_{\cM}$ always try to get back to the states in the
support $\cS_{\cM,\pi_E}$, whose reward values is necessarily larger than $-14$.
However, the rewards constructed by Algorithms \ref{alg: mce} and \ref{alg:
birl} do not depend by the probability with which the expert visited the various
states in its support, i.e., $d_{\cM,\pi_E}$, and so, for
$\pi^{\text{MCE}}_{\cM}$ and $\pi^{\text{BIRL}}_{\cM}$, it
is irrelevant the state $s\in\cS_{\cM,\pi_E}$ in which they re-play the action/s
demonstrated by the expert.

As another example, consider now the presence of constraints in $\cM'$ and look
at Fig. \ref{fig: on differences} (d)-(e). Here, even though $\gamma'=0.9$ is
relatively small, $\pi^{\text{MIMIC}}_{\cM,\pi_E,\cM'}$ still tries to go around
the obstacle, because its objective is to match the expert's behavior over the
whole support $\cS_{\cM,\pi_E}$ (Fig. \ref{fig: on differences} (d)). Instead,
as shown in Fig. \ref{fig: on differences} (e),
$\pi^{\text{MCE}}_{\cM}$ is not interested to the states beyond the
obstacle if they are too distant (in terms of $\gamma'$), but it can be
satisfied with states closer to the initial state $s_0'$. Observe that, if we
increase $\gamma'=0.99$, then also $\pi^{\text{MCE}}_{\cM}$ goes
around the obstacle, as shown in Fig. \ref{fig: experiments MCE} (e).

\subsection{Comparison with Policy Transfer and a ``Best-case'' Approach}\label{apx: BC policy}

\begin{figure}[!t]
  \centering
\begin{minipage}[t!]{0.195\textwidth}
  \centering
  \includegraphics[width=0.98\linewidth]{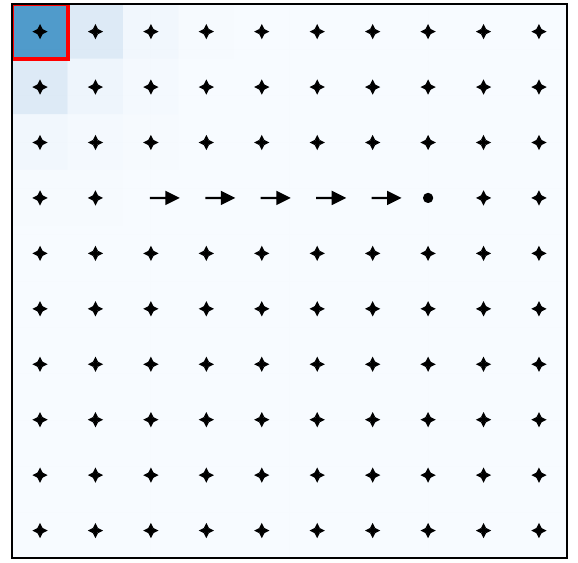}
\end{minipage}
\begin{minipage}[t!]{0.195\textwidth}
    \centering
    \includegraphics[width=0.98\linewidth]{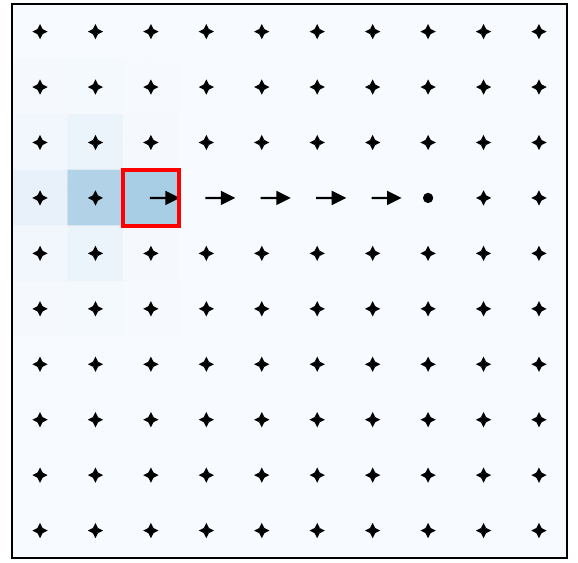}
\end{minipage}
\begin{minipage}[t!]{0.195\textwidth}
  \centering
  \includegraphics[width=0.98\linewidth]{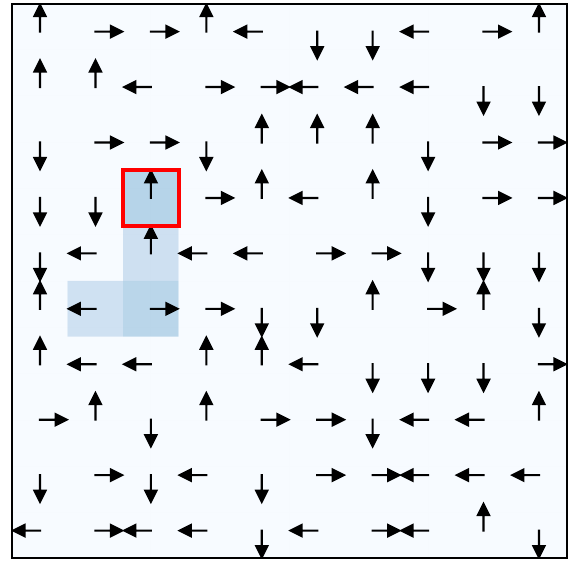}
\end{minipage}
\begin{minipage}[t!]{0.195\textwidth}
  \centering
  \includegraphics[width=0.98\linewidth]{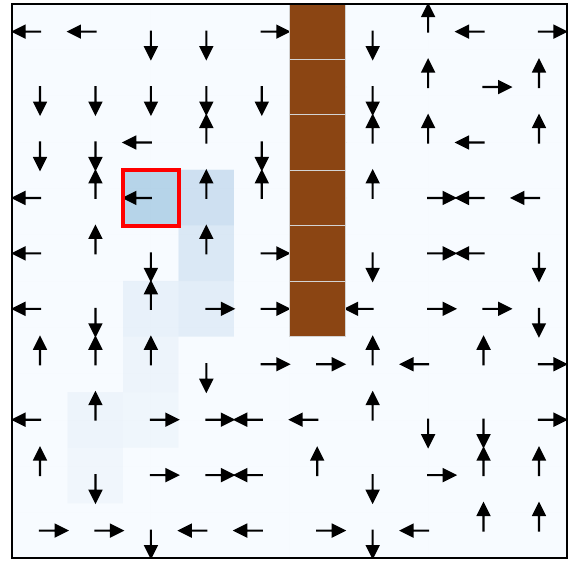}
\end{minipage}
\begin{minipage}[t!]{0.195\textwidth}
  \centering
  \includegraphics[width=0.98\linewidth]{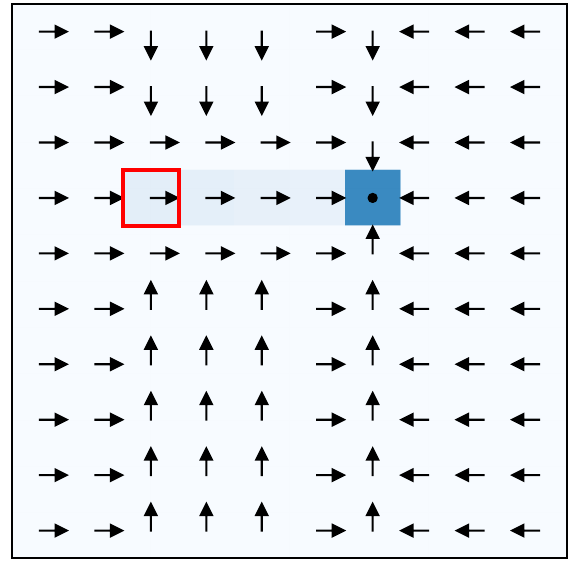}
\end{minipage}
\caption{From left to right: (a) $\pi^{\text{BC}}_{\cM,\pi_E}$ for the expert's
policy $\pi_E$ defined in Fig. \ref{fig: experiments OPT} (a). We consider
$\cM'$ with $\gamma'=0.7$ and $s_0'\neq s_0$. (b) $\pi^{\text{BC}}_{\cM,\pi_E}$ for the expert's
policy $\pi_E$ defined in Fig. \ref{fig: experiments OPT} (a). We consider
$\cM'$ with $\gamma'=0.7$ and $s_0'= s_0$. (c)-(d) Policy $\pi'$ obtained
through the ``best-case'' approach described in Appendix \ref{apx: BC policy},
respectively in absence and presence of constraints. (e) Illustration of
Proposition \ref{prop: int IL}.}
\label{fig: comp BC}
 \end{figure}

We now show how generalization to new environments and constraints is carried
out by other approaches in this specific problem setting.

\paragraph{Policy transfer.}

The first ``baseline'' that we consider here is $\pi^{\text{BC}}_{\cM,\pi_E}$
\cite{arora2020survey}, which is defined as:
\begin{align*}
  \pi^{\text{BC}}_{\cM,\pi_E}(a|s)\coloneqq\begin{cases}
    \pi_E(a|s)&\text{if }s\in\cS_{\cM,\pi_E},\\
    \frac{1}{A}&\text{otherwise}.
  \end{cases}
\end{align*}
This policy is called ``BC'' because it represents the result of Behavioral
Cloning \cite{pomerlau1988alvinn,osa2018IL} on $\pi_E$.
As shown in Fig. \ref{fig: comp BC} (a)-(b), directly transferring this policy
to the new environment $\cM'$ might induce quite meaningless behaviors in $\cM'$
regardless of whether $s_0'\neq s_0$ (Fig. \ref{fig: comp BC} (a)) or not (Fig.
\ref{fig: comp BC} (b)).
Moreover, note that this policy cannot even be applied in general when there are
constraints in $\cM'$, because it might not satisfy them.

\paragraph{A ``best-case'' approach.}

As a second ``baseline'', we consider the policy $\pi'$ obtained after carrying
out (constrained) planning in $\cM'$ with an \emph{arbitrary} reward $r$ in the
feasible set $\cR^{\text{OPT}}_{\cM,\pi_E}$:
\begin{align*}
  \pi'\in\argmax\limits_{\pi''\in\Pi_{c,k}} 
  V^{\pi''}(s_0';p',r).
\end{align*}
We have constructed this reward starting from the knowledge of the policy
$\pi_E$ (see Fig. \ref{fig: experiments OPT} (a)) in its support
$\cS_{\cM,\pi_E}$ as follows. First, we have sampled at random a deterministic
policy $\overline{\pi}$ from $[\pi_E]_{\cS_{\cM,\pi_E}}$. Next, based on Proposition \ref{prop:
fs OPT explicit}, we have sampled uniformly at random the values of
$V\in[-1,+1]^S$ and $A\in[-0.5,+0.5]^{S(A-1)}$ (the numbers 1 and 0.5 are
meaningless), and constructed $r$ as:
\begin{align*}
    r(s,a)=V(s)-\gamma\sum\limits_{s'\in\cS}p(s'|s,a)V(s')-\begin{cases}
      0&\text{if }a=\overline{\pi}(s),\\
      A(s,a)&\text{otherwise}.
    \end{cases}
\end{align*}
Note that this approach is representative of all the ``best-case'' approaches
described in Section \ref{sec: existing methods}.
Fig. \ref{fig: comp BC} (c)-(d) show the resulting policy $\pi'$ in a new
environment $\cM'$ respectively in absence and presence of constraints. As clear
from the images, the policies $\pi'$ and $\pi_E$ are rather uncorrelated.

\subsection{Other Simulations}\label{apx: other simulations}

In this appendix, we present the result of a simulation aimed to exemplify
Proposition \ref{prop: int IL}. Specifically, we focused on BIRL and we computed
the policy $\pi^{\text{BIRL}}_{\cM}$ using $\cM'=\cM$, i.e., as the policy
obtained through our approach in the same environment $\cM$ where the expert
demonstrated $\pi_E$ (see Fig. \ref{fig: experiments BIRL} (a)). The result is
plotted in Fig. \ref{fig: comp BC} (e).
In this setting, $\pi^{\text{BIRL}}_{\cM}$ can be interpreted as a
``more aggressive'' $\pi_E$, in the sense that it aims to achieve the same
objective as the stochastic expert's policy $\pi_E$, but deterministically.

\end{document}